\documentclass{article}

\usepackage{iclr2025_conference}

\usepackage[utf8]{inputenc} 
\usepackage[T1]{fontenc}    
\usepackage{titletoc}
\usepackage[backref=page]{hyperref}      
\usepackage{url}            
\usepackage{booktabs}       
\usepackage{amsfonts}       
\usepackage{nicefrac}       
\usepackage{microtype}      
\usepackage{xcolor}         
\usepackage{macros}
\usepackage{subcaption}
\usepackage{graphicx}
\usepackage{multirow}
\usepackage{array}
\usepackage{algorithm}
\usepackage[noend]{algpseudocode}
\usepackage{natbib} 
\usepackage{anyfontsize}
\usepackage{wrapfig}
\usepackage{comment}
\usepackage[makeroom]{cancel}
\usepackage{bbm}
\usepackage[capitalize,noabbrev]{cleveref}
\usepackage[noabbrev]{cleveref}
\usepackage{tikz}
\usetikzlibrary{bayesnet}
\usetikzlibrary{backgrounds}
\usetikzlibrary{calc}
\usepackage{caption}
\usepackage{subcaption}
\usepackage[group-separator={,}]{siunitx} 
\usepackage{placeins}

\allowdisplaybreaks

\usepackage{xcolor}
\definecolor{arrowlightgrey}{rgb}{0.70,0.70,0.70}
\definecolor{lightgrey}{rgb}{0.43,0.43,0.43}
\definecolor{orange}{rgb}{1,0.49,0.05}
\definecolor{green}{rgb}{0.17,0.63,0.17}

\definecolor{matplotlibblack}{RGB}{42,42,42}

\hypersetup{
    colorlinks,
    anchorcolor={blue!50!black},
    linkcolor={blue!50!black},
    citecolor={blue!50!black},
    urlcolor={blue!50!black},
}

\newcommand{\vsection}[1]{\vspace{-1pt}\section{#1}\vspace{-2pt}}
\newcommand{\vsubsection}[1]{\vspace{-1pt}\subsection{#1}\vspace{-2pt}}
\newcommand{\vparagraph}[1]{\vspace{-3pt}\paragraph{#1}}

\newcommand{\vspacefiguretop}{\vspace{-10pt}}
\newcommand{\vspacefigurecaptionbottom}{\vspace{-3pt}}

\newcommand{\rebuttal}[1]{\textcolor{black}{#1\xspace}}

\title{Standardizing Structural Causal Models}

\author{%
  Weronika Ormaniec%
  \thanks{Equal contribution.}\\
  ETH Z{\"u}rich, Switzerland\\
  \texttt{wormaniec@ethz.ch}
  \And
  Scott Sussex${}^*$\\
  ETH Z{\"u}rich, Switzerland\\
  \texttt{ssussex@ethz.ch}
  \And 
  Lars Lorch${}^*$\\
  ETH Z{\"u}rich, Switzerland\\
  \texttt{llorch@ethz.ch} 
  \AND
  Bernhard Sch{\"o}lkopf\\
  MPI for Intelligent Systems, T{\"u}bingen, Germany\\
  \texttt{bs@tuebingen.mpg.de}
  \And
  Andreas Krause\\
  ETH Z{\"u}rich, Switzerland\\
  \texttt{krausea@ethz.ch}
}

\iclrfinalcopy
\begin{document}

\maketitle

\vspace{-8pt}
\begin{abstract}
\vspace{-5pt}
Synthetic datasets generated by structural causal models (SCMs) are commonly used for benchmarking causal structure learning algorithms. However, the variances and pairwise correlations in SCM data tend to increase along the causal ordering. Several popular algorithms exploit these \artefacts, possibly leading to conclusions that do not generalize to real-world settings. Existing metrics like $\varop$-sortability and \rtwo-sortability quantify these patterns, but they do not provide tools to remedy them. To address this, we propose {\em internally-standardized structural causal models (\ourss)}, a modification of SCMs that introduces a standardization operation at each variable during the generative process. By construction, \ourss are not $\varop$-sortable\rebuttal{. We also find empirical evidence that they are mostly not} \rtwo-sortable for commonly-used graph families. Moreover, contrary to the post-hoc standardization of data generated by standard SCMs, we prove that linear \ourss are less identifiable from prior knowledge on the weights and do not collapse to deterministic relationships in large systems, which may make \ourss a useful model in causal inference beyond the benchmarking problem studied here.  
Our code is publicly available at: \href{https://github.com/werkaaa/iscm}{https://github.com/werkaaa/iscm}.
\end{abstract}

\vspace{-12pt}
\vsection{Introduction}\label{sec:introduction}
\vspace{-2.5pt}
Predicting the effects of interventions and policy decisions requires reasoning about causality.
Consequently, scientific fields ranging from biology and earth sciences to economics and statistics are interested in modeling causal structure
\citep{pearl2009causality,maathuis2010predicting,imbens2015causal,runge2019inferring}.
A wide array of causal discovery algorithms has been proposed with the goal of inferring causal structure from data (e.g., \citealp{Squires_2022,vowels2022d}).
However, benchmarking these algorithms is challenging, since real-world datasets with an agreed-upon, ground-truth causal structure are rare 
(e.g., \citealp{sachs2005causal}; see \citealp{mooij2020joint}).
The community predominantly relies on synthetic data for evaluating structure learning algorithms, where observations are generated according to a predetermined causal structure and system mechanisms. 
The inferred causal structures can then be directly compared to the ground truth. To generate synthetic data, it is common practice to sample from structural causal models with additive noise (SCMs) \citep{reisach2021beware}. Unless stated otherwise, this work considers SCMs in which the variance scale of the additive noise is the same for all variables, a typical simplification made in benchmarking. 

Under common benchmarking practices, synthetic datasets generated by SCMs contain patterns that are directly exploitable to make structure discovery easier.
We will refer to such patterns as {\em \artefacts}.
In SCMs, the pairwise correlations between variables tend to increase along the causal ordering, since variance builds up downstream and, as a result, the proportion of the variance driven by the additive noise vanishes (\cref{fig:figure1-standardized-SCM}).
\citet{reisach2024scale} characterize this phenomenon through an increase of the coefficients of determination (\rtwo) of the variables regressed on all others.
Crucially, this \artefact occurs both in the raw data and when shifting and scaling (standardizing) the variables to have zero mean and unit variance.
One of the implications is that downstream causal dependencies in SCMs become effectively deterministic, especially in large-scale systems.
As \citet{reisach2024scale} demonstrate, simple causal discovery baselines can perform competitively on benchmarks of this kind by directly exploiting this phenomenon. 
This makes SCMs \rebuttal{alone} in their general definition possibly \rebuttal{insufficient} for benchmarking.
Ultimately, evaluating on synthetic data with these patterns could lead to conclusions that do not generalize \rebuttal{as expected} to real-world scenarios. 

In this work, we propose a simple modification of SCMs that stabilizes the data-generating process and thereby removes exploitable covariance \artefacts.
Our models, denoted {\em internally-standardized SCMs (\ourss)}, introduce a standardization operation at each variable during the generative process (\cref{fig:figure1-iSCM}). 
In \cref{sec:theory}, we provide a theoretical motivation for this idea by studying linear \ourss. 
We prove that, contrary to SCMs, the causal dependencies of \ourss under mild assumptions never collapse to deterministic mechanisms as the graph size becomes large.
Moreover, we formalize the correlation \artefact commonly observed in benchmarks by proving that linear SCM structures in a Markov equivalence class (MEC) are partially identifiable for certain graph classes,
given weak prior knowledge on the weight distribution of the ground-truth SCM. 
Most importantly, we show that this is not the case for the corresponding \ourss. 
In \cref{sec:experiments}, we empirically demonstrate that the baselines proposed in \citet{reisach2021beware,reisach2024scale} are unable to exploit covariance \artefacts in \ourss, 
while practical classes of causal discovery algorithms are still able to learn causal structures in both linear and nonlinear systems.
Our findings reveal that SCM \artefacts affect structure learning both positively and negatively, 
\rebuttal{making \ourss a practical tool, alongside SCMs, for disentangling the drivers of causal discovery performance of different algorithms in practice.}

\begin{figure}[t]
\vspacefiguretop
\centering
\vspace{-5pt}
\vspace{-3pt}
\begin{subfigure}[c]{0.44\textwidth}
    \centering
    \hspace*{6pt}
    \begin{tikzpicture}[x=0.55cm,y=0.55cm]
    
    \node[latent] (x1) {$x_1$};
    \node[obs, below=of x1] (x1s) {$x_1^s$};
    \node[latent, right=of x1] (x2) {$x_2$};
    \node[obs, below=of x2] (x2s) {$x_2^s$};
    \node[const, right=of x2]  (dummy) {$~~\dots~~$};
    \node[const, right=of x2s]  (x3s) {$~~\dots~~$};
    \node[latent, right=of dummy] (x9) {$x_9$};
    \node[obs, below=of x9] (x9s) {$x_9^s$};
    \node[latent, right=of x9] (x10) {$x_{10}$};
    \node[obs, below=of x10] (x10s) {$x_{10}^s$};
    
    \edge [dashed, dash pattern=on 2pt off 2pt] {x1} {x1s} ; 
    \edge {x1} {x2} ; 
    \edge [dashed, dash pattern=on 2pt off 2pt] {x2} {x2s} ; 
    \edge {x2} {dummy} ; 
    \edge {dummy} {x9} ; 
    \edge [dashed, dash pattern=on 2pt off 2pt] {x9} {x9s} ; 
    \edge {x9} {x10} ; 
    \edge [dashed, dash pattern=on 2pt off 2pt] {x10} {x10s} ; 

    \path (x1s) -- (x2s) node[midway,below,yshift=-0.8cm] (corr12) {$0.75$};
    \path (x2s) -- (x3s) node[midway,below,yshift=-0.8cm] (corr23) {$0.86$};
    \path (x3s) -- (x9s) node[midway,below,yshift=-0.8cm] (corr39) {$0.98$};
    \path (x9s) -- (x10s) node[midway,below,yshift=-0.8cm] (corr910) {$0.98$};

    \node[left=of corr12, yshift=-0.00cm, xshift=0.8cm] {\hspace*{-20pt}${\lvert \rho\rvert}$:};

    \end{tikzpicture}
    \vspace{-6pt}
    \caption{Standardized SCM}\label{fig:figure1-standardized-SCM}
\end{subfigure}
\hfill
\begin{subfigure}[c]{0.44\textwidth}
    \centering
    \hspace*{5pt}
    \begin{tikzpicture}[x=0.55cm,y=0.55cm]
    
    \node[latent] (x1) {$x_1$};
    \node[obs, below=of x1] (x1s) {$\as{x}_1$};
    \node[latent, right=of x1] (x2) {$x_2$};
    \node[obs, below=of x2] (x2s) {$\as{x}_2$};
    \node[const, right=of x2]  (dummy) {$~~\dots~~$};
    \node[const, below=of dummy]  (dummycenter) {\hphantom{$~~\dots~~$}};
    \node[const, left=of dummycenter, xshift=0.4cm]  (dummyleft) {$~$};
    \node[const, right=of dummycenter, xshift=-0.4cm]  (dummyright) {$~$};
    \node[const, right=of x2s]  (x3s) {$~~\dots~~$};
    \node[latent, right=of dummy] (x9) {$x_9$};
    \node[obs, below=of x9] (x9s) {$\as{x}_9$};
    \node[latent, right=of x9] (x10) {$x_{10}$};
    \node[obs, below=of x10] (x10s) {$\as{x}_{10}$};
    
    \edge [dashed, dash pattern=on 2pt off 2pt] {x1} {x1s} ; 
    \edge {x1s} {x2} ; 
    \edge [dashed, dash pattern=on 2pt off 2pt] {x2} {x2s} ; 
    \edge {x2s} {dummyleft} ; 
    \edge {dummyright} {x9} ; 
    \edge [dashed, dash pattern=on 2pt off 2pt] {x9} {x9s} ; 
    \edge {x9s} {x10} ; 
    \edge [dashed, dash pattern=on 2pt off 2pt] {x10} {x10s} ; 

    \path (x1s) -- (x2s) node[midway,below,yshift=-0.8cm] (corr12) {$0.75$};
    \path (x2s) -- (x3s) node[midway,below,yshift=-0.8cm] (corr23) {$0.75$};
    \path (x3s) -- (x9s) node[midway,below,yshift=-0.8cm] (corr39) {$0.75$};
    \path (x9s) -- (x10s) node[midway,below,yshift=-0.8cm] (corr910) {$0.75$};

    \node[left=of corr12, yshift=-0.00cm, xshift=0.8cm] {\hspace*{-20pt}${\lvert \rho\rvert}$:};

    \end{tikzpicture}
    \vspace{-6pt}
    \caption{iSCM}\label{fig:figure1-iSCM}
\end{subfigure}
\vspace{1pt}
\caption{ 
\textbf{Standardizing SCMs two ways.} Generative process for a chain graph of (a) standard SCMs, with data $\bf{x}$ standardized post-hoc, and (b) SCMs with standardization performed during the generative process (iSCMs).
Dashed arrows indicate z-standardization.
Solid arrows indicate linear functions with weights from $\smash{\unif_{\pm}[0.5, 2.0]}$ and additive noise from $\smash{\mathcal{N}(0, 1)}$.
We report absolute correlations $\smash{\lvert\rho\rvert}$ of two consecutive observed variables, (a) $\s{x_j}$ and $\s{x_{j+1}}$, or (b) $\smash{\as{x}_{j}}$ and $\smash{\as{x}_{j+1}}$, averaged over \num{100000} models.
In standard SCMs (a), correlations tend to increase along the causal ordering.
\vspacefigurecaptionbottom
\vspace{-3pt}
}
\label{fig:figure1}
\end{figure}
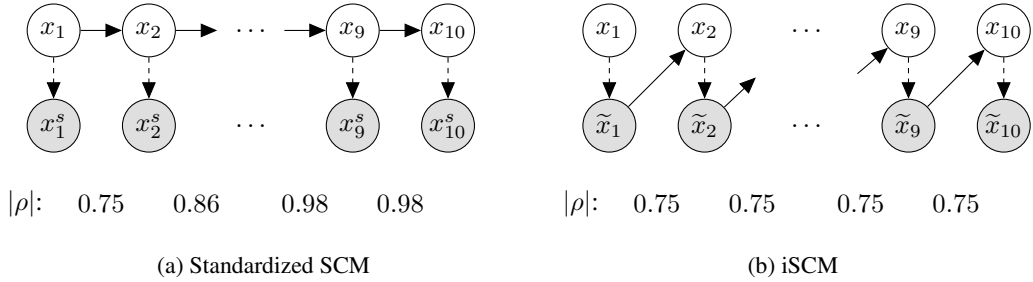

\vsection{Background and Related Work}\label{sec:background}
We begin by introducing structural causal models and the problem of causal structure learning, before discussing how synthetic data is often generated for evaluating structure learning algorithms. 
We then review existing works that study identifiability and patterns frequently present in synthetic data.

\vparagraph{Structural causal models}
A structural causal model (SCM) \citep{peters2017elements}
of $\numvars$ variables $\mathbf{x} = \{x_1, \dots, x_\numvars\}$ consists of a collection of structural assignments, each given by 
\begin{equation}
\label{eq:scm}
\tag{SCM}
    x_i := f_i(\pa{i}, \noise{i}) \, ,
\end{equation}
where $\pa{i} \subseteq \mathbf{x} \setminus \{x_i\}$ are called the {\em parents} of $x_i$.
Here, $f_i$ are arbitrary functions, and $\noise{i}$ are independent random variables that model exogenous noise (or unexplained variation).
Together, they entail a joint probability distribution $p(\mathbf{x})$ over the variables $\bf{x}$.
It is common to consider SCMs with additive noise, e.g., with linear functions $f_i$, as given by
\begin{equation}
\label{eq:linear_scm}
   f_i(\pa{i}, \noise{i}) = \mathbf{w}_i^\top \pa{i} + \noise{i} \, ,
\end{equation}
\rebuttal{where $w_{i,j} \in \RR$ denotes the weight from $j \in \paset{i}$ to $i$.}
The structural assignments in \eqref{eq:scm} induce a causal graph $\mathcal{G} = (\nodeset, \edgeset)$ over the variables $x_i$, which is assumed to be acyclic.
Specifically, the directed acyclic graph (DAG)  $\mathcal{G}$ has vertices $v_i \in \nodeset$ for every $x_i \in \mathbf{x}$ and a directed edge $(i,j) \in \edgeset$ if $x_i \in \pa{j}$.
We will explicitly distinguish this DAG $\mathcal{G}$ and its vertices $\nodeset$ from the variables $\mathbf{x}$.
The {\em skeleton} of $\mathcal{G}$ denotes $\mathcal{G}$ with all edges undirected.
If the skeleton of $\mathcal{G}$ is acyclic, we call  $\mathcal{G}$ a {\em forest}.

\vparagraph{Structure learning and benchmarking}
Given a set of \iid~observations from the probability distribution $p(\mathbf{x})$ induced by an unknown SCM, causal structure learning aims to infer the causal graph $\mathcal{G}$ underlying the SCM. 
In this work, we focus on structure learning from observational data and only consider SCMs with no latent confounders. 
Because it is difficult to obtain the true $\mathcal{G}$ for many real-world datasets, it is common to evaluate structure learning algorithms on synthetic data where $\mathcal{G}$ is known. A ubiquitous approach is to sample a DAG $\mathcal{G}$, then SCM functions defined over $\mathcal{G}$, 
and finally a dataset from this SCM, with the goal of later recovering $\mathcal{G}$ from the data.
It is common to consider $\noise{i}$ with mean \num{0} and fixed variance (often \num{1}), and for linear systems, to sample each $w_{i,j}$ uniformly and \iid with support bounded away from \num{0} \citep{shimizu2011directlingam,peters2014identifiability,zheng2018dags,yu2019dag,lachapelle2020gradient,zheng2020learning,ng2020role,reisach2021beware,lorch2022amortized,reisach2024scale}. 
There exist alternative benchmarking strategies with domain-specific simulators \citep{schaffter2011genenetweaver,dibaeinia2020sergio}.

\vparagraph{Data standardization and \artefacts of SCMs}
Previous work shows that generating data as described above can lead to strong \artefacts.
\citet{reisach2021beware} observe that the variance of variables tends to increase along the topological ordering of $\mathcal{G}$. 
This leads to the \varsortregress baseline, which sorts variables based on their empirical variance and then performs sparse regression to infer $\mathcal{G}$.
\citet{seng2024learning} show that structure learning algorithms minimizing an MSE-based loss \citep[\eg,][]{zheng2018dags} can identify $\mathcal{G}$ under similar conditions.  
Therefore, \citet{reisach2021beware} propose using standardization (\cref{fig:figure1-standardized-SCM}) to remove this variance \artefact from benchmarks.
Specifically, they first sample all $x_i$ according to a standard \ref{eq:scm} and then {\em post-hoc} transform the variables as 
\begin{align}
\label{eq:standardized_scm}
    \hspace*{81pt} 
    \s{x_i} := \frac{x_i - \Expempty[x_i]}{\sqrt{\var{x_i}}} \, ,
\tag{Standardized SCM}
\end{align}
\looseness-1
such that our observations correspond to samples from $p(\s{\bf{x}})$.
Standardization, however, only removes the variance \artefact.
Even in standardized SCMs,
the fraction of a variable's variance that is explained by all others, measured by the coefficient of determination \rtwo, tends to increase along the topological ordering \citep{reisach2024scale}. 
\rtwosortregress exploits this correlation \artefact analogously to \varsortregress.
Existing heuristics aiming to avoid the variance accumulation adjust the sampling process of $f_i$,
but they ultimately limit the causal dependencies that can be modeled, e.g., 
to certain levels of correlations among the observed $\mathbf{x}$ 
\citep{mooij2020joint} 
or a constant proportion of variance explained by the parents $\pa{i}$  
\citep{squires2022causal} and fail to induce data free from both \artefacts
(Appendix \ref{sec:background-heuristics}).
To our knowledge, there are currently no general methods for generating SCM data without strong correlation \artefacts or significant limitations on the functions $f_i$ and noise $\noise{i}$.

\vparagraph{Identifiability}
Given a class of SCMs, there may be several SCMs with different causal graphs $\mathcal{G}$ that entail the same distribution $p(\mathbf{x})$ \citep{peters2017elements}. 
Thus, even with infinite observations from $p(\mathbf{x})$, we may be unable to identify the causal graph $\mathcal{G}$ that generated the observations.
However, some identifiability results are known depending on the class of functions and noise distributions of the SCMs considered. 
For example, among all linear SCMs \eqref{eq:linear_scm} with Gaussian noise $\noise{i} \sim \mathcal{N}(0, \sigma^2_i)$, 
the graph $\mathcal{G}$ can only be uniquely identified up to its MEC \citep{verma2013equivalence}.
However, if the noise \rebuttal{is Gaussian with} equal variances $\sigma^2_i = \sigma^2$ \citep{peters2014identifiability} or the noise is non-Gaussian \citep{shimizu2006linear}, $\mathcal{G}$ can be uniquely identified given $p(\bf{x})$.

In this work, we present, to our knowledge, the first (partial) identifiability result for {\em standardized} SCMs \rebuttal{in the linear Gaussian case}. 
\rebuttal{Since standardization affects the {\em implied} noise scales, 
existing linear Gaussian identification results, which rely on $\sigma^2_i = \sigma^2$,
no longer hold when observing $p(\s{\bf{x}})$}.
\rebuttal{Other identifiability results, e.g., based on non-Gaussian noise, do continue to hold for standardized SCMs \citep[e.g.,][]{shimizu2006linear}.}  
Our result concerns a setting with prior knowledge on the magnitudes of $\textbf{w}$ in \cref{eq:linear_scm}, an assumption underlying common benchmarking practices. \rebuttal{Under this setup, we show a stark difference in the identifiability of standardized SCMs and the iSCMs we propose, which provides a novel explanation for what we empirically observe in benchmarks.}

\vsection{SCMs with Internal Standardization}\label{sec:iscm}

\vsubsection{Definition}

\label{sec:def_ours}
We propose {\em internally-standardized SCMs (iSCMs)} as a modification to the standard data-generating process of SCMs. 
An \ours 
$(\mathbf{S},\distnoisejoint)$ 
consists of $\numvars$ pairs of assignments,
where for each $i \in \{1, \dots, \numvars\}$, 
\begin{align}
\label{eq:iscm}
    x_i := f_i(\paours{i}, \noise{i})
    ~~~~\text{and}~~~~
    \as{x}_i := \frac{x_i - \Expempty[x_i]}{\sqrt{\var{x_i}}}
\tag{iSCM}
\end{align}
with parents $\smash{\paours{i} \subseteq \as{\bf{x}} \setminus \{\as{x}_i\}}$ of $\as{x}_i$ in the underlying DAG. 
\rebuttal{In the above, $f_i$ are general functions, and}
the exogenous noise variables $\Noise = [\noise{1}, ..., \noise{\numvars}] \sim \distnoisejoint$ are jointly independent, \rebuttal{as for SCMs}.
The variables $x_i$ are latent, and the variables $\smash{\as{x}_i}$ are observed. 
Figure \ref{fig:iscm-causal-mechanism} illustrates the generative process.
Algorithm \ref{alg:iscm} summarizes how to sample from \eqref{eq:iscm}.
If computing the population expectations and variances of $x_i$ is intractable, the empirical statistics obtained from $\numsamples$ samples can be used for standardization at each loop iteration of Algorithm \ref{alg:iscm}.

\vparagraph{Motivation}
By construction, \ourss model observed variables with zero mean and unit marginal variance.
Contrary to standard SCMs, \ourss avoid the accumulation of variance downstream in the causal ordering that can occur in standard SCMs (see Figure \ref{fig:figure1}) through the standardization operation.
Because each variable $x_i$ only depends on the standardized variables $\paours{i}$, the relative scales of the noise distribution $\distnoise{i}$ and the causal mechanisms $f_i$ are the same everywhere in the system and do not change, for example, downstream in the causal ordering.
The causal mechansims of \ourss are thus 
{\em scale-free}, 
in that the local interaction of mechanism $f_i$ and noise $\noise{i}$ occurs at a scale independent of the position of $x_i$ in the global ordering.
This property makes \ourss particularly useful for benchmarking, where random ground-truth models are commonly generated from a fixed distribution over functions $f_i$ and noise $\noise{i}$.
Contrary to existing heuristics (Section \ref{sec:background}), \ourss model arbitrarily strong or weak causal dependencies and levels of cause-explained variance.

\begin{figure}[t]
    \vspacefiguretop
    \vspace*{-10pt}
    \centering
    \begin{minipage}[t]{0.49\textwidth}
        \vspace{0pt} 
        \centering
            \centering
        \begin{tikzpicture}[x=0.55cm,y=0.55cm]
        
        \tikzset{grey edge/.style={draw=arrowlightgrey}}
        
        \node[latent] (xi) {$x_i$};     
        \node[obs, right=of xi, xshift=0.13cm] (chi) {${\widetilde{x}_{i}}$};
        
        \node[const, left=of xi, xshift=-0.2cm] (pai) {$~~\smash{\raisebox{-1ex}{\vdots}}~~~~$};
        \node[obs, above=of pai, xshift=0.35cm] (pa1) {${\widetilde{x}_{j}}$};
        \node[obs, below=of pai, xshift=0.35cm] (pa2) {${\widetilde{x}_{k}}$};
    
        \node[const, right=of chi, yshift=0.25cm] (chichi1) {$\hphantom{~~\dots~~}$};
        \node[const, right=of chi, yshift=-0.25cm] (chichi2) {$\hphantom{~~\dots~~}$};
    
        \node[const, left=of pa1, xshift=-0.2cm] (pa1pa) {$\hphantom{~~\dots~~}$};
        \node[const, left=of pa2, xshift=-0.2cm] (pa2pa) {$\hphantom{~~\dots~~}$};
    
        \edge {pa1} {xi} ; 
        \edge {pai} {xi} ; 
        \edge {pa2} {xi} ; 
        \edge [dashed, dash pattern=on 2pt off 2pt] {xi} {chi} ; 
    
        \edge [grey edge]  {chi} {chichi1} ; 
        \edge [grey edge]  {chi} {chichi2} ; 
        \edge [dashed, dash pattern=on 2pt off 2pt, grey edge]  {pa1pa} {pa1} ; 
        \edge [dashed, dash pattern=on 2pt off 2pt, grey edge]  {pa2pa} {pa2} ; 
        
        \node[below=1.95cm of pa1, xshift=0.60cm] (underbrace) {$\underbrace{\hspace{5.2em}}_{\displaystyle f_i}$};
    
        \end{tikzpicture}
        \vspace{-3pt}
        \caption{\textbf{Causal mechanisms in iSCMs.} The function $f_i$ modeling $x_i$ depends on the standardized $\smash{\paours{i}}$. Dashing indicates z-standardization.}
        \label{fig:iscm-causal-mechanism}
      \end{minipage}%
      \hfill 
      \begin{minipage}[t]{0.47\textwidth}
        \vspace{-13pt} 
        \centering
        \input{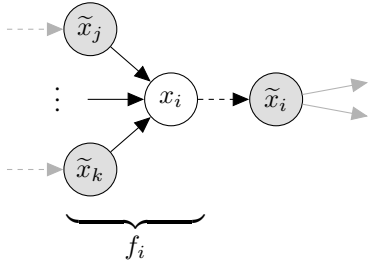}
      \end{minipage}
    \vspace*{5pt}
    \vspacefigurecaptionbottom
\end{figure}

\vparagraph{Interventions}

Analogous to standard SCMs,
interventions in \ourss can be defined as modifications of the structural assignments $f_i$ in \eqref{eq:iscm} (Figure \ref{fig:iscm-causal-mechanism}),
while keeping the standardization operation based on the observational distribution.
When the population statistics for standardization are intractable, we first sample observational data to obtain empirical statistics.
Since we do not study interventions in this work, we defer a further discussion of interventions in \ourss to Appendix~\ref{sec:interventionssss_iscm}.

\vparagraph{Units}
When modeling a physical system, the functional mechanisms in standard SCMs have to account for the difference in units between the variables for the model to be {\em unit-covariant} (see \citealp{villar2023towards}).
A side-effect of internal standardization is that variables of \ourss become {\em unit-less}, so \ourss obey the passive symmetry of unit covariance by construction.
Therefore, \ourss naturally model both unit-less quantities and variables measured in different units, which can make them useful beyond benchmarking.
Learned \ourss would be invariant to the units chosen by the experimenter, similar to the physical world being independent of the mathematical models chosen to describe it.

\vsubsection{Implied SCMs}\label{ssec:implied-scms-main}

It is natural to investigate whether SCMs can generate the same observations as standardized SCMs or \ourss, given the same causal graph $\mathcal{G}$ and exogenous variables $\Noise$.
In other words, can standardized SCMs and \ourss be written as SCMs?
For both models, the answer is yes.
Specifically, we can express the generative process of
$\smash{\s{\bf{x}}}$ in \eqref{eq:standardized_scm}
and $\smash{\as{\bf{x}}}$ in \eqref{eq:iscm}
as
\begin{align}
    \label{eq:implied-scms-overview}
    \s{x}_i = \s{g}_i(\s{\pa{i}}) + \s{\theta}_i \noise{i}
    \quad\quad\quad
    \text{and}
    \quad\quad\quad
    \as{x}_i = \as{g}_i(\paours{i}) + \as{\theta}_i \noise{i} \, ,
\end{align} 
respectively, by moving the standardization operations into the causal mechanisms of the observables but leaving the DAG $\mathcal{G}$ and the variables $\Noise$ unchanged.
Appendix \ref{sec:implied-models} describes how to construct these 
{\em implied causal mechanisms} $\smash{\s{g}_i}$ and $\smash{\as{g}_i}$
and {\em implied noise scales} $\smash{\s{\theta}_i}$ and $\smash{\as{\theta}_i}$.
We refer to the above SCM form of a standardized SCM or an \ours with additive noise as their {\bf \em implied (SCM) model}. 
Correspondingly, the implied SCMs have zero mean and unit variance.
The notion of implied SCMs is powerful, because it enables us to analyze standardized SCMs and \ourss as SCMs, and it sheds light on the performance of structure learning algorithms that assume unstandardized SCMs to underlie the generative process of the data (e.g., \citealp{shimizu2011directlingam,zheng2018dags,yu2019dag,lachapelle2020gradient,zheng2020learning}).

To provide a first characterization of standardized SCMs and \ourss, our theoretical analyses focus on systems where $f_i$ are linear functions with additive, zero-mean noise as given by Equation \eqref{eq:linear_scm}.
As a stepping stone for this analysis, we use an analytical expression for the covariance of linear SCMs, whose variables have unit variance by construction, without any form of standardization:

\medskip 

\begin{restatable}[Covariance in linear SCMs with unit marginal variances]{lemma}{covform}
\label{lemma:cov}
Let $\bf{x}$ be modeled by a linear SCM defined by \eqref{eq:linear_scm} with DAG $\mathcal{G}$ that satisfies $\var{x_i} = 1$.
Then, the covariance $\cov{x_i, x_j}$ is the sum of products of the weights along all unblocked paths between the nodes of $x_i$ and $x_j$ in $\mathcal{G}$.
Specifically, for any $i, j \in \numvarset$ such that $i\neq j$, it holds that
\begin{equation}
\label{eq:cov_form}
    \cov{x_i, x_j} = \sum_{\ugpath{j}{i} \in \ugpathset{j}{i}}\prod_{(l, m) \in \ugpath{j}{i}}w_{l, m} \, ,
\end{equation}
where $\ugpathset{j}{i}$ are all unblocked paths from $x_j$ to $x_i$ in $\mathcal{G}$, and $(l, m) \in \ugpath{j}{i}$ indicates that the directed edge $(l,m)$ is part of the path $\ugpath{j}{i}$.
\end{restatable}

This lemma, also called the {\em trek rule}, is
originally due to \citet{wright1934method}.
We give a proof in Appendix \ref{ssec:cov_formula}.
Since the implied SCMs of linear standardized SCMs and \ourss are linear SCMs, the setting of Lemma \ref{lemma:cov} applies precisely to the SCM forms of both models.
Thus, Lemma~\ref{lemma:cov} enables us to study the covariances in standardized SCMs and \ourss, and as we show next, derive conditions for the (non)identifiability of their DAGs $\mathcal{G}$ from the observational distribution.

\vsection{Analysis}
\label{sec:theory}

\looseness-1
In this section, we give two theoretical results that support the suitability of \ourss over standard SCMs for causal discovery benchmarking.
First, we prove the general case of \cref{fig:figure1}. 
Contrary to standardized SCMs, \ourss do not degenerate towards deterministic implied SCM mechanisms in deep graphs.
Moreover, we prove that the DAGs of linear \ourss cannot be identified beyond their MEC, assuming the DAG is a forest, even if the support of $\bf{w}$ is known. 
Crucially, we also show that this is not generally true for standardized SCMs. 
This suggests that algorithms can less easily game benchmarks based on linear \ourss when knowing the data-generating process.
For all results, we consider linear SCMs \eqref{eq:linear_scm} with zero-mean additive noise and equal noise variances. 
All results are at the population level, so assume we know $p(\s{\bf{x}})$ or $p(\as{\bf{x}})$.
Proofs are given in Appendix \ref{sec:proofs}.

\vsubsection{Behavior with Increasing Graph Depth}
\label{ssec:determinism}

Standardized SCMs tend towards increasing correlations between adjacent nodes down the topological ordering.
This correlation \artefact makes standardized SCMs problematic for benchmarking, because it may not be a property we expect to underlie real data. 
\citet{reisach2024scale} show, under some assumptions on $\bf{w}$, that the dependencies in standardized SCMs become {\em deterministic} with increasing graph depth. 
This implies that any exogenous variation $\noise{i}$ vanishes lower down in the system.
Unless prior domain knowledge leads us to assume this holds in applications of interest, 
it may not be desirable to implicitly bias structure learning benchmarks towards such systems.
For example, if the causal ordering represents time \citep{pamfil2020dynotears}, the mechanisms of standardized SCMs are unable to model or characterize time-invariant or stable processes. 
Moreover, if we expect causal mechanisms to be independent \citep{scholkopf2022causality},
the qualitative behavior of a causal mechanism should not provide information about its position in the topological ordering relative to other mechanisms, as it would in SCMs.
\citet{reisach2024scale} show that baselines like \rtwosortregress can perform competitively on benchmarks by exploiting this \artefact (Section \ref{sec:background}).

\ourss do not tend towards determinism with increasing graph depth (Figure \ref{fig:figure1-iSCM}). 
In standardized SCMs, the correlations increase downstream, because the marginal variances of the underlying SCM increase with node depth, while the variance scale is fixed \citep{reisach2021beware}. 
Thus, for large $i$, the variance scale of $x_{i-1}$ becomes large relative to the scale of $\noise{i}$, and the correlation of $x_i$ and $x_{i-1}$ tends towards $1$. 
Since $x^s_i$ and $x^s_{i-1}$ are just standardized versions of these variables, they  maintain the same correlation. 
\ourss avoid this by standardizing internally, which scales the variance of any parents in a mechanism $f_i$ to \num{1}, modulating the relative variance of $\noise{i}$ and $\smash{\pa{i}}$.
In the following, we formalize this result for general graphs by bounding the fraction of cause explained variance (CEV).
The fraction of CEV for $x_i$ is the proportion of $\varop[x_i]$ explained by its causal parents and given by
\begin{equation}
\label{eq:cev-def}
    \cevf{x_i} =  1 - \frac{\var{x_i - \Expempty[x_i | \pa{i}]}}{\var{x_i}} \, .
\end{equation}
The following results shows that we can bound the fraction of CEV for any variable in a linear \ours:

\medskip 

\begin{restatable}[Bound on  $\smash{\operatorname{CEV_f}}$ in linear \ourss]{theorem}{cevbound}
\label{th:cevbound}
Let $\bf{x}$ be modeled by a linear \ours \eqref{eq:linear_scm} with DAG $\mathcal{G}$ and additive noise of equal variances $\varop[\noise{i}] = \smash{\sigma^2}$. 
Suppose any node in $\mathcal{G}$ has at most $\maxnumparents$ parents and $w = \max_{i, j  \in \numvarset}\lvert w_{i, j} \rvert$.
Then, for any $i \in \numvarset$, the fraction of CEV for $\as{x}_i$ is bounded as 
\begin{align*}
    \cevf{\as{x}_i} \leq 1 - \frac{\sigma^2}{\maxnumparents^2w^2 + \sigma^2} \, .
\end{align*}
\end{restatable}
Since the fraction of CEV is bounded, \ourss are guaranteed not to collapse to determinism in large systems, alleviating several of the concerns with (standardized) SCMs discussed above.

\vsubsection{Identifiability}
\label{ssec:identifiability}
\begin{wrapfigure}[24]{R}{0.37\textwidth}
\vspace*{-10pt}
\centering
\begin{subfigure}{\linewidth}
\centering
\vspace{-2pt}
\hspace{-0.2cm}\begin{tikzpicture}
    \matrix[column sep={1.2cm,between origins}, row sep=0.4cm] (m) {
        \node[const] {$\mathrm{(i)}$}; & 
        \node[latent, xshift=-0.2cm] (x1) {$x_1$}; &
        \node[latent] (x2) {$x_2$}; &
        \node[latent, xshift=0.2cm] (x3) {$x_3$}; \\
        
        \node[const] {$\mathrm{(ii)}$}; & 
        \node[latent, xshift=-0.2cm] (y1) {$x_1$}; &
        \node[latent] (y2) {$x_2$}; &
        \node[latent, xshift=0.2cm] (y3) {$x_3$}; \\
        
        \node[const] {$\mathrm{(iii)}$}; & 
        \node[latent, xshift=-0.2cm] (z1) {$x_1$}; &
        \node[latent] (z2) {$x_2$}; &
        \node[latent, xshift=0.2cm] (z3) {$x_3$}; \\
    };
    
    \path[->] (x1) edge node[above, midway, yshift=0.1cm] {$\alpha$} (x2);
    \path[->] (x2) edge node[above, midway, yshift=0.1cm] {$\beta$} (x3);

    \path[->] (y2) edge node[above, midway, yshift=0.1cm] {$\alpha$} (y1);
    \path[->] (y2) edge node[above, midway, yshift=0.1cm] {$\beta$} (y3);

    \path[->] (z2) edge node[above, midway, yshift=0.1cm] {$\alpha$} (z1);
    \path[->] (z3) edge node[above, midway, yshift=0.1cm] {$\beta$} (z2);
    
\end{tikzpicture}
\vspace*{1pt}
\caption{DAGs with edge weights $\alpha$ and $\beta$}\label{fig:3-path-id-a}
\end{subfigure}

\vspace{10pt}

\begin{subfigure}{0.75\linewidth}
    \centering
    \hspace*{10pt}\includegraphics[width=0.85\textwidth]{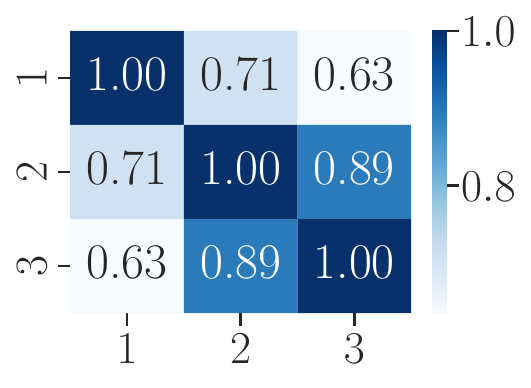}
    \vspace{-3pt}
    \caption{Cov.\ matrix of linear \ourss}\label{fig:3-path-id-b}
\end{subfigure}

\vspace{2pt}

\captionsetup{width=0.94\linewidth} 
\caption{%
\textbf{\ourss with the same covariance matrix.}
(a) DAGs in an MEC with the same edge weights.
(b) Covariance matrix for all linear \ourss in (a)
when $\alpha = 1$, $\beta = 2$.
}
\label{fig:3-path-id}

\end{wrapfigure}
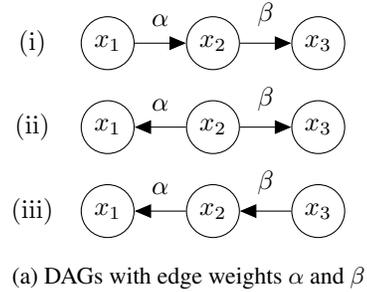
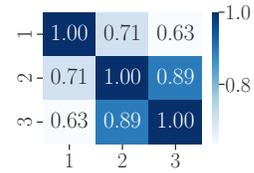

\cref{fig:figure1-standardized-SCM} illustrates that the pairwise correlations in SCMs over chain graphs depend on the position in the topological ordering.
This can allow algorithms like \rtwosortregress to infer the graph. 
By contrast, \cref{fig:figure1-iSCM} shows that \ourss do no exhibit this pattern, with correlations between variables not increasing the identifiability of any part of the system.

In the following, we formalize this phenomenon for forests, that is, all DAGs with acyclic skeletons
(Section \ref{sec:background}).
Specifically, we prove two results concerning the identifiability of the DAG $\mathcal{G}$ from the observational distribution, for standardized SCMs and \ourss.  
This makes our finding the first identifiability result for {\em standardized} SCMs.
While not every DAG is a forest, DAGs have forests as subgraphs and resemble forests as sparsity increases, thus providing us with intuition for generally sparse systems 
(e.g., \citealp{alon2016probabilistic}, Chapter 11).

Our first result leverages the observation that, for standardized SCMs, many DAGs in an MEC are infeasible given $p(\s{\bf{x}})$ when their edge directions are not consistent with the direction of increasing absolute covariance.
To illustrate this idea, suppose our goal is to distinguish between the DAGs in the MEC $\mec = \{\mathrm{(i)}, \mathrm{(ii)}, \mathrm{(iii)}\}$
in Figure \ref{fig:3-path-id-a}.
We overload notation and denote the weights of the edges $\alpha$ and $\beta$ regardless of orientation. 
For standardized SCMs, we can apply \cref{lemma:cov} to the implied SCM of graph $\mathrm{(i)}$ to obtain the covariances
\begin{align*}
   \cov{\s{x_1}, \s{x_2}} = \tfrac{\alpha}{\sqrt{\alpha^2 + 1}}
   \quad\quad\text{and}\quad\quad
   \cov{\s{x_2}, \s{x_3}} = \beta \sqrt{\tfrac{\alpha^2 + 1}{\beta^2(\alpha^2+1)+1}} \, .
\end{align*}
See Appendix \ref{ssec:ident_3node}. Together, both expressions imply that standardized SCMs with DAG $\mathrm{(i)}$ satisfy
\begin{equation}
\label{eq:cov_weight_rel}
\lvert\cov{\s{x_1}, \s{x_2}}\rvert < \lvert\cov{\s{x_2}, \s{x_3}}\rvert \quad\Longleftrightarrow\quad  \tfrac{\alpha^2}{\alpha^2+1} < \beta^2 \, .
\end{equation}
If $\lvert \beta \rvert \geq 1$, then the right-hand side of \cref{eq:cov_weight_rel} is always true.
In this case, the absolute covariance increases from $x_1$ to $x_3$ in all standardized SCMs with DAG $\mathrm{(i)}$. 
By symmetry, the covariance in SCMs with DAG $\mathrm{(iii)}$ increases from $x_3$ to $x_1$ when $\lvert \alpha \rvert \geq 1$.
Therefore, if both weights are greater than \num{1}, the absolute covariance increases downstream in all SCMs of $\mathrm{(i)}$ and  $\mathrm{(iii)}$.
This implies that, among $\mathrm{(i)}$ and  $\mathrm{(iii)}$, only the DAG whose edges align with the covariance ordering in $p(\s{\bf{x}})$ can induce $p(\s{\bf{x}})$.
Irrespectively, the DAG $\mathrm{(ii)}$ remains plausible.
We can extend the intuition of this 3-variable example to identify almost all edges in any forest MEC:

\medskip 

\begin{restatable}[Partial identifiability of standardized linear SCMs with forest DAGs]{theorem}{forestpartident}
\label{th:three_part_ident}
Let $\s{\bf{x}}$ be modeled by a standardized linear SCM \eqref{eq:linear_scm} with forest DAG $\mathcal{G}$, additive noise of equal variances $\var{\noise{i}} = \sigma^2$, and $\abs{w_{i,j}}>1$ for all $i \in \text{pa}(j)$. 
Then, given $p(\s{\bf{x}})$ and the partially directed graph $\mec$ representing the MEC of $\mathcal{G}$,
we can identify all but at most one edge of the true DAG $\mathcal{G}$ in each undirected connected component of the MEC $\mec$.
\end{restatable}
Our proof of \cref{th:three_part_ident} considers each undirected component separately from the rest of the MEC~$\mec$.
Hence, the identifiability result extends to undirected tree components of arbitrary, non-forest MECs as well.
\cref{th:three_part_ident} shows that, when using standardized SCM data for benchmarking, algorithms can use pairwise correlations to orient additional edges correctly.
The weights assumption of \cref{th:three_part_ident} is relevant to causal discovery benchmarking, because weights are often sampled \iid from intervals bounded away from $0$ (\cref{sec:background}).
Hence, empirical evaluations may render standardized linear SCMs identifiable only through the design of their weights distribution.
In the following, we show that, under similar conditions, \ourss are more difficult to identify from their MEC.
In the 3-variable example above, we can show that the observational distribution of \ourss is the same for all DAGs $\mathrm{(i)}$, $\mathrm{(ii)}$, and $\mathrm{(iii)}$ when the weights $\alpha$ and $\beta$ are shared over the corresponding  edges in the MEC (Figure \ref{fig:3-path-id-b}; see Appendix \ref{ssec:non-ident-three-nodes}). 
This result generalizes to forests:

\medskip 

\begin{restatable}[Nonidentifiability of linear Gaussian \ourss with forest DAGs]{theorem}{forestnonident}
\label{th:three_non_ident}
Let $\as{\bf{x}}$ be modeled by a linear \ours \eqref{eq:linear_scm} with forest DAG $\mathcal{G}$ and additive Gaussian noise of equal variances  $\var{\noise{i}}$.
Then, for every DAG $\mathcal{G}'$ in the MEC of $\mathcal{G}$, there exists a linear \ours with DAG $\mathcal{G}'$ that has the same observational distribution as $\as{\bf{x}}$, the same noise variances, and the same weights on the corresponding edges in the MEC.
\end{restatable}
Our proof consists of showing that the covariance matrices of these systems are equal. 
For linear Gaussian \ourss, this then implies that their observational distributions are identical.
\cref{th:three_non_ident} thus shows that additional knowledge of the weight distribution in a benchmark does not allow identifying any additional edges beyond the MEC. 
By contrast, \cref{th:three_part_ident} shows that, for standardized SCMs, lower-bounding the weight magnitudes is sufficient for identifying most of the graph from its MEC. 
Without standardization, $\mathcal{G}$ is fully identified from its observational distribution under even weaker assumptions \citep{peters2014identifiability}. 
Importantly, Theorem \ref{th:three_non_ident} does not generalize to arbitrary graphs beyond forests. 
Appendix \ref{rem:3node_counterexample} provides a counterexample involving a 3-node skeleton.
As we study in the next section, this implies that causal structure can still be learned from nontrivial \ourss.
However, DAGs in benchmarks are often sparse, so we expect  the implications of our identifiability results to capture relevant parts of empirical phenomena in benchmarking settings. 

\vsection{Experimental Results}
\label{sec:experiments}

\looseness-1
Our analyses suggest that \ourss address shortcomings of naive standardization, in particular, when sampling each $f_i$ and $\noise{i}$ from the same distribution, as common in benchmarking.
In this section, we now also provide empirical evidence that \ourss do not contain the covariance \artefacts of SCMs.
\rebuttal{This makes \ourss a useful tool for disentangling, alongside SCMs, which data patterns drive causal discovery in practice.}
\rebuttal{To show this,} we benchmark the \sortregress baselines and a suite of popular structure learning algorithms to gain insights into how their performance varies when benchmarked on standardized SCMs and \ourss.
Appendix \ref{sec:csl-setup} provides complete details of the experimental setup.

\vsubsection{\rtwo-Sortability} \label{ssec:experimentas-rtwo}

\cite{reisach2024scale} introduce the \rtwo-sortability metric to evaluate the correlation \artefact underlying a dataset.
\rtwo-sortability measures (rescaled to $[0, 1]$) the association between the variables' causal ordering and the \rtwo coefficients obtained from regressing each variable onto all others (Appendix \ref{sec:sortabilities-def}). 
An \rtwo-sortability close to $0.5$ suggests that the \rtwo coefficients from regression contain no information about the causal ordering.
Conversely, an \rtwo-sortability  of $0$ or $1$ implies that the causal ordering can be completely identified from this information. 
The metric gives rise to the \rtwosortregress baseline described in Section \ref{sec:background}.
\citet{reisach2024scale} show that \rtwo-sortability in SCMs is driven by an interplay of graph connectivity and the weight distribution of $f_i$.

\begin{figure}
    \vspacefiguretop
    \centering
    \vspace{-5pt}
    \includegraphics[width=\linewidth]{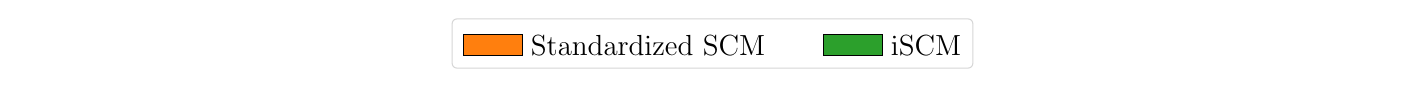}
    \vspace{-15pt}

    \begin{subfigure}{0.49\textwidth}
    \centering
    \hspace{-13pt}\includegraphics[width=\linewidth]{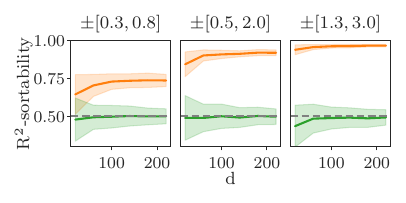}
    
    \vspace{-5pt}
    \caption{$\er{\numvars, 2}$}
    \label{fig:r2_sortability_er2}

    \end{subfigure}
    \hfill
    \begin{subfigure}{0.49\textwidth}
    \centering
    \hspace{-13pt}\includegraphics[width=\linewidth]{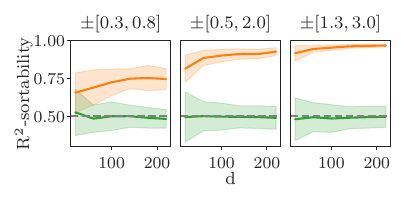}
    
    \vspace{-5pt}
    \caption{$\usf{\numvars, 2}$}
    \label{fig:r2_sortability_sf2}
    \end{subfigure}
    \caption{\textbf{\rtwo-sortability for different graph sizes.}
    Linear standardized SCMs and \ourss with $\noise{i} \sim \mathcal{N}(0, 1)$ and weights drawn from uniform distributions with supports given above each plot. 
    For every model, we evaluate \num{100} systems and $\numsamples =$\num{1000} samples each.
    Lines and shaded regions denote mean and standard deviation.
    Datasets that satisfy \rtwo-sortability $=0.5$ (dashed) are not \rtwo-sortable.
    \vspacefigurecaptionbottom
    }
    \label{fig:r2-sortability}
\end{figure}

\looseness-1
Figure \ref{fig:r2-sortability} summarizes the \rtwo-sortability statistics for linear SCM and \ours data.
We write $\er{\numvars, k}$ and $\usf{\numvars, k}$ to denote \erdosrenyi and scale-free graphs of size $\numvars$ and (expected) degree $k$, respectively (see Appendix \ref{ssec:experiment-configuration} for details). 
We find that \ourss generate datasets that are not \rtwo-sortable (\rtwo-sortability $\approx$ \num{0.5}) and thus \artefact-free while sampling over common graph structures
(e.g., \citealp{zheng2018dags,yu2019dag,reisach2021beware}).
Conversely, standardized SCMs generate datasets that are strongly \rtwo-sortable
($\lvert \text{\rtwo-sortability} - 0.5 \rvert \gg 0$).
Since \rtwo-sortability can be exploited for causal discovery, \ours data serves as a test for evaluating whether algorithms utilize any data properties beyond the association between \rtwo and the causal ordering in SCMs.
Our results do not exclude the possibility of \ours configurations that still produce \rtwo-sortable datasets.
However, we show empirically that, for commonly-used $\mathcal{G}$, $\distnoisejoint$, and $\bf{w}$, \ours datasets are not \rtwo-sortable with high probability.
Appendix \ref{sec:background-heuristics} reports the sortability metrics of the existing heuristics in Section \ref{sec:background}, showing that neither mitigate both $\varop$- and \rtwo-sortability.
Appendix \ref{sec:additional-results} provides results for denser graphs.

\vsubsection{Structure Learning}\label{ssec:experimentas-structure-learning}

Under the same weight and noise distributions, standardized SCMs and \ourss have different implied SCMs and generate qualitatively different datasets.
Here, we study how this affects causal structure learning in practice.
For this, we evaluate $\varop$- and \rtwo-\sortregress (\textsc{SR}) \citep{reisach2021beware,reisach2024scale}
as well as an SR variant that uses a random ordering (Random SR).
In addition, we evaluate representative algorithms from several approaches to learning structure from (co)variance information.
\notears by \citet{zheng2018dags} leverages continuous optimization to minimize an MSE loss, which is affected by noise scaling \citep{loh2014high,seng2024learning}.
\golem \citep{ng2020role}, similar to \notears, formulates causal discovery as a continuous optimization problem. 
Its EV and NV versions assume equal and potentially unequal noise scales, respectively. 
\rebuttal{\cam \citep{buhlmann2014cam} searches over causal orderings and performs sparse nonlinear regression to find the parents, while also estimating the noise scales.}
\pc \citep{spirtes1991algorithm} and \ges \citep{chickering2002optimal} are approaches based on statistical independence testing and greedy search, respectively. 
Finally, \avici by \citet{lorch2022amortized} predicts graphs using a model pretrained on  simulated data and is thus optimized to exploit any \artefacts that improve predictive accuracy. 
To investigate its susceptibility to \artefacts, we evaluate the public model checkpoints trained on standardized SCMs.

\newcommand{\yaxislabel}[4]{\hspace{#1}\makebox[0pt][r]{\raisebox{#2}{\rotatebox{90}{\color{matplotlibblack}#3}}}\hspace{#4}}

\begin{figure}
    \vspacefiguretop
    \centering

  \hspace{20pt}\includegraphics[width=0.72\linewidth]{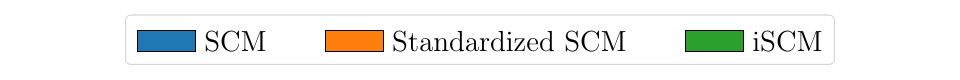}
        \vspace{-5pt}
    
        \begin{tabular}{cc}
            \yaxislabel{0pt}{15pt}{$\er{20, 2}$}{-10pt} &
        \begin{subfigure}{0.96\linewidth}
        \includegraphics[width=0.52\linewidth, trim=0 6pt 7pt 0 , clip]{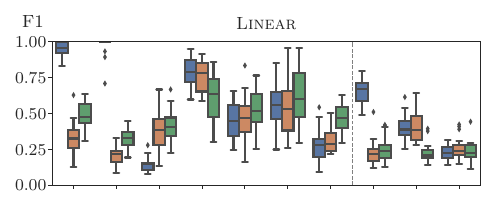}
            \hspace{5pt}\includegraphics[width=00.462\linewidth, trim=25pt 6pt 7pt 0 , clip]{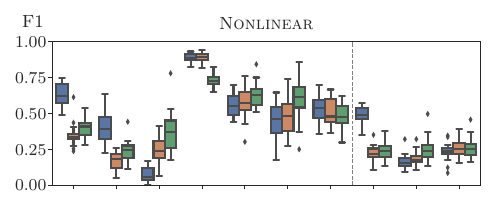}
       \end{subfigure}     
            \\
            \yaxislabel{0pt}{45pt}{$\er{100, 2}$}{-10pt} &
            \begin{subfigure}{0.96\linewidth}\includegraphics[width=0.52\linewidth, trim=0 0 7pt 6pt , clip]{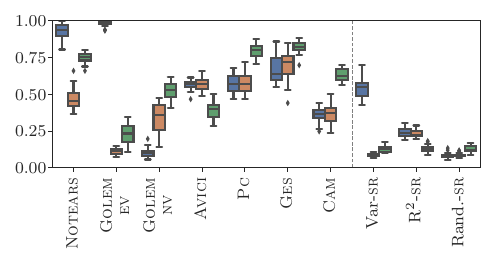}
            \hspace{5pt}\includegraphics[width=00.462\linewidth, trim=25pt 0 7pt 6pt , clip]{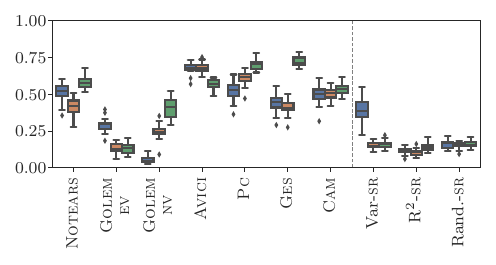}
            \end{subfigure}\\
        \end{tabular}
        \vspace{-6pt}
    \caption{%
    \looseness-1
    \textbf{Structure learning performance on SCM and \ours data.}
    F1 scores for recovering the edges of the true graph. 
    Box plots show median and interquartile range (IQR). Whiskers extend to the largest value inside \num{1.5}$\times$IQR from the boxes. Left (right) column shows results for linear (nonlinear) causal mechanisms with additive noise $\noise{i} \sim \mathcal{N}(0, 1)$ and $w_{i,j} \sim \unif_{\pm [0.5, 2.0]}$ (Appendix E). For every model, we evaluate \num{20} systems each using $\numsamples =$\num{1000} data points.
    }%
    \label{fig:results-benchmark-main}
\end{figure}

Figure \ref{fig:results-benchmark-main} summarizes the results for linear and nonlinear systems \rebuttal{with Gaussian noise (see Figure \ref{fig:results-benchmark-nongaussian}, Appendix \ref{sec:additional-results} for non-Gaussian systems)}.
The nonlinear mechanisms $f_i$ are samples from a Gaussian process with squared exponential kernel.
As expected, \varsortregress performs best when SCMs are not standardized.
Likewise, \rtwosortregress performs better on SCMs and standardized SCMs, as \ourss have \rtwo-sortability close to \num{0.5} (Section \ref{ssec:experimentas-rtwo}).
\avici shows the same trend, suggesting it may indeed be exploiting the correlation \artefacts present in its training distribution.
Like \citet{reisach2021beware},
we find that \notears performs best on unstandardized data.  
However, and more interestingly, \notears also performs better on \ourss than on standardized SCMs, especially in linear and larger systems.
As we investigate later on, this gap may be explained by the fact that the implied models of standardized SCMs violate the assumptions of \notears more strongly than \ourss.
Overall, we find \golemev shows the same patterns as \notears, severely underperforming on standardized SCMs and slightly improving the predictive accuracy on iSCMs. 
\rebuttal{\cam and} \golemnv, which do not assume equal noise scales, perform \rebuttal{equally well or better} on standardized data, respectively, \rebuttal{and generally better on \ourss.} 
The poor performance of \golemnv \begin{wrapfigure}[32]{r}{0.37\textwidth} 
    \vspacefiguretop
    \centering
    \vspace{10pt}
    \begin{subfigure}[b]{0.33\textwidth}
        \centering
    \includegraphics[width=0.99\linewidth, trim=2pt 8pt 0 0 , clip]{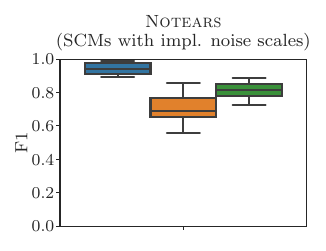}
        \vspace{-2pt}
        
        \hspace{0pt}\includegraphics[width=0.98\linewidth, trim=0 0 0 7pt]{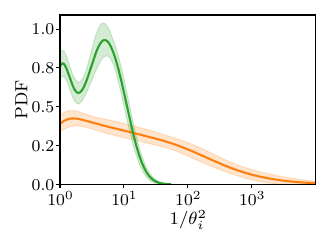}       
        
        \vspace{-4pt}
        \includegraphics[width=\linewidth, trim=0 0 0 6pt]{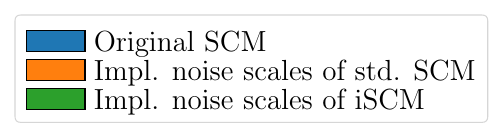}
        \vspace{-4pt}
    \end{subfigure}
    \vspace{-5pt}
    \captionsetup{width=0.95\linewidth} 
    \caption{%
    \textbf{Implied Noise.}
    Bottom panel shows the distribution over inverse implied noise scales 
    in the implied SCMs 
    for $\er{100, 2}$ graphs (kernel density estimate).
    Lines and shading denote mean and standard deviation.
    Top panel shows the performance of NOTEARS on systems with these noise scale statistics but the same $\varop$-sortability as SCMs 
    (see Appendix \ref{ssec:experiment-configuration} and \ref{ssec:adjusting-variances}).
    \vspacefigurecaptionbottom
    }
    \label{fig:results-induced-noise-main}
\end{wrapfigure}
on unstandardized SCMs was also observed by \citet{reisach2021beware}. 
%
%
In addition, for approaches based on discrete search,
we find that, in particular on large systems, the \pc and \ges algorithms  perform better on \ourss. 
Overall, performance differences \rebuttal{tend to be} more pronounced for linear systems, where the downstream variance accumulation in SCMs is unbounded.
Appendix \ref{sec:additional-results} reports the results for the structural Hamming distance (SHD) and different weight ranges.

\vparagraph{Properties of the implied SCMs}
%
When standardizing SCM data, the implied SCM corresponds to the SCM that could have generated the observations.
Therefore, algorithms assuming that unstandardized SCMs generated the data will be susceptible to any assumption violations of the implied SCM, such as assumptions about the exogenous noise.
Figure \ref{fig:results-induced-noise-main} (bottom) shows the distribution of inverse implied noise scales $1 / \theta_i^2$ for the variables of the implied models (see Equation \ref{eq:implied-scms-overview}).
Since $\var{\noise{i}} = 1$ in our experiments, these inverse squared noise scales are equal to the inverse variances of the full additive noise terms.
We find that standardized SCMs induce inverse noise scales that are orders of magnitude greater than those of \ourss.
This distribution is essentially the footprint of the determinism in the depth limit discussed in Section \ref{ssec:determinism}. 
This observation also provides empirical support for our earlier explanation for the improved performance of the \pc algorithm on \ours data.
The modes at $1 /\theta_i^2 = 1$ and at $1 /\theta_i^2 > 1$ in the \ours plot correspond to root and non-root nodes, respectively.


\looseness-1
Figure \ref{fig:results-induced-noise-main} (top) shows the performance of \notears when isolating the noise properties of the implied models from the fact that standardized SCMs and \ourss are not $\varop$-sortable.
For this, we construct SCMs that have the marginal variances (and $\varop$-sortability, here \num{0.99} on average) of unstandardized SCMs but the noise variances of the implied models by correcting their weights (see Appendix \ref{ssec:adjusting-variances}).
\notears performs better in such systems, suggesting that {(i)} the noise statistics may indeed explain the performance difference on \ours data, and {(ii)} $\varop$-sortability may not be the only reason why \notears performs significantly worse on standardized data \citep{reisach2021beware}.
Conversely, when the weight ranges of (standardized) SCMs are smaller, the phenomenon of exploding marginal variances is less pronounced (Figure \ref{fig:inverse-noise-scales-dist} in Appendix \ref{sec:results-similar-implied-noise-scales}).
In this case, we indeed find that \notears performs similarly on standardized SCMs and \ourss (Figure \ref{fig:results-benchmark-appendix-full}, left, in Appendix \ref{ssec:csl-performance}).

This sheds light on previous benchmarking results, where MSE-based algorithms perform below expectations despite perhaps not intending to evaluate the algorithms under model mismatch (e.g., \citealp{reisach2021beware,kaiser2021unsuitability}).
For the MSE loss, \citet{loh2014high} and \citet{seng2024learning} show that smaller ratios of noise variances increase the magnitude of weights required for the true DAG to be the unique minimizer. 
The MSE loss ultimately does not account for the inverse variance factor in the Gaussian noise likelihood.
Overall, the statistics of the implied models of standardized SCMs are empirically further from SCMs with equal noise variances than their \ours counterparts.

\vsection{Conclusions}\label{sec:conclusion}

We describe the \ours, a one-line modification of the SCM that modulates the scale of interaction between the causal mechanism $f_i$ and noise $\noise{i}$ at each variable $x_i$.
Through several theoretical and experimental results,
we study its properties in relation to standard SCMs and its ramifications for benchmarking causal discovery algorithms.
To conclude, we highlight the following key takeaways:

\vparagraph{Standardizing during the generative process removes sortability \artefacts.}
When the functions $f_i$ and the noise $\noise{i}$ are, for example, sampled \iid for each variable $x_i$, SCMs exhibit \artefacts that are not removed when shifting and scaling the generated data.
Our results in Section \ref{sec:experiments} show that \ourss are effective at removing $\varop$- and \rtwo-sortability.
This makes \ourss a useful complement to structure learning benchmarks with SCMs, enabling a specific evaluation of the ability of algorithms to transfer to real-world settings that do not exhibit \rtwo \artefacts.
Despite the removed sortability \artefacts, causal discovery algorithms are able to infer nontrivial structure from \ours data (Figure \ref{fig:results-benchmark-main}).

\vparagraph{Standardizing post-hoc can lead to partial identifiability and degenerate implied SCMs.}
%
Scaling the units of SCM data is not innocuous.
Theorem \ref{th:three_part_ident} shows that mild knowledge on the distribution of $f_i$ can identify edges in standardized SCMs that are typically not identifiable from observational data.
To our knowledge, our result is the first concerning the identifiability of $\mathcal{G}$ from the standardized observational distribution of linear SCMs. This may make benchmarks, where similar assumptions on $f_i$ often hold, trivial under standardized SCMs. 
Moreover, Figure \ref{fig:results-induced-noise-main} shows that standard SCMs can collapse to modeling near-zero exogenous noise.
Theorems \ref{th:cevbound} and \ref{th:three_non_ident} demonstrate that
neither property appears in the analogous \ourss. 
Ultimately, (non)identifiability may be either a feature or bug, depending on whether assumptions are verifiable in practice or a priori known during evaluation.

\vparagraph{\ourss are stable and scale-free, making them useful models beyond benchmarking.}
Beyond data generation,
the stable generative process of \ourss might also provide insights for modeling, e.g., large, temporal  \citep{kilian2013structural,pamfil2020dynotears} or physical systems.
In \ourss, the scale of a causal mechanism $f_i$ and its unexplained variation $\noise{i}$ are both unit-less and independent from its position in the causal ordering (Section \ref{sec:iscm}).
If we think of each structural assignment as a physical mechanism, energy conservation must be respected, since a mechanism can only output as much energy as it receives from its inputs (including unexplained noise). 
Standardization may thus not be completely unrealistic, since it naturally bounds the output scale of every mechanism.

Since each \ours implies a standard SCM, \ourss can also be viewed as a reparameterization of SCMs that facilitates modeling and learning the functions $f_i$ on the same scale, \eg, under a shared prior or level of regularization.
Conceptually, \ourss are related to batch normalization \citep{ioffe2015batch}, a technique used to stabilize the optimization of neural networks, which compose sequences of functions like SCMs, by adding internal standardization. 
Overall, these properties may make the \ours a useful structural equation model beyond the benchmarking problem studied here.

\section*{Reproducibility Statement}
To facilitate reproducibility, we provide code, configuration files, and the commands used to obtain all the experimental results in this manuscript as supplementary material. They are also available at: \href{https://github.com/werkaaa/iscm}{https://github.com/werkaaa/iscm}. In \Cref{sec:csl-setup}
, we describe the experimental setup, including the computational resources and wall time used to produce the results. Finally, we provide detailed proofs of our theoretical results in \Cref{sec:proofs}.

\section*{Acknowledgements}
This research was supported by the European Research Council (ERC) under the European Union's Horizon 2020 research and innovation program grant agreement no.\ 815943 and the Swiss National Science Foundation under NCCR Automation, grant agreement 51NF40 180545.
This work was also supported by the German Federal Ministry of Education and Research (BMBF): T{\"u}bingen AI Center, FKZ: 01IS18039B, and by the Machine Learning Cluster of Excellence, EXC number 2064/1, project number 390727645.

\bibliography{bibliography}


\newpage
\appendix
\section*{Contents}
\startcontents[sections]
\printcontents[sections]{l}{1}{\setcounter{tocdepth}{2}}
\newpage

\section{Implied Models}\label{sec:implied-models}
In this section, we describe how to express the assignments of the observed variables of standardized SCMs and \ourss with a general additive noise mechanism
\begin{align}
    \label{eq:implied-model-mechanism}
    f_i(\mathbf{x}, \varepsilon_i) = f_i(\mathbf{x}) + \varepsilon_i \, ,
\end{align}
in the form of \eqref{eq:scm},
while sharing the same causal graph $\mathcal{G}$ and exogenous noise variables $\Noise$.
We obtain the SCM form by moving the standardization steps into the causal mechanisms by linearly rescaling $f_i$ and $\noise{i}$, such that each observed variable is only a function of observed variables and the noise $\noise{i}$.
Throughout this work, the {\bf \em implied (SCM) model} denotes the specific construction given in the following two subsections. 
For this, we assume that we can express the first two moments of the system in closed form.
Similar to the main text, we overload notation for both standardized SCMs and \ourss and write
\begin{equation*}
    \mu_i := \Expempty[x_i] \quad\quad\text{and}\quad\quad s_i := \sqrt{\var{x_i}} \, .
\end{equation*}
We also derive analytic expressions for the weights of the implied models of linear \ourss defined by Equation \eqref{eq:linear_scm}, which we later use in our proofs.

\subsection{Implied Model of a Standardized SCM}

Let $\s{\bf{x}}$ be modeled by \eqref{eq:standardized_scm} 
with causal mechanisms defined by Equation \eqref{eq:implied-model-mechanism}.
We recall that $\s{\bf{x}}$ are the observations obtained after standardizing $\bf{x}$.
Thus, we can rearrange $\s{x}_i$ as
\begin{equation*}
    x_i = s_i\s{x}_i + \mu_i 
\end{equation*}
and substitute every unstandardized variable $x_i$ by a function of its standardized parents $\pa{i}^s$ as
\begin{equation*}
    \s{x}_i = \frac{x_i - \mu_i}{s_i} = \frac{f_i(\pa{i}) + \noise{i} - \mu_i}{s_i} 
    = \frac{f_i(\pa{i}^s \odot \boldsymbol{s}_{\panox{i}}
    + \boldsymbol{\mu}_{\panox{i}}) - \mu_i}{s_i} + \frac{1}{s_i} \noise{i}\, ,
\end{equation*}
where $\odot$ denotes elementwise multiplication, and $\boldsymbol{\mu}_{\panox{i}}$ and $\boldsymbol{s}_{\panox{i}}$ are the vectors of the parent means and standard deviations before standardization. 
Thus, the assignments of $\s{\bf{x}}$ in a standardized SCM can be written as the SCM given by
\begin{equation*}
    \s{x}_i = \s{g}_i(\s{\pa{i}}) + \s{\theta}_i \noise{i} \, ,
\end{equation*}
with implied noise scales $ \s{\theta}_i := 1/s_i$
and implied causal mechanisms 
\begin{align*}
    \s{g}_i(\s{\pa{i}}) &:= 
    \begin{cases}
        \displaystyle \frac{f_i(\s{\pa{i}} \odot \boldsymbol{s}_{\panox{i}} + \boldsymbol{\mu}_{\panox{i}}) - \mu_i}{s_i} & \text{if $i$ is a non-root variable, and} \\
        \displaystyle \frac{f_i - \mu_i}{s_i} & \text{if $i$ is a root variable.}
    \end{cases}
\end{align*}

\subsection{Implied Model of an \ours}\label{ssec:implied_models_iscm}

Let $\as{\bf{x}}$ be modeled by \eqref{eq:iscm} with causal mechanisms defined by Equation \eqref{eq:implied-model-mechanism}.
In an \ours, 
$\as{\bf{x}}$ are the observed variables and $\bf{x}$ are the latent variables.
We can express every observation $\as{x}_i$ in terms of its observed parents $\paours{i}$ as
\begin{equation*}
    \as{x}_i = \frac{x_i - \mu_i}{s_i} 
    = \frac{f_i(\paours{i}) + \noise{i} - \mu_i}{s_i}
    = \frac{f_i(\paours{i}) - \mu_i}{s_i} +  \frac{1}{s_i} \noise{i}  \, .
\end{equation*}
Thus, the assignments of $\as{\bf{x}}$ in a \ours can be written as the SCM given by
\begin{equation*}
    \as{x}_i = \as{g}_i(\paours{i}) + \as{\theta}_i \noise{i}  \, ,
\end{equation*}
with implied noise scales $\as{\theta}_i := 1/s_i$
and implied causal mechanisms 
\begin{align*}
    \as{g}_i(\paours{i}) &:= 
    \begin{cases}
        \displaystyle \frac{f_i(\paours{i}) - \mu_i}{s_i} & \text{if $i$ is a non-root variable, and} \\
        \displaystyle \frac{f_i - \mu_i}{s_i} & \text{if $i$ is a root variable.}
    \end{cases}
\end{align*}

\subsection{Weights of the Implied Model of a Linear \ours}
\label{ssec:implied_linear_iscm}
Here, we derive the analytical form for the mechanisms of the implied model of a linear \ours with zero-centered, additive noise $\noise{i}$.
This \ours is given by
\begin{equation*}
     x_i := \mathbf{w}_i^T\paours{i} + \noise{i} 
     \quad\quad\text{and}\quad\quad
    \as{x}_i := \frac{x_i}{\sqrt{\var{x_i}}} \, ,
\end{equation*}
where $\noise{i}$ satisfies $\Expempty[\noise{i}] = 0$ and $\var{\noise{i}} = \sigma_i^2$.
We can write the above as
\begin{align*}
     \as{x}_i &= \frac{\mathbf{w}_i^T\paours{i} + \noise{i}} {\sqrt{\var{x_i}}} = \frac{\sum_{j \in \paset{i}} w_{j, i}\as{x}_j + \noise{i}} {\sqrt{\var{x_i}}}
     = \sum_{j \in \paset{i}}\frac{w_{j, i}}{\sqrt{\var{x_i}}} \, \as{x}_j + \frac{1}{\sqrt {\var{x_i}}}\noise{i} \, .
\end{align*}
It follows that the implied SCM of a linear \ours is also linear, with weights and noise variances given by
\begin{equation}
\label{eq:implied_weights_noise_variances}
    \as{w}_{j, i} = \frac{w_{j, i}}{\sqrt{\var{x_i}}}
    \quad\quad
    \text{and}
    \quad\quad
    \as{\sigma}^2_i = \frac{\sigma^2_i}{\var{x_i}} \, .
\end{equation}
In the above, we can write the variance of $x_i$ explicitly as
\begin{equation}
\label{eq:var_hidden_vars}
\begin{aligned}
    \var{x_i} 
    &= \varop \bigg [
    \sum_{j \in \paset{i}}w_{j, i}\as{x}_j + \noise{i} 
    \bigg ] 
    = \varop \bigg [
    \sum_{j \in \paset{i}}w_{j, i}\as{x}_j
    \bigg ]
    + \sigma_i^2 \\
    &\overset{\tiny \circled{1}}{=} \sum_{k \in \paset{i}}\sum_{j \in \paset{i}} \cov{w_{k, i}\as{x}_k, w_{j, i}\as{x}_j} + \sigma_i^2 \\
    &\overset{\tiny \circled{2}}{=} \sum_{k \in \paset{i}}\sum_{j \in \paset{i}} w_{k, i}w_{j, i}\cov{\as{x}_k,\as{x}_j} + \sigma_i^2 \, ,
\end{aligned}
\end{equation}
where $\tiny \circled{1}$ follows from Bienaym{\'e}'s identity and $\tiny \circled{2}$ from covariance being bilinear. 
Substituting the variance into the expressions for the weights and noise variances, 
we obtain
\begin{align}
    \as{w}_{j, i} &=  \frac{w_{j, i}}{\sqrt{\sum_{k \in \paset{i}}\sum_{j \in \paset{i}} w_{k, i}w_{j, i}\cov{\as{x}_k, \as{x}_j} + \sigma_i^2}} \, , \label{eq:weights_implied_ours} \\
    \as{\sigma}^2_i &=  \frac{\sigma^2_i}{\sum_{k \in \paset{i}}\sum_{j \in \paset{i}} w_{k, i}w_{j, i}\cov{\as{x}_k, \as{x}_j} + \sigma_i^2} \, . \label{eq:noise_var_implied_ours}
\end{align}
Finally, by construction, the variables $\as{\bf{x}}$ of an \ours have unit marginal variances.
Thus, when the parents of $\as{x}_i$ are pairwise independent, \cref{eq:noise_var_implied_ours} simplifies to
\begin{equation}
\label{eq:weights_implied_ours_indep}
    \as{w}_{j, i} =  \frac{w_{j, i}}{\sqrt{\sum_{j \in \paset{i}}w_{j, i}^2 + \sigma_i^2}}.
\end{equation}
This independence condition always holds when the DAG $\mathcal{G}$ is a forest.

\paragraph{Efficient computation}
We can efficiently compute the implied model weights using a bottom-up dynamic programming approach.
This allows sampling data directly from the exact implied model of an \ours without resorting to empirical standardization statistics.
Algorithm \ref{alg:implied_weights} describes the procedure.
We iteratively compute the weights and noise variances of the implied model following Equations \eqref{eq:weights_implied_ours} and \eqref{eq:noise_var_implied_ours}. 
At each iteration, we  update the covariance matrix according to Lemma \ref{lemma:cov}. 
The algorithm processes the nodes in topological order, mirroring the proof by induction of Lemma \ref{lemma:cov}. 

\begin{algorithm}
\caption{Computing the Implied Model Parameters of Linear \ourss}
\begin{algorithmic}[0]
\State \textbf{Input:} DAG $\mathcal{G}$, weight matrix $[W]_{i,j} := w_{i,j}$, noise variances $\bm{\sigma}^2 \in \RR_{+}^\numvars$
\State $\as{W} \gets 0_{d\times d}$ 
\State $\Sigma \gets I_{\numvars}$ 
\State $\pi \gets$ topological ordering of $\mathcal{G}$ 
\For{$i=1$ to $\numvars$}
    \vspace{1pt}
    \State $\mathbf{w} \gets W_{:, \pi_i}$ \Comment{Edge weights ingoing to $\pi_i$}
    \vspace{1pt}
    \State $\var{x_{\pi_i}} \gets \mathbf{w}^\top \Sigma \mathbf{w} + \sigma^2_{\pi_i}$ 
    \Comment{Equation \eqref{eq:var_hidden_vars}}
    \vspace{0pt}
    \State $\as{W}_{:, \pi_i} \gets \mathbf{w}/\sqrt{\var{x_{\pi_i}}}$ 
    \Comment{Equation \eqref{eq:weights_implied_ours}}
    \vspace{3pt}
    \State $\as{\sigma}^2_{\pi_i} \gets \sigma^2_{\pi_i} /\var{x_{\pi_i}}$ 
    \Comment{Equation \eqref{eq:noise_var_implied_ours}}
    \vspace{2pt}
    \For{$j=1$ to $i$}
        \State $\Sigma_{\pi_{j}, \pi_{i}} \gets (\Sigma_{\pi_{j}, :})^\top \as{W}_{:, \pi_i}$
        \State  $\Sigma_{\pi_{i}, \pi_{j}} \gets \Sigma_{\pi_{j}, \pi_{i}}$
    \EndFor
\EndFor \\
\Return implied weights $\as{W}$, implied noise variances $\as{\bm{\sigma}}^2$
\end{algorithmic}
\label{alg:implied_weights}
\end{algorithm}
\vspace*{10pt}

\newpage

\section{Interventions in \ourss}
\label{sec:interventionssss_iscm}

For an \ours $(\mathbf{S},\distnoisejoint)$, we can formalize interventions as changes to its causal mechanisms $f_i$, analogous to the common definition for SCMs
\citep{peters2017elements}.
Specifically, let $\mu_i := \Expempty[x_i]$ and $s_i := \smash{\sqrt{\var{x_i}}}$ be the mean and standard deviation of the  latent variable $x_i$.
We define an {\em intervention} as replacing one (or several) of the assignments to the latent variables as
\begin{align*}
\begin{split}
    x_i &:= h_i(\paours{i}, \noise{i}),
\end{split}
\end{align*}
for some function $h_i$.
Importantly, the statistics $\mu_i$ and $s_i$ used for the standardization operation
\begin{align*}
    \as{x}_i := \frac{x_i - \mu_i}{s_i}
\end{align*}
remain {\em unchanged}.
Thus, if we intervene on mechanisms of \ourss, the variables $\as{\bf{x}}$ may no longer have zero mean and unit variance, and the perturbations of $x_i$ propagate downstream through the causal mechanisms.
We note that, under the above definition, intervening on an \ours through a new mechanism $h_i$ is equivalent to intervening on the implied SCM of an \ours with the mechanism
\begin{align*}
    \as{h}_i(\mathbf{x}, \varepsilon) = \frac{h_i(\mathbf{x}, \varepsilon) - \mu_i}{s_i} \, .
\end{align*}
Appendix \ref{ssec:implied_models_iscm} provides details on the implied models of \ourss.

\section{Proofs}\label{sec:proofs}

\subsection{Definitions}
We define the key concepts used throughout our analysis.
A {\em path} $\ugpath{j}{i}$ between $v_i$ and $v_j$ is a set of directed edges that allows reaching $v_i$ from $v_j$ (and vice versa), not taking into account edge directionality, and that joins unique vertices.
We call a node a {\em collider} in a path if the node has two ingoing directed edges in the path.
We say that a path between $v_i$ and $v_j$ is {\em unblocked} if and only if there is no node $v_k$ that is a collider in the path (see Figure \ref{fig:collider-in-path}). 
Finally, we use the term {\em undirected connected component} to refer to any maximal subgraph of $\mec$ in which any two nodes are connected by a path containing only undirected edges \citep{wienobst2023efficient}.

\subsection{Explicit Covariance in Linear SCMs with Unit Marginal Variances}
\label{ssec:cov_formula}

\covform*
\begin{proof}
We will give a proof by induction on the number of vertices $\numvars = |\mathcal{V}|$ in the DAG $\mathcal{G}$. 
Without loss of generality, we assume that the indices of the nodes are ordered according to some fixed topological ordering $\pi$, so $\pi(j) < \pi(i)$ if $j < i$.
By the unit marginal variance assumption, 
\begin{equation}
\label{eq:var1}
    \cov{x_i, x_i} = \var{x_i} = 1 \, .
\end{equation}
From now on and without loss of generality, we consider two arbitrary indices $j < i$. The covariance between $x_i$ and $x_j$ is symmetric.

\paragraph{Base case ($\numvars = 2$)} 

If $v_j$ is not an ancestor of $v_i$ in graph $\mathcal{G}$, they both must be root nodes, because the edge $v_i \leftarrow v_j$ is the only possible edge when $\pi(j) < \pi(i)$.
Since $x_i$ and $x_j$ are root nodes, they are independent and $\cov{x_i, x_j} = 0$. 
Since a path of one edge cannot contain a collider, there are no unblocked paths between $v_i$ and $v_j$, so the RHS of Equation \eqref{eq:cov_form} is also $0$. 

Conversely, if $v_j$ is an ancestor of $v_i$ in graph $\mathcal{G}$, $v_j$ is the only parent and ancestor of $v_i$.
This implies that
\begin{align*}
\begin{split}
    \cov{x_i, x_j} &= \cov{w_{j, i}x_j + \noise{i}, x_j} \\
    &= w_{j, i}\cov{x_j, x_j}\\
    &= w_{j, i} \, ,
\end{split}
\end{align*}
where the last equality follows from Equation \eqref{eq:var1}. This is exactly Equation \eqref{eq:cov_form} for a two-node graph.

\begin{figure}
    \centering
\begin{tikzpicture}[node distance=1.7cm]

  \node[state] (vi)                {$v_i$};
  \node[state] (v2) [above left of=vi]  {};
  \node[state] (v3) [left of=v2]  {};
  \node[state] (v4) [above right of=vi]  {};
  \node[state] (v5) [right of=v4]  {};
  \node[state] (vj) [above right of=v2]  {$v_j$};

  \path[->] (v2) edge (vi)
            (v3) edge (vi)
            (v4) edge (vi)
            (v5) edge (vi)
            (v4) edge[dotted] (v5)
            (vj) edge[dotted] (v3)
            (vj) edge[dotted, bend right] (v3)
            (vj) edge[dotted, bend left] (v5)
            (vj) edge[blue] (vi)
            (vj) edge[dotted, bend left] (v4)
            (vj) edge[dotted] (v4)
            (vj) edge[dotted, bend right] (v4);

\end{tikzpicture}
    \caption{\textbf{Lemma \ref{lemma:cov} inductive step.} If $v_j$ is before $v_i$ in the topological ordering, then all unblocked paths from $v_j$ to $v_i$ must contain a parent of $v_i$ as the second to last node. 
    To see this, suppose an unblocked path from $v_j$ to $v_i$ would instead contain a child of $v_i$ as the last node. 
    Then, there either exists a collider on the path to $v_j$, contradicting that the path is unblocked, or all edges in the path point away from $v_i$, implying that $v_j$ is a descendant of $v_i$ and contradicting the topological ordering.
    Dotted lines represent unblocked paths (which may have common nodes). Solid lines represent edges. $v_j$ may or may not be a parent of $v_i$, which we illustrate with a blue arrow.}
    \label{fig:cov_form_vis}
\end{figure}
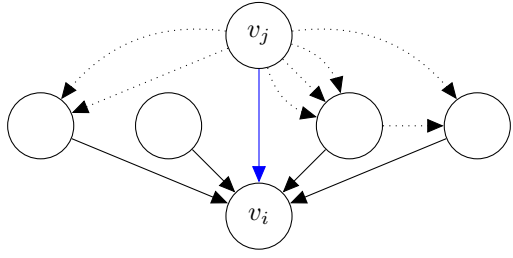

\paragraph{Induction step ($\numvars > 2$)} Let us assume that Equation \eqref{eq:cov_form} holds for all graphs of size $\numvars -1$,
and let $\mathcal{G}$ have $\numvars$ nodes.
We will apply the inductive hypothesis to the subgraph of the first $d-1$ nodes in $\mathcal{G}$ and show that the full DAG $\mathcal{G}$ including the $d$-th vertex still satisfies Equation \eqref{eq:cov_form}.
First, we note that, since the $\numvars$-th vertex is last in the topological ordering, it has no outgoing edges. 
Because the node has no outgoing edges, it is not visited on any unblocked paths between $v_j$ and $v_i$ for $i,j<\numvars$, as $v_d$ must be a collider in any path.
Second, adding the node $v_\numvars$ to a subsystem containing $x_1, \dots, x_{\numvars-1}$ results in no change to the joint distribution of $x_i, x_j$.
Therefore, it has no effect on the covariance between $x_i, x_j$. Hence, both sides of \cref{eq:cov_form} are unchanged by the presence of a node $v_d$ for all $i,j<\numvars$ and the equation still holds for all $i,j<\numvars$.

We want to show that \cref{eq:cov_form} also holds for $i = \numvars$ and any $j < i$. 
For this, we first construct all unblocked paths from $v_j$ to $v_i$. First, we note that any unblocked path must go through the parents $k \in \text{pa}(i)$, because $j<i$ in the topological ordering (see Figure \ref{fig:cov_form_vis}). Moreover, for any $k \in \text{pa}(i)$, appending $k\rightarrow i$ to an unblocked path $\ugpath{j}{k}$ between $v_j$ and $v_k$, creates a new unblocked path between $v_j$ and $v_i$. 
Hence, for $i=\numvars$ and any $j < i$, it holds that
\begin{align*}
\begin{split}
    \cov{x_{i}, x_j} &= \cov{\sum_{k \in \paset{i}}w_{k, i}x_k + \noise{i}, x_j} \\
    &= \sum_{k \in \paset{i}}w_{k, i}\cov{x_k, x_j} \\
    &\overset{\tiny \circled{1}}{=} w_{j, i} \cov{x_j, x_j} + \sum_{k \in \paset{i} \setminus{j}}w_{k, i}\cov{x_k, x_j} \\
    &\overset{\tiny \circled{2}}{=} w_{j, i} + \sum_{k \in \paset{i} \setminus{j}}w_{k, i}\sum_{\ugpath{j}{k} \in \ugpathset{j}{k}}\prod_{(l, m) \in \ugpath{j}{k}}w_{l, m}\\
    &= w_{j, i} + \sum_{k \in \paset{i}\setminus{j}} \left( \sum_{\ugpath{j}{k} \in \ugpathset{j}{k}}w_{k, i}\prod_{(l, m) \in \ugpath{j}{k}}w_{l, m} \right) \\
     &\overset{\tiny \circled{3}}{=} \sum_{k \in \paset{i}}\left(\mathbbm{1}[k=j] w_{j, i} +  \mathbbm{1}[k\neq j]  \left(\sum_{\ugpath{j}{k} \in \ugpathset{j}{k}}w_{k, i}\prod_{(l, m) \in \ugpath{j}{k}}w_{l, m}\right)\right) \\
    &\overset{\tiny \circled{4}}{=} \sum_{\ugpath{j}{i} \in \ugpathset{j}{i}}\prod_{(l, m) \in \ugpath{j}{i}}w_{l, m} \,  .
\end{split}
\end{align*}
For step $\tiny \circled{1}$, consider two cases. If $j \notin \text{pa}(i)$, then $w_{j,i}=0$ and the equality trivially holds. If $j \in \text{pa}(i)$, then it holds by pulling the term for $j$ out of the sum in the previous line. In $\tiny \circled{2}$, we apply the inductive hypothesis to express the covariances in terms of a sum of products of weights. 
In $\tiny \circled{3}$, we rearrange terms to pull the $w_{j, i}$ term into the sum over parents. 
In $\tiny \circled{4}$, we use the fact that the set of unblocked paths from $v_j$ to $v_i$ corresponds to all paths from $v_j$ to any parent of $v_i$, which is $v_k$ here, with an extra edge $k\rightarrow i$ appended, and a possible single-edge path directly connecting $v_j$ with $v_i$ (if $j\in\paset{i}$). 

This completes the induction step and the proof.
\end{proof}

\subsection{Bound on the Fraction of CEV}
\label{ssec:cev_bound}
\cevbound*
\begin{proof}
We begin by bounding the variance of the latent variables $x_i$ in \ourss. Starting from Equation \eqref{eq:var_hidden_vars}, we can  bound the covariances with a product of unit variances as
\begin{align*}
\begin{split}
    \var{x_i} &= \sum_{k \in \paset{i}}\sum_{j \in \paset{i}} w_{k, i}w_{j, i}\cov{\as{x}_j, \as{x}_k} + \sigma^2 \\
    &\overset{\tiny \circled{1}}{\leq} \sum_{k \in \paset{i}}\sum_{j \in \paset{i}} w_{k, i}w_{j, i} + \sigma^2 \\
    &= \Big(\sum_{j \in \paset{i}}w_{j, i}\Big)^2 + \sigma^2 \\
    &\overset{\tiny \circled{2}}{\leq} m^2w^2 + \sigma^2 \, ,
\end{split}
\end{align*}
where $\tiny \circled{1}$ uses $\cov{\as{x}_j, \as{x}_k} \leq 1$ since $\var{\as{x}_j} = 1$ and $\var{\as{x}_k} = 1$,
and $\tiny \circled{2}$ applies the Cauchy-Schwartz inequality.
Since we obtain $\as{x_i}$ from $x_i$ just by shifting and scaling the latter, we observe that $\cevf{\as{x_i}} = \cevf{x_i}$. Using the upper bound on the variance of $x_i$ and the definition of the fraction of cause-explained variance in Equation \eqref{eq:cev-def}), we get
\begin{align*}
\begin{split}
    \cevf{\as{x_i}} 
    &= \cevf{x_i} 
    =  1 - \frac{\var{x_i - \Expempty[x_i | \pa{i}]}}{\var{x_i}} 
    = 1 - \frac{\var{x_i - \mathbf{w}_i^\top \pa{i}}}{\var{x_i}} \\
    &= 1 - \frac{\var{\noise{i}}}{\var{x_i}}
    = 1 - \frac{\sigma^2}{\var{x_i}} \leq  1 - \frac{\sigma^2}{m^2w^2 + \sigma^2} \,  .
\end{split}
\end{align*}
\end{proof}

\subsection{Identifiability}

In this section, we prove Theorems \ref{th:three_part_ident} and \ref{th:three_non_ident}. 
We begin by deriving the covariances for the 3-node example in \cref{ssec:identifiability} and then give the general proofs for forests. 
The proofs of both theorems share the same underlying argument.
We first derive the SCM forms of the original models, i.e., standardized SCMs in Theorem \ref{th:three_part_ident} and \ourss in Theorem \ref{th:three_non_ident}.
By showing that the standardized SCMs and \ourss are SCMs with the same causal graphs $\mathcal{G}$ and observational distributions $p(\bf{x})$, we can leverage Lemma \ref{lemma:cov} to obtain the covariances between the observed variables in both model classes.
Ultimately, these covariances allow us to derive (non)identifiability conditions for the DAGs $\mathcal{G}$ in an MEC underlying the original models.

Theorems \ref{th:three_part_ident} and \ref{th:three_non_ident} assume that the exogenous noise is sampled from a zero-centered distribution with equal variance across variables. Since the results are based on the analysis of covariances, they also hold with the assumption that $\Expempty[\noise{i}] \neq 0$, but the zero-mean assumption simplifies notation. 
To derive the results for \ourss, we additionally assume that the noise is Gaussian (see Theorem \ref{th:three_non_ident}) . 
When referring to an undirected edge between nodes $v_i,v_j$, for example, in an MEC, we still denote the edge with $(v_i, v_j)$, but the ordering of the nodes is arbitrary. 

\subsubsection{3-Node Case}
\label{ssec:ident_3node}

We begin by studying the 3-node example of Figure \ref{fig:3-path-id} in  \cref{ssec:identifiability}.
Let $\alpha_i, \beta_i, \gamma_i, \lambda_i \in \mathbb{R}$ be linear function weights, and consider the following three causal graphs $\mathcal{G}$ belonging to the same MEC, along with their corresponding SCMs and \ourss.

\newcommand{\proofvisualizationspacing}{\hspace{10pt}}
\begin{figure}[htbp]
    \centering
     \begin{minipage}[b]{0.3\textwidth}
     \centering
     $\mathcal{G}$
     \end{minipage}
     \hfill
     \begin{minipage}[b]{0.3\textwidth}
     \centering
     SCM
     \end{minipage}
     \hfill
     \begin{minipage}[b]{0.3\textwidth}
     \centering
     iSCM
     \end{minipage}
    \begin{minipage}[b]{0.3\textwidth}
        \centering
       \begin{tikzpicture}[node distance=1.5cm]
            \node[circle,draw] (circle1) {$v_1$};
            \node[circle,draw,right of=circle1] (circle2) {$v_2$};
            \node[circle,draw,right of=circle2] (circle3) {$v_3$};

            \draw[->] (circle1) -- (circle2);
            \draw[->] (circle2) -- (circle3);
        \end{tikzpicture}
    \end{minipage}%
    \hfill
    \begin{minipage}[b]{0.3\textwidth}
        \begin{equation}
        \label{eq:s1}
            \begin{aligned}[c]
                x_1 &:= \noise{1} \\
                x_2 &:= \alpha_1x_1 + \noise{2} \\
                x_3 &:= \beta_1x_2 + \noise{3} \\
            \end{aligned}
        \end{equation}
    \end{minipage}
    \hfill
    \begin{minipage}[b]{0.3\textwidth}
        \begin{equation}
        \label{eq:s1_ours}
            \begin{aligned}[c]
                x_1 &:= \noise{1} \\
                x_2 &:= \gamma_1 \as{x}_1 + \noise{2} \\
                x_3 &:= \lambda_1 \as{x}_2 + \noise{3} \\
            \end{aligned}
        \end{equation}
    \end{minipage}

    \vspace{20pt}
    \begin{minipage}[b]{0.3\textwidth}
        \centering
       \begin{tikzpicture}[node distance=1.5cm]
            \node[circle,draw] (circle1) {$v_1$};
            \node[circle,draw,right of=circle1] (circle2) {$v_2$};
            \node[circle,draw,right of=circle2] (circle3) {$v_3$};

            \draw[<-] (circle1) -- (circle2);
            \draw[->] (circle2) -- (circle3);
        \end{tikzpicture}
    \end{minipage}%
    \hfill
    \begin{minipage}[b]{0.3\textwidth}
        \begin{equation}
        \label{eq:s2}
            \begin{aligned}[c]
                x_1 &:= \alpha_2x_2 + \noise{1} \\
                x_2 &:=  \noise{2} \\
                x_3 &:= \beta_2x_2 + \noise{3} \\
            \end{aligned}
        \end{equation}
    \end{minipage}
    \hfill
    \begin{minipage}[b]{0.3\textwidth}
        \begin{equation}
        \label{eq:s2_ours}
            \begin{aligned}[c]
                x_1 &:= \gamma_2 \as{x}_1 + \noise{1} \\
                x_2 &:= \noise{2} \\
                x_3 &:= \lambda_2 \as{x}_2 + \noise{3} \\
            \end{aligned}
        \end{equation}
    \end{minipage}
    \vspace{20pt}

    \begin{minipage}[b]{0.3\textwidth}
        \centering
       \begin{tikzpicture}[node distance=1.5cm]
            \node[circle,draw] (circle1) {$v_1$};
            \node[circle,draw,right of=circle1] (circle2) {$v_2$};
            \node[circle,draw,right of=circle2] (circle3) {$v_3$};

            \draw[<-] (circle1) -- (circle2);
            \draw[<-] (circle2) -- (circle3);
        \end{tikzpicture}
    \end{minipage}%
    \hfill
    \begin{minipage}[b]{0.3\textwidth}
        \flushright
        \begin{equation}
        \label{eq:s3}
            \begin{aligned}[c]
                x_1 &:= \alpha_3x_2 + \noise{1} \\
                x_2 &:= \beta_3x_3 + \noise{2} \\
                x_3 &:= \noise{3} \\
            \end{aligned}
        \end{equation}
    \end{minipage}
    \hfill
    \begin{minipage}[b]{0.3\textwidth}
        \flushright
        \begin{equation}
        \label{eq:s3_ours}
            \begin{aligned}[c]
                x_1 &:= \gamma_3\as{x}_2 + \noise{1} \\
                x_2 &:= \lambda_3\as{x}_3 + \noise{2} \\
                x_3 &:= \noise{3} \\
            \end{aligned}
        \end{equation}
    \end{minipage}

\end{figure}

In the following subsections, we derive the covariance matrices of each of the three systems, respectively. This leads us to the equivalence presented in Equation \eqref{eq:cov_weight_rel} for standardized SCMs. 
Moreover, we show that, for \ourss, all three systems induce exactly the same observational distribution if and only if $\lambda_{1} = \lambda_2 = \lambda_3$ and  $\gamma_{1} = \gamma_2 = \gamma_3$.
These are the 3-node special cases of Theorems \ref{th:three_part_ident} and \ref{th:three_non_ident}. 

\subsubsection*{Standardized SCM}
To obtain the covariances between the observed variables in the standardized SCMs of Equations \eqref{eq:s1}, \eqref{eq:s2}, and \eqref{eq:s3}, 
we first show that the assignments to the observed variables in standardized SCMs can be written in the form of linear SCMs over the same causal graph, which allows us to use Lemma \ref{lemma:cov}.
In all three systems, every vertex has at most one parent. When the node $v_j$ is the only parent of $v_i$, under our assumptions on the noise, we have $x_j = \smash{\sqrt{\var{x_j}}\s{x_j}}$, so the assignment of $\s{x_i}$ can be written in the form of an SCM over $\s{\bf{x}}$ as
\begin{equation}
\label{eq:implied_stand}
    \s{x_i} := 
    \frac{x_i}{\sqrt{\var{x_i}}} =
    \frac{w_{j, i}x_j + \noise{i}} {\sqrt{\var{x_i}}} = \frac{w_{j, i}{\sqrt{\var{x_j}}}\s{x_j} + \noise{i}} {\sqrt{\var{x_i}}} = w_{j, i}\sqrt{\frac{\var{x_j}}{\var{x_i}}}\s{x_j} + \frac{\noise{i}}{\sqrt {\var{x_i}}} \, .
\end{equation}
To use Equation \eqref{eq:implied_stand}, we first need to compute the marginal variances of the unstandardized observations $x_i$. 
For the standardized SCMs, these marginal variances are, respectively:

\begin{center}
\vspace{5pt}
\footnotesize
\begin{tabular}{ lll } 
{\normalsize for Equation \eqref{eq:s1}:}
& {\normalsize for Equation \eqref{eq:s2}:} \hspace*{38pt}
& {\normalsize for Equation \eqref{eq:s3}:} \\[10pt]
  $\var{x_1} = \sigma^2$ 
& $\var{x_1} = (\alpha_2^2 + 1)\sigma^2  $ 
& $\var{x_1} = (\alpha_3^2(\beta_3^2 + 1) + 1)\sigma^2 $ \\[7pt] 
  $\var{x_2} = (\alpha_1^2 + 1)\sigma^2$ 
& $\var{x_2} = \sigma^2$ 
& $\var{x_2} = (\beta_3^2 +1)\sigma^2$ \\[7pt]  
  $\var{x_3} = (\beta_1^2(\alpha_1^2+1)+1)\sigma^2 $ 
& $\var{x_3} = (\beta_2^2+1)\sigma^2 $ 
& $\var{x_3} = \sigma^2 $ \\ 
\end{tabular}
\vspace{10pt}
\end{center}

Given Equation \eqref{eq:implied_stand} and the marginal variances, we know the weights of all three implied SCMs explicitly.
Since all implied SCMs are linear, have unit marginal variances, and share the same causal graph, we can apply Lemma \ref{lemma:cov} and obtain the covariances of the observational distributions in the original models:

\begin{center}
\vspace{5pt}
\footnotesize
\begin{tabular}{ lll } 
{\normalsize for Equation \eqref{eq:s1}:}
& {\normalsize for Equation \eqref{eq:s2}:}
& {\normalsize for Equation \eqref{eq:s3}:} \\[5pt]
  $\cov{\s{x_1}, \s{x_2}} = \tfrac{\alpha_1}{\sqrt{\alpha_1^2 + 1}}$ 
& $\cov{\s{x_1}, \s{x_2}} = \tfrac{\alpha_2}{\sqrt{\alpha_2^2 + 1}}$ 
& $\cov{\s{x_1}, \s{x_2}} = \alpha_3 \sqrt{\tfrac{\beta_3^2 + 1} {\alpha_3^2(\beta_3^2+1)+1}}$ \\[12pt] 
  $\cov{\s{x_1}, \s{x_3}} = \tfrac{\alpha_1\beta_1}{\sqrt{\beta_1^2(\alpha_1^2+1)+1}}$ 
& $\cov{\s{x_1}, \s{x_3}} = \tfrac{\alpha_2\beta_2}{\sqrt{(\alpha_2^2 + 1)(\beta_2^2 + 1)}}$ 
& $\cov{\s{x_1}, \s{x_3}} = \tfrac{\alpha_3}{\sqrt{\alpha_3^2(\beta_3^2+1)+1}}$ \\[8pt]  
  $\cov{\s{x_2}, \s{x_3}} = \beta_1 \sqrt{\tfrac{\alpha_1^2 + 1}{\beta_1^2(\alpha_1^2+1)+1}}$ 
& $\cov{\s{x_2}, \s{x_3}} = \tfrac{\beta_2}{\sqrt{\beta_2^2 + 1}} $ 
& $\cov{\s{x_2}, \s{x_3}} = \tfrac{\beta_3}{\sqrt{\beta_3^2 + 1}} $ \\ 
\end{tabular}
\vspace{10pt}
\end{center}
In the standardized SCM \eqref{eq:s1}, 
the causal graph is $v_1 \rightarrow v_2 \rightarrow v_3$. 
Hence, the edge directions of the DAG $\mathcal{G}$ are consistent with the direction of increasing absolute covariance if and only if 
\begin{align}
\begin{split}\label{eq:cov-inequality-algebra}
    \lvert\cov{\s{x_1}, \s{x_2}}\rvert < \lvert\cov{\s{x_2}, \s{x_3}}\rvert 
    &\quad\Longleftrightarrow\quad 
    \left \lvert \tfrac{\alpha_1}{\sqrt{\alpha_1^2 + 1}} \right \rvert < \left \lvert \beta_1 \sqrt{\tfrac{\alpha_1^2 + 1}{\beta_1^2(\alpha_1^2+1)+1}} \right \rvert   \\
    &\quad\Longleftrightarrow\quad 
    \tfrac{\alpha_1^2}{\alpha_1^2 + 1} < \beta_1^2 \tfrac{\alpha_1^2 + 1}{\beta_1^2(\alpha_1^2+1)+1} \\
    &\quad\Longleftrightarrow\quad 
    \alpha_1^2(\beta_1^2(\alpha_1^2+1)+1) < \beta_1^2(\alpha_1^2 + 1)^2\\
    &\quad\Longleftrightarrow\quad 
    \cancel{\beta_1^2\alpha_1^4}+ \cancel{\beta_1^2\alpha_1^2}+\alpha_1^2 < \cancel{\beta_1^2\alpha_1^4} + \cancel{2}\beta_1^2\alpha_1^2 + \beta_1^2\\
    &\quad\Longleftrightarrow\quad 
    \alpha_1^2 < \beta_1^2(\alpha_1^2 + 1)\\
    &\quad\Longleftrightarrow\quad 
    \tfrac{\alpha_1^2}{\alpha_1^2+1} < \beta_1^2 \, .
\end{split}
\end{align}
In the above equivalences, we always multiply or divide by quantities greater than \num{0}, so the direction of the inequality does not change, and transformations are equivalent.
For the standardized SCM \eqref{eq:s3} with causal graph $v_1 \leftarrow v_2 \leftarrow v_3$, we get an analogous condition for the edges to be aligned with the order of increasing absolute covariance when following the same algebraic manipulations:
\begin{align*}
    \lvert\cov{\s{x_3}, \s{x_2}}\rvert < \lvert\cov{\s{x_2}, \s{x_1}}\rvert 
    &\quad\Longleftrightarrow\quad  \tfrac{\beta_3^2}{\beta_3^2+1} < \alpha_3^2.
\end{align*}
We make use of both of these conditions in Section \ref{sec:theory}.
Since $z / (z+1) < 1$ for any $z > 0$, the right-hand sides of both conditions are true if all weights are greater than $1$.
In this case, the absolute covariance increases downstream in all SCMs of Equations \eqref{eq:s1} and \eqref{eq:s3}.
Hence, among these two systems, only the DAG $\mathcal{G}$ whose edges aligns with the covariance ordering in the observed $p(\s{\bf{x}})$ can induce $p(\s{\bf{x}})$, and we can conclude that the other DAG is not the true causal graph. 

\subsubsection*{\ours}

\label{ssec:non-ident-three-nodes}
To derive the observational distributions of the \ourss in Equations
\eqref{eq:s1_ours}, \eqref{eq:s2_ours}, and \eqref{eq:s3_ours}, we proceed in the same way as we did for standardized SCMs.
We first show that the \ours is an SCM with a specific set of mechanisms and then apply Lemma \ref{lemma:cov} to obtain the covariances between the observed variables.
To see this, we write the assignment of $\as{x_i}$ as
\begin{equation}
\label{eq:implied_ours}
    \as{x_i} := 
    \frac{x_i}{\sqrt{\var{x_i}}} = 
    \frac{w_{j, i}\as{x_j} + \noise{i}} {\sqrt{\var{x_i}}} = \frac{w_{j, i}}{\sqrt{\var{x_i}}}\as{x_j} + \frac{\noise{i}}{\sqrt {\var{x_i}}}
\end{equation}
As before, using \cref{eq:implied_ours} requires first computing the marginal variances of the latent variables $x_i$. For the \ourss defined by Equations \eqref{eq:s1_ours}, \eqref{eq:s2_ours}, and \eqref{eq:s3_ours}, they are given by
\begin{center}
\vspace{5pt}
\small
\begin{tabular}{ lll } 
{\normalsize for Equation \eqref{eq:s1_ours}:} \hspace{20pt}
& {\normalsize for Equation \eqref{eq:s2_ours}:} \hspace*{20pt}
& {\normalsize for Equation \eqref{eq:s3_ours}:} \\[10pt]
  $\var{x_1} = \sigma^2$ 
& $\var{x_1} = \gamma_2^2 + \sigma^2 $ 
& $\var{x_1} = \gamma_3^2 + \sigma^2 $ \\[7pt] 
  $\var{x_2} = \gamma_1^2 + \sigma^2$ 
& $\var{x_2} = \sigma^2$ 
& $\var{x_2} = \lambda_3^2 + \sigma^2$ \\[7pt]  
  $\var{x_3} = \lambda_1^2 + \sigma^2 $ 
& $\var{x_3} = \lambda_2^2 + \sigma^2 $ 
& $\var{x_3} = \sigma^2 $ \\ 
\end{tabular}
\vspace{10pt}
\end{center}
Given Equation \eqref{eq:implied_ours} and the marginal variances, we obtain an explicit form for the weights of all three implied SCMs.
Since the implied SCMs are linear, have unit marginal variances, and share the same causal graph, we can apply Lemma \ref{lemma:cov} and obtain the covariances of the observational distributions in the original models.
It turns out that the observational distribution of all three ground-truth systems $(\as{x}_1, \as{x}_2, \as{x}_3)$
in Equations \eqref{eq:s1_ours}, \eqref{eq:s2_ours}, and \eqref{eq:s3_ours}
is a multivariate Gaussian with {\em the same covariance matrix}, with the diagonal elements equal to $1$ and the off-diagonal elements given by

\begin{align}
\begin{split}\label{eq:iscm_covar_mat}
    \cov{\as{x}_1, \as{x}_2} &= \displaystyle \frac{\gamma_i}{\sqrt{\gamma_i^2 + \sigma^2}} \\
    \cov{\as{x}_1, \as{x}_3} &= \displaystyle \frac{\gamma_i \lambda_i}{\sqrt{(\lambda_i^2+\sigma^2)(\gamma_i^2+\sigma^2)}} \\
    \cov{\as{x}_2, \as{x}_3} &= \displaystyle \frac{\lambda_i}{\sqrt{\lambda_i^2 + \sigma^2}} 
\end{split}
\end{align}
Since the observational distribution of all three SCMs is a zero-centered multivariate Gaussian, the distributions are equal if and only if their their covariance matrices are identical. 
The covariances are equal if and only if $\lambda_1 = \lambda_2 = \lambda_3$ and $\gamma_1 = \gamma_2 = \gamma_3$, because the function $f(z) = \smash{z/\sqrt{z^2 + \sigma^2}}$ appearing in $\smash{\cov{\as{x}_1, \as{x}_2}}$ and $\smash{\cov{\as{x}_2, \as{x}_3}}$ of \cref{eq:iscm_covar_mat} is {\em injective} for any $\sigma>0$, which means that distinct weights $z$ are mapped to distinct covariances.
Therefore, the three node linear \ourss in the above MEC share the same observational distribution if and only if they also share the same weights for each edge, regardless of edge orientation. 

This implies that the three DAGs $\mathcal{G}$ in the MEC of Equations \eqref{eq:s1_ours}, \eqref{eq:s2_ours}, and \eqref{eq:s3_ours} are not identifiable from $p(\as{\bf{x}})$: given $p(\as{\bf{x}})$ induced by an \ours with DAG in this 3-node MEC, the two other DAGs with the same linear function weights induce the same distribution $p(\as{\bf{x}})$.

\subsubsection{Forests}
In this section, we generalize the above partial identifiability result for standardized SCMs to arbitrary forest DAGs (Theorem \ref{th:three_part_ident}). 
After that, we similarly generalize the nonidentifiability of \ourss to forests (Theorem \ref{th:three_non_ident}). 
Our results concern the identification edge directions in an MEC represented by its partially directed graph $\smash{\mec = (\nodeset, \tilde{\edgeset})}$, where $\smash{\tilde{\edgeset}}$ contains both directed and undirected edges. 

\subsubsection*{Standardized SCM}

Before proving the main theorem, we extend the 3-node example to chains of arbitrary length. We show that all but at most one edge in the MEC can be correctly oriented  from observational data using the assumption on the support of the weights. Analogous to the three node case, we then use this to prove a similar result for forest graphs. 

\smallskip

\begin{restatable}[Orientation of edges in undirected chains of standardized SCMs]{lemma}{identpath}
\label{lemma:identpath}

Let $\s{\bf{x}}$ be modeled by a standardized linear SCM \eqref{eq:linear_scm} with chain DAG $\mathcal{G} = (\nodeset, \edgeset)$ , where $\var{\noise{i}} = \sigma^2$ for non-root nodes and $\abs{w_{i,j}}>1$ for all $i \in \text{pa}(j)$. 
Additionally, suppose $\mathcal{G}$ contains no colliders.
Then, given $p(\s{\bf{x}})$ and the partially directed graph $\smash{\tilde{\mathcal{G}}}$ representing the MEC of $\mathcal{G}$, 
we can identify all but at most one edge $(v_i, v_j)$ of the true DAG $\mathcal{G}$ in each undirected connected component of the MEC $\mec$.
The possible undirected edge has the smallest absolute covariance of all variables connected by edges in the MEC, satisfying $\smash{\lvert \cov{\s{x}_i,\s{x}_j} \rvert < \lvert \cov{\s{x}_k,\s{x}_l} \rvert}\,$ for all $(k,l) \in \smash{\tilde{\edgeset}} \setminus (i, j)$.
\end{restatable}
\begin{proof}

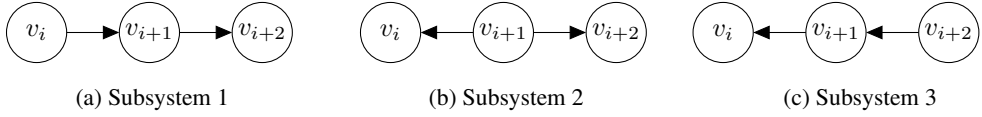
\begin{figure}[!t]
\centering
\begin{subfigure}{0.32\linewidth}
\centering
\begin{tikzpicture}[node distance=1.5cm]

\node[circle,draw, minimum size=0.8cm, inner sep=0cm,outer sep=0] (circle1) {$v_i$};
\node[circle,draw,right of=circle1,  minimum size=0.8cm, inner sep=0cm,outer sep=0] (circle2) {$v_{i+1}$};
\node[circle,draw,right of=circle2,  minimum size=0.8cm, inner sep=0cm,outer sep=0] (circle3) {$v_{i+2}$};

\draw[->] (circle1) -- (circle2);
\draw[->] (circle2) -- (circle3);

\end{tikzpicture}
\vspace{3pt}
\caption{Subsystem 1}\label{fig:chain_sub_systems-1}
\end{subfigure}
\hfill
\begin{subfigure}{0.32\linewidth}
\centering
\begin{tikzpicture}[node distance=1.5cm]

\node[circle,draw, minimum size=0.8cm, inner sep=0cm,outer sep=0] (circle1) {$v_i$};
\node[circle,draw,right of=circle1,  minimum size=0.8cm, inner sep=0cm,outer sep=0] (circle2) {$v_{i+1}$};
\node[circle,draw,right of=circle2,  minimum size=0.8cm, inner sep=0cm,outer sep=0] (circle3) {$v_{i+2}$};

\draw[<-] (circle1) -- (circle2);
\draw[->] (circle2) -- (circle3);

\end{tikzpicture}
\vspace{3pt}
\caption{Subsystem 2}\label{fig:chain_sub_systems-2}
\end{subfigure}
\hfill
\begin{subfigure}{0.32\linewidth}
\begin{tikzpicture}[node distance=1.5cm]

\node[circle,draw, minimum size=0.8cm, inner sep=0cm,outer sep=0] (circle1) {$v_i$};
\node[circle,draw,right of=circle1,  minimum size=0.8cm, inner sep=0cm,outer sep=0] (circle2) {$v_{i+1}$};
\node[circle,draw,right of=circle2,  minimum size=0.8cm, inner sep=0cm,outer sep=0] (circle3) {$v_{i+2}$};

\draw[<-] (circle1) -- (circle2);
\draw[<-] (circle2) -- (circle3);

\end{tikzpicture}        
\vspace{3pt}
\caption{Subsystem 3}\label{fig:chain_sub_systems-3}
\end{subfigure}
\vspace{5pt}
\caption{\textbf{Proof subcases of \cref{lemma:identpath}.} Three possible subgraphs in a chain without a collider.}
\label{fig:chain_sub_systems}
\end{figure}

Throughout the proof, we label the nodes $v_i \in \nodeset$ such that $v_{i-1}$ and $v_{i+1}$ are its neighbors for $i \in \{2, \dots, \numvars-1\}$. 
We start with the analysis of three arbitrary, consecutive vertices in a chain graph. The three possible subgraphs are depicted in Figure \ref{fig:chain_sub_systems}. 
We can always find $p \in \R$ such that the variance of the latent root of this directed subgraph is $p^2\sigma^2$. 
This relaxed assumption on specifically the root node allows for the root of the subgraph to have potential parents outside the subgraph, or to be the root of the whole chain, when later using this lemma to prove the main theorem.

We will follow similar derivations as in Section \ref{ssec:ident_3node}.
Specifically, we first write the observed variables of the standardized SCM in SCM form, and then invoke Lemma \ref{lemma:cov} to obtain the covariances of the observed variables.
To use \cref{eq:implied_stand}, we again need to compute the marginal variances of the variables before standardization. For the subsystems in Figures \ref{fig:chain_sub_systems-1} and \ref{fig:chain_sub_systems-2}, these are, respectively:

\begin{center}
\vspace{5pt}
\small
\begin{tabular}{ ll } 
{\normalsize for \cref{fig:chain_sub_systems-1}:} \hspace{20pt}
& {\normalsize for \cref{fig:chain_sub_systems-2}:} \hspace{20pt} \\[10pt]
  $\var{x_i} = p^2\sigma^2$ 
& $\var{x_i} = (w_{i+1, i}^2p^2 + 1)\sigma^2 $ \\[7pt] 
  $\var{x_{i+1}} = (w_{i, i+1}^2p^2 + 1)\sigma^2 $ 
& $\var{x_{i+1}} = p^2\sigma^2$ \\[7pt]  
  $\var{x_{i+2}} = (w_{i+1, i+2}^2(w_{i, i+1}^2p^2+1)+1)\sigma^2 $ 
& $\var{x_{i+2}} = (w_{i+1, i+2}^2p^2 + 1)\sigma^2  $  
\end{tabular}
\vspace{10pt}
\end{center}

By substituting the expressions for the marginal variances into \cref{eq:implied_stand}, 
we obtain the weights of the implied models of the standardized SCM.
Using Lemma \ref{lemma:cov}, we obtain the covariances between the observed variables $\s{x_{i-1}}, \s{x_i}, \s{x_{i+1}}$. By construction, the marginal variances of the observed variables are equal to $1$.
We treat each subsystem separately:
\paragraph{Subsystem 1 {\normalfont (Figure \ref{fig:chain_sub_systems-1})}}
Given the marginal variances and Lemma \ref{lemma:cov}, the covariances are

\begin{align*}
    \cov{\s{x}_i, \s{x}_{i+1}} &= \frac{w_{i, i+1}p}{\sqrt{w_{i, i+1}^2p^2 + 1}} \\
    \cov{\s{x}_{i+1}, \s{x}_{i+2}} &= w_{i+1, i+2} \sqrt{\frac{w_{i, i+1}^2p^2 + 1}{w_{i+1, i+2}^2(w_{i, i+1}^2p^2+1)+1}}
\end{align*}

Following the same algebraic manipulations as in \cref{eq:cov-inequality-algebra}, substituting $\alpha_1 := w_{i,i+1}p$ and $\beta_1 := w_{i+1,i+2}$ in the derivation, we obtain
\begin{align}\label{eq:chain_one_cond}
    \abs{\cov{\s{x}_i, \s{x}_{i+1}}} < \abs{\cov{\s{x}_{i+1}, \s{x}_{i+2}}}
    \quad\Longleftrightarrow\quad
    \frac{w_{i, i+1}^2p^2}{w_{i, i+1}^2p^2 + 1} < w_{i+1, i+2}^2 \, .
\end{align}
The left-hand side of the right-hand inequality in \cref{eq:chain_one_cond} is upper-bounded by $1$, similar to the 3-node case.
Therefore, if we assume that $\lvert w_{i+1, i+2} \rvert \geq 1$, it must hold that $\lvert{\cov{\s{x}_i, \s{x}_{i+1}}}\rvert < \lvert{\cov{\s{x}_{i+1}, \s{x}_{i+2}}}\rvert$ for any choice of $p$.

\paragraph{Subsystem 2 {\normalfont (Figure \ref{fig:chain_sub_systems-2})}} Given the marginal variances and Lemma \ref{lemma:cov}, the covariances are

\begin{align*}
    \cov{\s{x}_i, \s{x}_{i+1}} &= \frac{w_{i+1, i}p}{\sqrt{w_{i+1, i}^2p^2 + 1}} \\
    \cov{\s{x}_{i+1}, \s{x}_{i+2}} &= \frac{w_{i+1, i+2}p}{\sqrt{w_{i+1, i+2}^2p^2 + 1}} \, .
\end{align*}
The ordering of the covariances in this case depends on the specific choice of the weights.

\paragraph{Subsystem 3 {\normalfont (Figure \ref{fig:chain_sub_systems-3})}}
Following steps analogous to the symmetric subsystem 1, we conclude that, if $\lvert{w_{i+1, i}}\rvert \geq 1$, it must hold that $\lvert{\cov{\s{x}_i, \s{x}_{i+1}}}\rvert > \lvert{\cov{\s{x}_{i+1}, \s{x}_{i+2}}}\rvert$ for any $p$.

\bigskip

Given the above, we can now study the relationship between the underlying DAG $\mathcal{G}$ and the absolute covariance magnitudes under the assumption that $\smash{\abss{w_{i,i+1}}}>1$.
We will use the fact that, if the chain does not contain a collider, then there can be at most one node contained in edges pointing in opposite directions. 

First, we treat the case where there exists a vertex $v_i$ such that $\abss{\cov{\s{x}_{i-1}, \s{x}_{i}}} = \abss{\cov{\s{x}_{i}, \s{x}_{i+1}}}$, that is, where some neighboring covariances are equal.
If this occurs in a 3-node subsystem, only subsystem 2 can describe the true graph. 
To be consistent with the assumption that there are no colliders in the graph (see Lemma \ref{lemma:identpath}), all other edges must be oriented in a direction away from $v_i$, which  completely identifies the graph $\mathcal{G}$ in the MEC. 

In the second case,  $\abss{\cov{\s{x}_{j-1}, \s{x}_{j}}} \neq \abss{\cov{\s{x}_{j}, \s{x}_{j+1}}}$ holds for all nodes $v_j$ that have two neighbors in the path.
Let $\s{x}_i, \s{x}_{i+1}$ be the unique pair of consecutive variables in the chain that minimizes $\abss{\cov{\s{x}_i, \s{x}_{i+1}}}$.  
We can show that this pair is the unique minimizer using a proof by contradiction. Suppose there exist two pairs $\s{x}_i, \s{x}_{i+1}$ and $\s{x}_j, \s{x}_{j+1}$ such that $\abss{\cov{\s{x}_i, \s{x}_{i+1}}} = \abss{\cov{\s{x}_j, \s{x}_{j+1}}}$ is the minimum covariance. 
Without loss of generality, let $j+1<i$. 
Then, the triple $\s{x}_{i-1}, \s{x}_i, \s{x}_{i+1}$ is consistent with only subsystems 2 or 3 based on their relative covariances, which implies that we must have $v_{i-1} \leftarrow v_{i}$. Using the fact that we have no colliders, we can then orient all edges $v_{k-1} \leftarrow v_{k}$ for $1 < k < i$. Thus, we can find a subsystem containing $v_j, v_{j+1}, v_{j+2}$, which has been already oriented as subsystem 3, meaning $\abss{\cov{\s{x}_{j}, \s{x}_{j+1}}} > \abss{\cov{\s{x}_{j+1}, \s{x}_{j+2}}}$, a contradiction.

Given $\s{x}_i, \s{x}_{i+1}$ is the unique pair of consecutive variables that minimizes $\lvert {\cov{\s{x}_i, \s{x}_{i+1}}\rvert}$, 
we now show that we can orient all edges except $(v_i, v_{i+1})$. 
We will do this in two parts. 
First, we show that one can orient all edges $(v_{j}, v_{j+1})$ with $j < i$, and then we show that we can do the same for all edges $(v_{j}, v_{j+1})$ with $j > i$. 
If $i>1$, consider the subsystem $v_{i-1}, v_i, v_{i+1}$. 
Since $\smash{\lvert{\cov{\s{x}_{i-1}, \s{x}_{i}}}\rvert } > \smash{\lvert {\cov{\s{x}_i, \s{x}_{i+1}}}\lvert } $, 
only subsystems 2 and 3 are possible for this subgraph. 
We can therefore orient $v_{i-1} \leftarrow v_i$. 
Similarly, if $i<d-1$, by a symmetric argument on $v_i, v_{i+1}, v_{i+2}$, we can orient $v_{i+1} \rightarrow v_{i+2}$. 
Since the graph cannot contain colliders, all other edges must be oriented as $v_j \leftarrow v_{j+1}$ for $j < i$, and $v_j \rightarrow v_{j+1}$ for $j>i$. 
In other words, all edges except $(v_i, v_{i+1})$ point away from the two vertices $v_i, v_{i+1}$, and one of the two variables must be the root of the chain.
Therefore, if $\smash{\abss{\cov{\s{x}_{j-1}, \s{x}_{j}}} \neq \abss{\cov{\s{x}_{j}, \s{x}_{j+1}}}}$ holds for all vertices $v_j$ that have two neighbors, then there exists a unique covariance minimizing pair $\s{x}_i, \s{x}_{i+1}$, and
all edges except $(v_i, v_{i+1})$ are oriented. 

The two cases above are exhaustive, and in the worst case at most one edge $(v_{j}, v_{j+1})$ is left unoriented in the chain. This edge always corresponds to the minimizer of $\abss{\cov{\s{x}_{j}, \s{x}_{j+1}}}$.
This completes the proof. 
\end{proof}

\paragraph{Remark} From the proof of Lemma \ref{lemma:identpath}, it follows that if we are able to orient all the edges in the chain, then the root of the chain is the node joining the two edges with minimum absolute covariance. 
When we orient all but one edge $(v_i, v_{i+1})$, the root node of the chain is either $v_i$ or $v_j$. 

We can extend Lemma \ref{lemma:identpath} to forest graphs. 
For this, we will make use of the first Meek rule \citep{meek1995causal}. 
The first Meek rule concerns an MEC $\mec$, containing the undirected edges $(v_i, v_j), (v_j, v_k)$ but not the edge $(v_i, v_k)$. It states that, if one can orient $v_i\rightarrow v_j$, we must have $v_j \rightarrow v_k$. 

\bigskip

\forestpartident*
\begin{proof}

The undirected parts of an MEC $\mec$ are disjoint undirected connected components.
Orienting the edges in all these undirected connected components without introducing a v-structure produces a valid DAG $\mathcal{G}$ in $\mec$ \citep{andersson1997characterization}. Each undirected connected components represents a Markov equivalence class of its own \citep{andersson1997characterization}. 
Thus, to prove the theorem, we consider these undirected connected components independently with respect to the rest of the graph and show how to orient the edges in each undirected connected component.%
\footnote{Orienting edges of an undirected connected component that touch a directed edge in $\mec$ never introduces an additional v-structure. If a directed edge pointed into the undirected connected component, the undirected edge downstream would have had to already be directed in $\mec$ by the first Meek rule. Hence, all directed edges bordering the undirected connected component must be oriented away from it, and none of the possible undirected edge orientations creates a new collider at the border node. 
This implies that all undirected connected components in $\mec$ are upstream of the colliders and directed subgraphs of $\mec$.}
In the following argument, we therefore consider $\mec$ to be a single undirected connected component, with no directed edges by definition, and show that we can orient all but one edge in $\mec$.
This argument then extends to all undirected connected components of the original MEC $\mec$, implying the statement made in Theorem~\ref{th:three_part_ident}.

If $\mec$ is an undirected connected component with no directed edges, we only have to consider SCMs with a ground-truth DAG $\mathcal{G}$ that are members of this MEC $\mec$ to distinguish among possible edge orientations in $\mec$.
In the case of undirected trees, the ground-truth DAG $\mathcal{G}$ must be a tree with no colliders and the same skeleton as $\mec$, since any other DAGs would belong to a different MEC.

\begin{figure}[!t]
\centering
\vspace{-10pt}
\begin{tikzpicture}[node distance=1.8cm, auto]
minimum size=0.8cm
\node[draw, circle] (v1)[minimum size=0.85cm, inner sep=0cm,outer sep=0] {$v_1$};
\node[draw, circle, right of=v1] (v2)[minimum size=0.85cm, inner sep=0cm,outer sep=0] {$v_{i-1}$};
\node[draw, circle, right of=v2] (vi)[minimum size=0.85cm, inner sep=0cm,outer sep=0] {$v_i$};
\node[draw, circle, right of=vi] (vj)[minimum size=0.85cm, inner sep=0cm,outer sep=0] {$v_{i+1}$};
\node[draw, circle, right of=vj] (v3)[minimum size=0.85cm, inner sep=0cm,outer sep=0] {$v_{i+2}$};
\node[draw, circle, right of=v3] (v4)[minimum size=0.85cm, inner sep=0cm,outer sep=0] {$v_{k}$};

\node[draw, circle, below of=v2] (v21) [minimum size=0.85cm, inner sep=0cm,outer sep=0, xshift=-0.8cm] {$u$};
\node[draw, circle, above of=vi] (vi1) [minimum size=0.85cm, inner sep=0cm,outer sep=0, xshift=-0.8cm] {};
\node[draw, circle, below of=vi] (vi2) [minimum size=0.85cm, inner sep=0cm,outer sep=0, xshift=-0.0cm] {};
\node[draw, circle, above of=vj] (vj1) [minimum size=0.85cm, inner sep=0cm,outer sep=0, xshift=0.5cm] {};
\node[draw, circle, above of=vj1] (vj11) [minimum size=0.85cm, inner sep=0cm,outer sep=0, xshift=-1.3cm, yshift=-0.4cm] {};
\node[draw, circle, right of=vj1] (vj12) [minimum size=0.85cm, inner sep=0cm,outer sep=0, yshift=0.0cm] {};
\node[draw, circle, below of=vj] (vj2) [minimum size=0.85cm, inner sep=0cm,outer sep=0, xshift=0.8cm] {};

\draw[->] (v2) -- (v1);
\draw[->] (v2) -- (v21);
\draw[->] (vi) -- (v2);
\draw[->, blue] (vi) -- (vj);
\draw[->] (vj) -- (v3);
\draw[->] (v3) -- (v4);
\draw[->] (vi) -- (vi1);
\draw[->] (vi2) -- (vi);
\draw[->] (vj) -- (vj1);
\draw[->] (vj) -- (vj2);
\draw[->] (vj1) -- (vj11);
\draw[->] (vj1) -- (vj12);

\begin{pgfonlayer}{background}

    \fill[gray!15, rounded corners=20pt] 
        ($(vi.south west) + (-0.4,-2.3)$) rectangle 
        ($(vi.north east) + (0.4,0.4)$);

    \fill[gray!15, rounded corners=20pt, rotate around={25:(vi)}] 
        ($(vi.south west) + (-0.3,-0.6)$) rectangle 
        ($(vi.north east) + (0.3,2.5)$);

    \fill[gray!15, rounded corners=20pt, rotate around={25:(vj2)}] 
        ($(vj2.south west) + (-0.3,-0.6)$) rectangle 
        ($(vj2.north east) + (0.3,2.5)$);
    \fill[gray!15, rounded corners=20pt, rotate around={42:(vj1)}] 
        ($(vj1.south west) + (-0.3,-0.6)$) rectangle 
        ($(vj1.north east) + (0.3,2.5)$);
    \fill[gray!15, rounded corners=20pt, rotate around={90:(vj1)}] 
        ($(vj12.south west) + (-0.4,-1.0)$) rectangle 
        ($(vj12.north east) + (0.4,2.8)$);
    \fill[gray!15, rounded corners=20pt, rotate around={-16:(vj)}] 
        ($(vj.south west) + (-0.5,-0.2)$) rectangle 
        ($(vj.north east) + (0.4,2.0)$);

\end{pgfonlayer}

\end{tikzpicture}
\vspace{10pt}
\caption{
\textbf{Inductive step of the proof of \cref{th:three_part_ident}.}
Ground-truth DAG $\mathcal{G}$ underlying an undirected connected component $\mec$ in some given MEC.
The nodes $\nodeset_C = \{ v_1, \dots, v_k\}$ are a longest chain in $\mathcal{G}$.
Using \cref{lemma:identpath}, we 
can orient all edges in $\smash{\mec_C}$ except possibly $(v_i, v_{i+1})$ (blue).
Edges like $(v_{i-1}, u)$ are oriented by the first Meek rule.
After \cref{lemma:identpath}, we are left with either the single undirected tree of $v_i$ (left shaded tree) or the single undirected tree consisting of $(v_i, v_{i+1})$ (blue) and both undirected trees of $v_i$ and $v_{i+1}$ (both shaded trees). 
Either $v_i$ or $ v_{i+1}$ must be the root of $\smash{\mathcal{G}_C}$.
In this specific example, $v_i$ is the root of $\smash{\mathcal{G}_C}$ and is therefore the only node that can have a parent outside $\mathcal{G}_C$.
Any node in $\mathcal{G}$ can have directed, outgoing edges to children in a (possibly non-forest) MEC the undirected connected component $\mec$ may be a subgraph of.
}\label{fig:theorem-3-proof-intuition}
\end{figure}
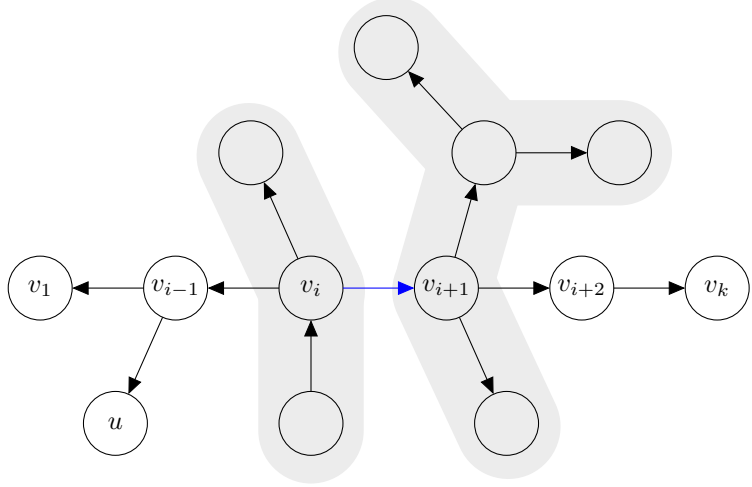

We give a proof by strong induction on the number of vertices $\abss{\mathcal{V}}$ in the MEC $\mec$.
The base case of the induction argument is an MEC with $\abss{\mathcal{V}} = 2$ nodes. 
This case holds trivially, since this MEC can contain at most one undirected edge. 
For the inductive step, we consider an undirected tree MEC $\mec$ with $\abs{\nodeset} = \numvars$ and assume that we can orient all but one edge of undirected tree MECs with $\abs{\nodeset} < \numvars$. 

Our argument will proceed by considering the longest chain of the undirected tree $\mec$.
We will use \cref{lemma:identpath} to orient all but at most one edge in this chain and then apply the first Meek rule to possibly orient additional edges in $\mec$ outside the chain.
After orienting these edges, we show that we reduced the original problem of orienting all but one edge in $\mec$ with $\abs{\nodeset} = d$ 
to orienting all but one edge in a single undirected connected component that has strictly fewer than $\numvars$ nodes. This allows us to apply the inductive hypothesis and complete the proof (see Figure \ref{fig:theorem-3-proof-intuition}). 

Consider a longest undirected chain $\tilde{\mathcal{G}}_C  = (\nodeset_{C}, \tilde{\edgeset}_{C})$ that is a subgraph of the undirected tree $\mec$. 
Let $\smash{\mathcal{G}_C}$ refer to the  directed subgraph of the DAG $\mathcal{G}$ induced by considering only the vertices $\nodeset_C$. 
We label the $k$ vertices in $\nodeset_{C}$ as $v_1, ..., v_k$, with undirected edges $(v_{i}, v_{i+1}) \in \smash{\tilde{\edgeset}}$ for all $i \in \{1, \dots, k-1 \}$. 
The nodes $v_1, v_k$ can have no undirected neighbours in $\mec$ outside the chain, because otherwise we could construct a longer chain in $\mec$. 

The only vertex in $\nodeset_C$ that can have a parent in the DAG $\mathcal{G}$ outside the chain $\smash{\mathcal{G}_C}$, that is, in $\nodeset \smash{\setminus} \nodeset_C$, is the unique root of $\smash{\mathcal{G}_C}$.
To see this, we first note that all nodes $v_i$ have at most one parent in $\mathcal{G}$, because any $v_i$ with $\abs{\text{pa}(v_i))} > 1$ in $\mathcal{G}$ would be a collider, but $\mathcal{G}$ contains no colliders.
Since non-root nodes in $\smash{\mathcal{G}_C}$ have an in-chain parent, they cannot have a parent outside of $\nodeset_C$.
Therefore, besides the root node of $\smash{\mathcal{G}_C}$ via its potential outside parent, $\smash{\mathcal{G}_C}$ is a completely disconnected subgraph from the rest of $\mathcal{G}$.
This implies that we may treat $\smash{\mathcal{G}_C}$ as a separate standardized SCM with undirected chain MEC, in which the potential parent of the root of $\smash{\mathcal{G}_C}$ is modeled as part of the exogenous noise of the root. This allows us to apply \cref{lemma:identpath} to the variables of the subgraph $\smash{\mathcal{G}_C}$. 

By applying \cref{lemma:identpath} to $\smash{\mathcal{G}_C}$, we can orient all but at most one undirected edge in $\smash{\mathcal{\tilde{G}}_C}$.
We split the resulting analysis into the two cases of \cref{lemma:identpath} leaving either $0$ or $1$ undirected edge. 
In the first case, we can orient all edges in $\smash{\mathcal{\tilde{G}}_C}$ with \cref{lemma:identpath}. 
In this case, we know that the root of $\smash{\mathcal{G}_C}$ is the node $v_i$ (see {\em Remark} of \cref{lemma:identpath}).
By the first Meek rule, we can recursively orient all additional edges in $\mec$ outside of $\smash{\mathcal{\tilde{G}}_C}$ away from $v_i$, except for the subtrees of $\mec$ connected to $v_i$ itself (Figure \ref{fig:theorem-3-proof-intuition}). 
This leaves at most a single connected undirected subtree containing $v_i$ and strictly less than $d$ vertices.

In the second case, we orient all but one edge $(v_i, v_{i+1})$ in $\smash{\mathcal{\tilde{G}}_C}$ by applying \cref{lemma:identpath}.
In this case, we know that the root of $\smash{\mathcal{G}_C}$ is either the node $v_i$ or $v_{i+1}$ (see {\em Remark} of \cref{lemma:identpath}).
Similar to the first case, we can recursively use the first Meek rule to orient all additional edges in $\mec$ pointing away from $v_i$ and $v_{i+1}$, except for the subtrees of $\mec$ connected to $v_i$ and $v_{i+1}$ itself. 
Since $v_{i}$ and $v_{i+1}$ are connected by an undirected edge, we are left with a single connected subtree containing the undirected edge  $(v_i, v_{i+1})$ that is strictly smaller than before.

In both cases, we orient at least one undirected edge of $\mec$, because the longest undirected chain in $\mec$ with $\abs{\nodeset} > 2$ has at least length $2$.
We always obtain at most a single undirected connected tree component with strictly less than $d$ vertices, allowing us to apply the inductive hypothesis and complete the proof. 

\end{proof}

\subsubsection*{iSCM}

\smallskip

\forestnonident*

\begin{proof}
Because we consider linear \ourss with Gaussian noise, the implied model is a linear SCM with additive Gaussian noise (see \cref{ssec:implied_models_iscm}). Hence, the observational distribution is a multivariate Gaussian with mean zero.
In \ourss, the marginal variance of an observed variable is always $1$. 
Hence, we prove the statement if we show that for all $\as{x}_i, \as{x}_j$ in the \ours with graph $\mathcal{G}$, and the corresponding $\as{x}_i', \as{x}_j'$ in the \ours with graph $\mathcal{G}'= (\nodeset, \edgeset')$, $\cov{\as{x}_i, \as{x}_j} = \cov{\as{x}_i', \as{x}_j'}$. 

Let $\as{x}_{i}'$ and $\smash{\as{x}_{j}'}$ be the random variables associated with the nodes $v_i$ and $v_j$ from $\mathcal{G}'$, respectively. We consider two cases. First, if there is no path between $v_i$ and $v_j$ in the skeleton of $\mathcal{G}'$ then there is no path between $v_i$ and $v_j$ in the skeleton of $\mathcal{G}$ and hence $\cov{\as{x}_{i}, \as{x}_{j}} = \cov{\as{x}_{i}', \as{x}_{j}'} = 0.$
In the second case, there is a path between $v_i$ and $v_j$ in the skeleton of $\mathcal{G}'$, so there also exists a path in the skeleton of $\mathcal{G}$, as both graphs have the same skeleton. 
Due to the acyclicity of the skeleton in forests, this path is the only one connecting $v_i$ and $v_j$ in both $\mathcal{G}$ and $\mathcal{G}'$.

\begin{figure}[!t]
\centering
\begin{subfigure}[b]{0.99\linewidth}
    \centering
    \begin{tikzpicture}[node distance=1.3cm, auto]
        \node[draw, circle] (vi)[minimum size=0.7cm] {$v_i$};
        \node[draw, circle, right of=vi] (B)[minimum size=0.7cm] {};
        \node[draw, circle, right of=B] (C)[minimum size=0.7cm] {$v_k$};
        \node[draw, circle, right of=C] (D)[minimum size=0.7cm] {};
        \node[draw, circle, right of=D] (vj)[minimum size=0.7cm] {$v_j$};
        
        \draw[dotted] (vi) -- (B);
        \draw[->] (B) -- (C);
        \draw[->] (D) -- (C);
        \draw[dotted] (vj) -- (D);
    \end{tikzpicture}
    \vspace{5pt}
    \caption{First subcase}
    \vspace{5pt}
    \label{fig:collider-in-path}
\end{subfigure}
\hfill
\begin{subfigure}[b]{0.49\linewidth}
    \centering
    \begin{tikzpicture}[node distance=1.3cm, auto]
        \node[draw, circle] (vi)[minimum size=0.7cm] {$v_i$};
        \node[draw, circle, right of=vi] (vl)[minimum size=0.7cm] {$v_l$};
        \node[draw, circle, right of=vl] (vk)[minimum size=0.7cm] {$v_k$};
        \node[draw, circle, above of=vk] (vp)[minimum size=0.7cm] {$v_p$};
        \node[draw, circle, right of=vk] (D)[minimum size=0.7cm] {};
        \node[draw, circle, right of=D] (vj)[minimum size=0.7cm] {$v_j$};
        
        \draw[dotted] (vi) -- (vl);
        \draw[->] (vl) -- (vk);
        \draw[->] (vp) -- (vk);
        \draw[->] (vk) -- (D);
        \draw[dotted] (vj) -- (D);
    \end{tikzpicture}
    \vspace{5pt}
    \caption{Second subcase (More than one parent in $\mathcal{G}'$)}
    \label{fig:multiple-parents-in-tree}
\end{subfigure}
\hfill
\begin{subfigure}[b]{0.49\linewidth}
    \centering
    \begin{tikzpicture}[node distance=1.3cm, auto]
        \node[draw, circle] (vi)[minimum size=0.7cm] {$v_i$};
        \node[draw, circle, right of=vi] (vl)[minimum size=0.7cm] {$v_l$};
        \node[draw, circle, right of=vl] (vk)[minimum size=0.7cm] {$v_k$};
        \node[draw, circle, right of=vk] (D)[minimum size=0.7cm] {};
        \node[draw, circle, right of=D] (vj)[minimum size=0.7cm] {$v_j$};
        
        \draw[dotted] (vi) -- (vl);
        \draw[->] (vl) -- (vk);
        \draw[->] (vk) -- (D);
        \draw[dotted] (vj) -- (D);
    \end{tikzpicture}
    \vspace{5pt}
    \caption{Second subcase (A single parent in $\mathcal{G}'$)}
    \label{fig:single-parent-in-tree}
\end{subfigure}
\vspace{5pt}
\caption{
\textbf{Proof subcases of \cref{th:three_non_ident}.}
(a) Path with a collider. In other words, a path blocked by an empty set. In the case of forests, this configuration implies that $v_i$ and $v_j$ are $d$-separated.
(b) Unblocked path connecting $v_i$ and $v_j$ with one of the path nodes having a parent both in the path and outside the path. The weight $w_{p, k}$ influences the weight $\as{w}_{l, k}$ in the implied model of the \ours. 
If this structure is present in a forest, it has to be present in other graphs in the same MEC.
(c) Unblocked path connecting $v_i$ and $v_j$ with the only parent of $v_k$ being part of the considered path. The weight $\as{w}_{l, k}$ depends only on $w_{l, k}$, irrespective of the edge direction.
}
\end{figure}
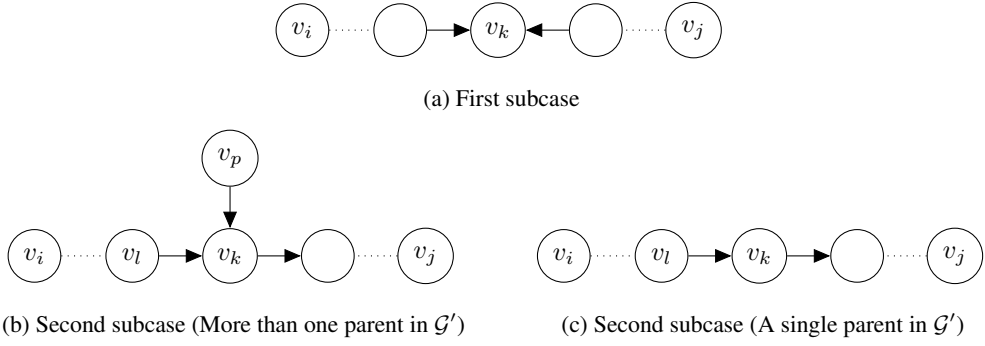

We further break this second case into two subcases. In the first subcase, this path contains a collider in $\mathcal{G}$ as shown in Figure \ref{fig:collider-in-path}. 
Because the skeleton cannot have undirected cycles under the forest assumption, this collider forms a $v$-structure. $\mathcal{G}' \in \mec$ implies that the same $v$-structure must be present in $\mathcal{G}$. Hence, $v_i$ and $v_j$ are $d$-separated in both $\mathcal{G}$ and $\mathcal{G}'$. By the global Markov condition, this implies that $\as{x}_{i}'$ and $\as{x}_{j}'$ are independent, and that $\as{x}_{i}$ and $\as{x}_{j}$ are independent. This implies that both $\cov{\as{x}_{i}', \as{x}_{j}'} = \cov{\as{x}_{i}, \as{x}_{i}} = 0$.

In the second subcase, there exists an unblocked path between $v_i$ and $v_j$ in both $\mathcal{G}$ and $\mathcal{G}'$. Here, we denote the weight matrix associated with both \ourss by $W := [w_{i, j}]$, with $W$ being symmetric, so that $w_{i, j} = w_{j, i}$ is the linear weight of the edge $(i, j)$ regardless of its orientation in the graph.

We now derive the analogous weights $\as{W}, \as{W}'$ in the implied SCMs for $\mathcal{G}, \mathcal{G}'$ respectively. 
Ultimately, we will demonstrate that the implied SCMs have the same weights.
Specifically, we will show that $\smash{\as{w}_{k,l}} = \smash{\as{w}_{k, l}'}$. 
Given this, \cref{lemma:cov} implies that both \ourss have the same covariance matrix over the observed variables.

Without loss of generality, since the node labelling is arbitrary, let $v_k$ have at least as many incoming edges as $v_l$ in $\mathcal{G}'$. 
We divide the analysis into two cases: $v_k$ having only $1$ parent in $\mathcal{G}'$, and $v_k$ having more than $1$ parent.
The node $v_k$ must have at least one parent, since at least one of $v_k, v_l$ have an incoming edge in $\mathcal{G}'$, and we chose $v_k$ to have at least as many incoming edges as $v_l$. 

\paragraph{More than one parent in $\mathcal{G}'$} We know that any collider in $\mathcal{G}'$ will appear as part of a $v$-structure in $\mec$ due to the forest assumption, and therefore will also be a collider in  $\mathcal{G}$. Therefore, if $v_k$ has more than one parent in $\mathcal{G}'$ (see \cref{fig:multiple-parents-in-tree}), all pairs of edges incoming to $v_k$ will form $v$-structures, so $v_k$ must have exactly the same set of parents in $\mathcal{G}$. 

Moreover, any two parents of $v_k$ are d-separated in $\mathcal{G}$ and $ \mathcal{G}'$ by the forest assumption, since the blocked path going through $v_k$ is the only path connecting them. By the global Markov condition, the parents are pairwise independent. Hence, we can use Equation \eqref{eq:weights_implied_ours_indep} to compute $\as{w}_{k, l}, \as{w}_{k, l}'$. Since the parent sets are the same between the two graphs, and $W$ is shared between the two \ourss, the weight associated with the edge $(l, k)$ in both graphs in the implied models is given by
\begin{equation}
\label{eq:tree_proof_more_parents}
    \as{w}_{l, k} = \as{w}_{l, k}' = \frac{w_{l, k}}{\sqrt{\sum_{u \in \paset{k}}w_{u, k}^2 + \sigma^2}} \, .
\end{equation}

\paragraph{A single parent in $\mathcal{G}'$} Let $(l, k)$ be the only incoming edge to $v_k$ in $\mathcal{G}'$, as depicted in Figure \ref{fig:single-parent-in-tree}. 
Then, the edge connecting $v_l$ and $v_k$ in $\mathcal{G}$ is either the only incoming edge to $v_k$ or the only incoming edge to $v_l$. To see this, suppose that it was not the only incoming edge to $v_k$ or $v_l$ in $\mathcal{G}$. This would make $v_k$ or $v_l$ a collider that would be common to both graphs, implying that $v_k$ or $v_l$ would have at least two parents in $\mathcal{G}'$. We operate under the assumption that $v_k$ has at least as many parents as $v_l$, so it would imply that $v_k$ has more than one parent, contradicting the assumption we made for case we consider in this paragraph. Irrespective of the direction, the weight associated with the edge $(l, k)$ in the skeleton of both graphs in the implied model is, similar to Equation \eqref{eq:implied_ours}, given by
\begin{equation}
\label{eq:tree_proof_single_parents}
    \as{w}_{l,k} = \as{w}_{l, k}' = \frac{w_{l, k}}{\sqrt{w_{l, k}^2 + \sigma^2}} \, .
\end{equation}
Equations \eqref{eq:tree_proof_more_parents} and \eqref{eq:tree_proof_single_parents}  show that, for the SCM form of each \ours, the edges connecting the same nodes irrespective of their direction in $\mathcal{G}'$ and $\mathcal{G}$ have the same weights. 
By Lemma \ref{lemma:cov}, the covariance between any $\as{x}_{i}$ and $\as{x}_{j}$ can be expressed as a product of the weights in the implied SCM corresponding to the edges on the path between $v_i, v_j$. Hence, $\cov{\as{x}_{i}, \as{x}_{j}} = \cov{\as{x}_{i}', \as{x}_{j}'}$.
\end{proof}

Figure \ref{fig:tree-id} shows an example for Theorem \ref{th:three_non_ident} for two trees from the same MEC.

\begin{figure}[t]
    \centering
\hspace{10pt}
\begin{subfigure}[c]{0.28\textwidth}
    \begin{tikzpicture}[node distance=2cm]
    \node[circle,draw] (x1) at (0.2,-2.4) {$v_{1}$};
    \node[circle,draw] (x2) at (1,-1.2) {$v_{2}$};
    \node[circle,draw] (x3) at (1.8,0) {$v_{3}$};
    \node[circle,draw] (x4) at (1.8,-1.7) {$v_{4}$};
    \node[circle,draw] (x5) at (2.6,-1.2) {$v_{5}$};
    \node[circle,draw] (x6) at (3.4,0) {$v_{6}$};
    \draw[->] (x2) edge node[above] {$1$~~~} (x1);
    \draw[->] (x3) edge node[above] {$2$~~~} (x2);
    \draw[->] (x3) edge node[right] {$3$~} (x4);
    \draw[->] (x3) edge node[above] {~~$4$} (x5);
    \draw[->] (x6) edge node[above] {$5$~~} (x5);
\end{tikzpicture}
\end{subfigure}
\hfill
\begin{subfigure}[c]{0.28\textwidth}
\begin{tikzpicture}[node distance=2cm]
    \node[circle,draw] (x1) at (0.2,0) {$v_{1}$};
    \node[circle,draw] (x2) at (1,1.2) {$v_{2}$};
    \node[circle,draw] (x3) at (1.8,0) {$v_{3}$};
    \node[circle,draw] (x4) at (1,-1.2) {$v_{4}$};
    \node[circle,draw] (x5) at (2.6,-1.2) {$v_{5}$};
    \node[circle,draw] (x6) at (3.4,0) {$v_{6}$};
    \draw[->] (x2) edge node[above] {$1$~~~} (x1);
    \draw[->] (x2) edge node[above] {~~~$2$} (x3);
    \draw[->] (x3) edge node[above] {$3$~~~} (x4);
    \draw[->] (x3) edge node[above] {~~~$4$} (x5);
    \draw[->] (x6) edge node[above] {$5$~~~} (x5);
    \end{tikzpicture}
\end{subfigure}
\hfill
\begin{subfigure}[c]{0.40\textwidth}
        \includegraphics[width=0.95\textwidth]{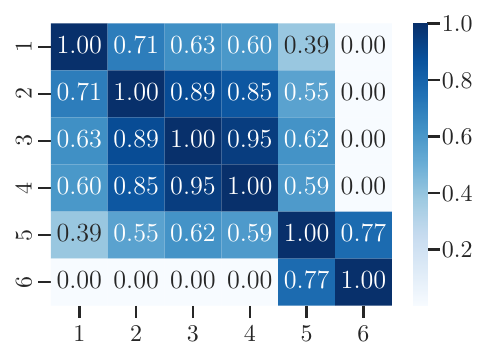}
\end{subfigure}
\hspace{-10pt}
\vspace{0pt}
    \caption{\textbf{Illustrating Theorem \ref{th:three_non_ident} for trees in the same MEC.} 
    Covariance matrix of observed \ours variables for two example forests belonging to the same MEC with the same weights assigned to the edges of the skeleton.}
    \label{fig:tree-id}
\end{figure}

\paragraph{Remark}
\label{rem:3node_counterexample}
In Figure \ref{fig:triangle-id}, we empirically demonstrate that Theorem \ref{th:three_non_ident} no longer holds if we drop the forest assumption. 
For data generated from an \ours and two graphs from the same $\mec$ with the same weights assigned to the skeleton edges, we observe that the estimated covariances differ. The two systems entail different observational distributions.

\begin{figure}[t]
    \centering
    \vspace{15pt}
    \hfill
    \begin{subfigure}[c]{0.20\textwidth}
    \begin{tikzpicture}[node distance=2cm]
        \node[circle,draw] (x1) at (0,0) {$v_{1}$};
        \node[circle,draw] (x2) at (-0.8,-1.5) {$v_{2}$};
        \node[circle,draw] (x3) at (0.8,-1.5) {$v_{3}$};
        \draw[->] (x1) edge node[above] {$1$~~~~} (x2);
        \draw[->] (x2) edge node[below,yshift=-0.1cm] {$2$} (x3);
        \draw[->] (x1) edge node[above] {~~~$3$} (x3);
    
    \end{tikzpicture}
    \end{subfigure}
    \begin{subfigure}[c]{0.25\textwidth}
    \includegraphics[width=0.899\textwidth]{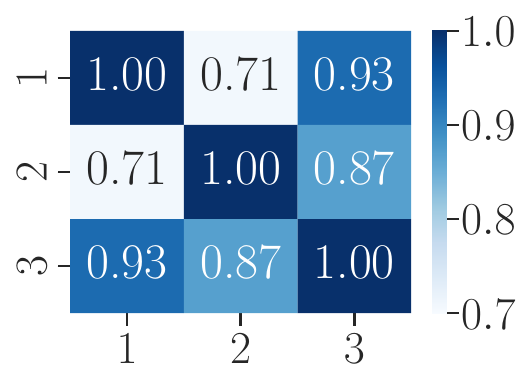}
    \end{subfigure}
    \hfill
    \begin{subfigure}[c]{0.20\textwidth}
    \begin{tikzpicture}[node distance=2cm]
        \node[circle,draw] (x1) at (0,0) {$v_{1}$};
        \node[circle,draw] (x2) at (-0.8,-1.5) {$v_{2}$};
        \node[circle,draw] (x3) at (0.8,-1.5) {$v_{3}$};
        \draw[->] (x1) edge node[above] {$1$~~~~} (x2);
        \draw[->] (x3) edge node[below,yshift=-0.1cm] {$2$} (x2);
        \draw[->] (x1) edge node[above] {~~~~$3$} (x3);
    \end{tikzpicture}
    \end{subfigure}
    \begin{subfigure}[c]{0.25\textwidth}
    \includegraphics[width=0.899\textwidth]{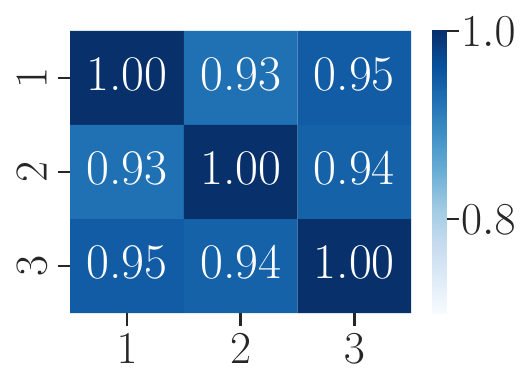}
    \end{subfigure}
    \hfill
    \caption{\textbf{Non-forest counterexample for Theorem \ref{th:three_non_ident}.} Covariance matrix of observed \ours variables for two non-forests belonging to the same MEC with the same weights assigned to the edges of the skeleton.}
    \label{fig:triangle-id}
\end{figure}

\section{Background on Related Work}

\subsection{Heuristics for Mitigating Variance Accumulation}
\label{sec:background-heuristics}

Here, we review existing heuristics for avoiding the exploding variance in structure learning benchmarking with linear SCMs as defined in Equation \eqref{eq:linear_scm}.
We also describe how these heuristics limit the causal dependencies that can be modeled in terms of the correlations among the SCM variables or their cause-explained variance, both of which do not occur in linear \ourss. Finally, in Figure \ref{fig:heuristic-sortabilities} we show that \textbf{the heuristics fail to induce data that is both not $\varop$-sortable and not \rtwo-sortable}.

\paragraph{Scaling weights by the inverse weight norm} \citet[][Section 5.2]{mooij2020joint}
sample the edge weights in linear SCMs as $w_{i,j} \sim \unif_{\pm}{[0.5, 1.5]}$. To achieve a comparable variance of each variable $x_j$ in the SCM, they propose re-scaling the sampled weights prior to the data-generating process as
\begin{align*}
    w_{i,j} \gets \frac{w_{i,j}}{\sqrt{1 + \sum_{i \in \paset{j}} w_{i,j}^2}} \, .
\end{align*}
If all parents of $x_j$ are \iid Gaussian with variance \num{1}, this adjustment ensures that the variance of $x_j$ is similar for all $x_j$.
However, this approximation does not take into account the covariances of the parents.
Moreover, since $\smash{\var{\noise{j}}}$ is unchanged, the scaling limits the strength of the causal effect that parents can have on $x_j$.
For example, 
when $x_1 = \noise{1}$ and $x_2 = w x_1 + \noise{2}$ with $\smash{\var{\noise{j}} = 1}$ as for \citet{mooij2020joint}, the adjusted weight is $\smash{w' = w/\sqrt{1 + w^2}} < 1$.
Thus, for any $w \neq 0$, we have
\begin{align*}
    \lvert \mathrm{Corr}[x_1, x_2] \rvert 
    = \frac{\lvert\cov{\noise{1}, \smash{w'}\noise{1} + \noise{2}}\rvert}{\sqrt{\var{\noise{1}}\var{\smash{w'}\noise{1} + \noise{2}}}}
    = \frac{\lvert\smash{w'}\rvert}{\sqrt{\smash{w'}^2 + 1}}
    < \frac{1}{\sqrt{2}} 
    \approx 0.707 \, .
\end{align*}
This is the maximum correlation between neighbouring variables that any SCM can model under the proposed re-scaling when $\smash{\var{\noise{j}} = 1}$, since additional parents decrease the parent-child correlations.
By contrast, \ourss can model any level of correlation by sampling arbitrary values of $w_{i,j}$, while guaranteeing unit-variance observations $x_j$.
Intuitively, \ourss achieve this by standardizing $x_j$ after the exogenous noise $\noise{j}$ is added to the endogenous contributions of the parents $\pa{j}$, while weight scaling is done before $\noise{j}$ is added to $x_j$.

\paragraph{Scaling weights by the incoming variance}
\citet[][Section 5.1]{squires2022causal} sample the weights of linear SCMs as $w_{i,j} \sim \unif_{\pm}{[0.25, 1.0]}$. 
Given the initial edge weights, they propose adjusting the weights during the generative process by first estimating the total variance $\smash{\hat{\sigma}_j^2}$ that the parents of $x_j$ contribute to $x_j$ from samples drawn under an initial level of additive noise with $\var{\noise{j}} = 1$ and then re-scaling the weights as
\begin{align*}
    w_{i,j} \gets \frac{w_{i,j}}{\sqrt{2\hat{\sigma}_j^2}} \, .
\end{align*}
When using additive noise with $\var{\noise{j}} = 0.5$ to generate the actual samples,
this scaling results in $\var{x_j} = 1$ with a constant fraction of cause-explained variance $\cevf{x_i} = 0.5$.
In benchmarks, however, we may be interested in evaluating SCMs with arbitrary levels of cause-explained variance. \ourss allow this by construction. 
Contrary to \citet{squires2022causal}, \ourss scale the variables $x_j$ rather than the weights $w_{i,j}$ while leaving the exogenous noise $\noise{j}$ unchanged,
which enables modeling arbitrarily small or large levels of unexplained variation.

\begin{figure}
    \centering
    \includegraphics[width=0.5\linewidth]{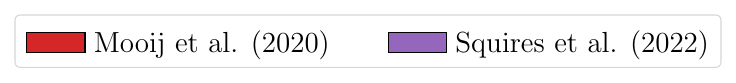}
        \vspace{-5pt}
    
    \begin{subfigure}[b]{0.45\textwidth}
     \centering
    \includegraphics[width=0.49\linewidth]{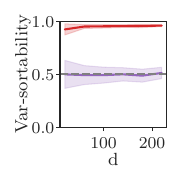}
    \includegraphics[width=0.49\linewidth]{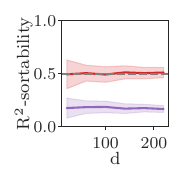}
    \caption{$\er{d, 2}$}
    \end{subfigure}
    \hspace{0.05\textwidth}
    \begin{subfigure}[b]{0.45\textwidth}
     \centering
     \includegraphics[width=0.49\linewidth]{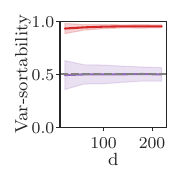}
    \includegraphics[width=0.49\linewidth]{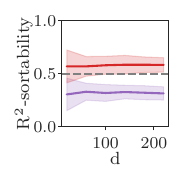}
    \caption{$\usf{d, 2}$}
    \end{subfigure}
    \caption{\textbf{Sortabilities for data generated according to heuristics that aim to remove artifacts}. Weight ranges were assumed as in the original papers: $w_{i,j} \sim \pm \unif_{[0.5, 1.5]}$ for \citet{mooij2020joint}, $w_{i,j} \sim \pm \unif_{[0.25, 1.0]}$
    for \citet{squires2022causal}. For every model, we evaluate \num{100} systems each using $\numsamples =$\num{1000} samples.
    Lines and shaded regions denote mean and standard deviation respectively.}
    \label{fig:heuristic-sortabilities}
    
\end{figure}

\subsection{Sortability Metrics}
\label{sec:sortabilities-def}
In this section, we describe the definition of a sortability metric as introduced by \cite{reisach2024scale}, which we use in Section \ref{sec:experiments}. 
For a function $\tau$, $\tau$-sortability assigns a scalar in $[0, 1]$ to the variables $\mathbf{x}$ and graph $\mathcal{G}$ (with weight matrix $W_\mathcal{G}$) as
\begin{equation*}
    \frac{\sum_{i=1}^{d} \sum_{\dgpath{s}{t} \in  W^i_{\mathcal{G}}} \text{incr}(\tau(\mathbf{x}, s), \tau(\mathbf{x}, t)) }{ \sum_{i=1}^{d} \sum_{\dgpath{s}{t} \in W^i_{\mathcal{G}}} 1} \quad \text{ where } \text{incr}(a, b) =
\begin{cases}
    1 & \text{if } a < b \\
    \frac{1}{2} & \text{if } a = b \\
    0 & \text{if } a > b
\end{cases}
\end{equation*}
and $ W^i_{\mathcal{G}}$ is the $i$-th power of the adjacency matrix $W_{\mathcal{G}}$ and $\dgpath{s}{t} \in W^i_{\mathcal{G}}$
if and only if at least one directed path from $v_s$ to $v_t$ of length $i$ exists in $\mathcal{G}$.
If $\tau(\mathbf{x}, t) = \var{x_t}$, we obtain $\varop$-sortability from \cite{reisach2021beware}. If 
$$\tau(\mathbf{x}, t) = R^2[x_t] = 1 - \frac{\var{x_t - \Expempty[x_t | \mathbf{x}_{\{1,...,d\}\backslash\{t\}}]}}{\var{x_t}} \, ,$$ 
we obtain \rtwo-sortability. 
Estimating $R^2[x_t]$ requires performing regression of $x_t$ onto $\mathbf{x}_{\{1,...,d\}\backslash\{t\}}$. 

\subsection{Structure Learning Algorithms}
\label{sec:csla-algos}

To complement the interpretation of the results in Section \ref{sec:experiments}, we provide some background on the structure learning methods we evaluate. 

\paragraph{\notears \citep{zheng2018dags}}
\notears uses continuous optimization to minimize the regularized mean-squared error (MSE) between the the variables modeled by a linear SCM and the observations, while enforcing a differentiable acyclicity constraint. 
The objective function of \notears is given by $F(\mathbf{W}) = ||\mathbf{X} - \mathbf{X}\mathbf{W}||^2_F / 2\numsamples + \lambda ||\mathbf{W}||_1 $, where $||\cdot||_F$ and $||\cdot||_1$ are a Frobenius and $\ell_1$ norm respectively. When the objective is minimized, weights below a fixed threshold are set to zero.

\paragraph{\avici \citep{lorch2022amortized}}
\avici is an amortized variational inference method that approximates the posterior distribution over causal structures given a dataset through a pretrained inference model.
The variational approximation of \avici uses a fully-factored product of Bernoulli distributions for every possible graph edge. The inference model is a neural network that predict the variational parameters of the Bernoulli distributions by minimizing the expected forward KL divergence between the true posterior and the approximation.
To train the inference model, \avici can be optimized on any training distribution of (synthetic) dataset-graph pairs.
\citet{lorch2022amortized} publish the pretrained parameters of inference models trained on standardized SCMs with linear and nonlinear mechanisms, which we evaluate in this work.

\paragraph{SortnRegress methods (\citealp{reisach2021beware, reisach2024scale})}
The \sortregress methods order the vertices by a chosen statistic and sparsely regress every node on all of its predecessors in the obtained order. 
They use Lasso regression with the Bayesian Information Criterion to learn the regression function for a given variable. \varsortregress uses estimated marginal variances as the sorting criterion. \rtwosortregress uses \rtwo coefficient of determination estimated after performing a regression of every variable onto all remaining variables. 
\randsortregress orders the vertices randomly.

\section{Experimental Setup}
\label{sec:csl-setup}

\subsection{Data}
\label{ssec:experiment-data-gen}

\paragraph{Causal mechanisms}
We consider systems with additive noise, where
\begin{equation*}
    f_i(\mathbf{x}, \noise{i}) = h_i(\mathbf{x}) + \noise{i},
\end{equation*}
for a chosen function $h_i$.
The \linear systems used in this experiments have causal mechanisms as defined in \cref{eq:linear_scm}.
To model nonlinear systems, we use smooth nonlinear functional mechanisms as used by \cite{lorch2022amortized}. 
Specifically, the function $h_i$ that models the relationship between $x_i$ and its parents is sampled from a Gaussian Process
\begin{equation*}
    h_i \sim \mathcal{GP}(0, k_i) \, ,
\end{equation*}
where $k$ is a squared exponential kernel $k_i(\mathbf{x}, \mathbf{x}') = c_i^2\exp\left(-||\mathbf{x} - \mathbf{x}'||^2_2 / 2l_i^2\right)$ with output and length scales $c_i$ and $l_i$ respectively. 
We can approximately express the function sample $h_i$ analytically using random Fourier features \citep{rahimi2007random} by sampling
\begin{equation*}
    h_i(\mathbf{x}) = c_i\sqrt{\tfrac{2}{M}}\sum_{j=1}^M\alpha^{(i)}\cos\left(\tfrac{\omega^{(i)}\cdot \mathbf{x}}{l_i} + \delta^{(i)}\right)
\end{equation*}
where $\alpha^{(i)} \sim \mathcal{N}(0, 1)$, $\omega^{(i)} \sim \mathcal{N}(0, \mathbf{I})$, and $\delta^{(i)} \sim \unif{[0, 2\pi]}$. 
In this work, we use $M=100$.

\paragraph{Generating a random model}
Following prior work (Section \ref{sec:background}), we sample random systems in any simulation performed in this work by first drawing a graph $\mathcal{G}$ from the specified random graph distribution. 
Given the graph $\mathcal{G}$, we sample function parameters of the structural mechanisms over $\mathcal{G}$. 
For linear systems, we sample $w_{i, j} \sim \unif_{\pm}{}[a, b]$, where $a, b$ are fixed, \iid for every graph edge. 
Similarly, for nonlinear systems, for every graph vertex, we draw the length scales $l_{i} \sim \unif[a_1, b_1]$ and output scales $c_{i} \sim \unif[a_2, b_2]$ with predefined $a_1, b_1, a_2, b_2$.

\paragraph{Sampling data from a model}
Given a graph $\mathcal{G}$, noise distribution $\distnoisejoint$, and a set of functions $\{f_1, ... f_\numvars\}$,
we sample $\numsamples$ datapoints from an \ref{eq:scm}
by traversing $\mathcal{G}$ in a topological ordering.
For every vertex $v_i$, we draw a noise sample $\noise{i} \sim \distnoise{i}^\numsamples$. 
The sample for $x_i$ is then deterministically computed by $f_i$ from the exogenous $\noise{i}$ and the parents of $x_i$.
To sample from a \ref{eq:standardized_scm}, we draw a dataset from an SCM and standardize it. 
To sample from an \ref{eq:iscm}, we use Algorithm \ref{alg:iscm}.

\subsection{Experiment Configurations}
\label{ssec:experiment-configuration}

\paragraph{Sortability} 
For Figures \ref{fig:r2_sortability_er2}, \ref{fig:r2_sortability_er4}, and \ref{fig:r2_sortability_vs_degree_er} we generate \erdosrenyi graphs \rebuttal{$\er{d, k}$ \citep{erdos1959}, with $d$ denoting the graph size and $k$ the expected node degree.}
For Figures \ref{fig:r2_sortability_sf2}, \ref{fig:r2_sortability_sf4}, and \ref{fig:r2_sortability_vs_degree_sf} we generate undirected scale-free graphs \rebuttal{$\usf{d, k}$ \citep{barabasi1999emergence}, where $d$ is the graph size and $k$ the number of outgoing edges generated for each vertex.}
Then, we orient the edges in the graph according to a random topological ordering. 
We do not sample directed scale-free graphs initially to avoid high sortability by in-degree, which may confound the results.

For all four figures, we generate \linear systems with weights sampled from three possible distributions $w_{i, j} \sim \unif_{\pm}{}[0.3, 1.8]$, $w_{i, j} \sim \unif_{\pm}{}[0.5, 2.0]$ or $w_{i, j} \sim \unif_{\pm}{}[1.3, 3.0]$ and noise sampled from $\noise{i} \sim \mathcal{N}(0, 1)$. 
For every model configuration, we sample \num{100} systems and $\numsamples =$\num{1000} data points each. 
To create Figures \ref{fig:r2-sortability} and \ref{fig:r2-sortability-appendix} we sampled graphs of sizes $\{20, 60, 100, 140, 180, 220 \}$. \rebuttal{To obtain Figure \ref{fig:r2-sortability-vs-degree}, we sampled graphs with $k \in \{4, 8, 12, 16, 20 \}$.}

\paragraph{Structure Learning (Section \ref{ssec:experimentas-structure-learning})}

For Figures \ref{fig:results-benchmark-main} and  \ref{fig:results-benchmark-appendix1}, we sample \linear systems with weights $w_{i, j} \sim \unif_{\pm}{}[0.5, 2.0]$. 
Following \citet{lorch2022amortized}, \nonlinear mechanisms have length scales $l_{i} \sim \unif[7.0, 10.0]$ and output scales $c_{i} \sim \unif[10.0, 20.0]$.
Both mechanisms are defined in Appendix \ref{ssec:experiment-data-gen}.
For Figures \ref{fig:results-benchmark-appendix2} and \ref{fig:results-benchmark-appendix3}, we generate \linear systems with weights $w_{i, j} \sim \unif_{\pm}{}[0.3, 0.8]$ and $w_{i, j} \sim \unif_{\pm}{}[1.3, 3.0]$.
For all four figures, we sample random $\er{20, 2}$ and $\er{100, 2}$ graphs with noise $\noise{i} \sim \mathcal{N}(0,1)$. For every model configuration, we sample \num{20} systems and $\numsamples = 1000$ data points each.

\paragraph{Noise Transfer}
For Figure \ref{fig:results-induced-noise-main} (top), we sample SCMs, standardized SCMs, and iSCMs with exactly the same underlying graph and weights sampled from $w_{i, j} \sim \unif_{\pm}{}[0.5, 2.0]$. The noise variables are drawn from $\noise{i} \sim \mathcal{N}(0, 1)$. 
Then, for every triple of SCM, standardized SCM, and \ours that shares a graph and weights, 
we create two more SCMs with the same marginal variances as the SCM, but with the noise variances of the implied models of the standardized SCM and \ours, respectively. 
Appendix \ref{ssec:adjusting-variances} provides a motivation and detailed explanation of this procedure. 
Figure \ref{fig:results-induced-noise-main} (top) shows the performance of \notears on the original SCMs and the two SCMs with transferred noise.

For Figure \ref{fig:results-induced-noise-main} (bottom), we sample multiple instances of standardized SCMs, and iSCMs with weights drawn from $w_{i, j} \sim \unif_{\pm}{}[0.5, 2.0]$ and noise from $\noise{i} \sim \mathcal{N}(0, 1)$. 
For every model instance, we approximate the density of the inverse of their implied noise variances using kernel density estimation. 
The figure shows the mean and standard deviation of the p.d.f. values over \num{100} systems.
For both figures, we use $\er{100, 2}$ graphs.

\subsection{Methods}
\label{sec:algo-setup}

\paragraph{\notears \citep{zheng2018dags}}
To run \notears, we use the original implementation provided by the authors of \cite{zheng2018dags} (Apache-2.0 license). Before benchmarking \notears, we run a hyperparameter search to calibrate the weight penalty ($\lambda$) and threshold on held-out instances of each data generation method. 
The hyperparameters can be found in Appendix \ref{sec:csl-hyperparams}.

\paragraph{\avici \citep{lorch2022amortized}}
To evaluate \avici, we use the code and model checkpoints provided by the authors of the method (MIT license). 
Specifically, we use the model trained on linear data to benchmark the method on \linear systems and the model trained on nonlinear data to benchmark on \nonlinear systems. 
We score an edge as predicted if the probability prediction by \avici is greater than \num{0.5}.
Since the parameters are pretrained, the method has otherwise no tuneable hyperparameters. 

\paragraph{Sortabilities and \sortregress methods (\citealp{reisach2021beware, reisach2024scale})}
To compute the sortability metrics and run the \sortregress baselines, we use the \texttt{CausalDisco} library (BSD-3-Clause license) created by the authors of the method. 
The algorithms require no tuneable hyperparameters. 

\paragraph{\golem \citep{ng2020role}} 
For \golemev, we tune $\lambda_1$ (sparsity penalty coefficient), $\lambda_2$ (acyclicity penalty coefficient) and the threshold for zeroing weights. 
For \golemnv, we tune the same hyperparameters as for \golemev. We do not initialize the model with the solution returned by \golemev, as done in the original paper, since we want to evaluate a method that does not assume equal noise variances at any point. Not initializing with the \golemev weights is consistent with the benchmarking approach of \citet{reisach2021beware}. We use the implementation of the original work \citep{ng2020role}. 

\paragraph{\pc Algorithm}
For linear data, we use a Gaussian conditional independence test. For nonlinear data, we use the Hilbert-Schmidt Independence Criterion (HSIC) gamma test. We treat the test significance level as a hyperparameter that we tune. We use the implementation by the Causal Discovery Toolbox \citep{kalainathan2020causal}.

\paragraph{\ges}
GES uses the linear Gaussian BIC score function and does not require hyperparameter tuning. We use the implementation by the Causal Discovery Toolbox \citep{kalainathan2020causal}.

\paragraph{\rebuttal{\cam}}\rebuttal{ \cam estimates a causal ordering using maximum likelihood and then performs sparse nonlinear regression using splines on the possible parents in this ordering.
We use the implementation from the \texttt{dodiscover} library (MIT license) and include the preliminary neighbor search option to make the algorithm scale to large graphs. We tune the cutoff value $\alpha$ for variable selection with hypothesis testing over regression coefficients, and the number and order of splines
to use for the feature function.
}

\paragraph{\rebuttal{\lingam}}\rebuttal{ 
\lingam uses independent component analysis, an algorithm for source separation, to find a causal ordering, which is identifiable in linear systems if the additive noise in an SCM is non-Gaussian. 
We use the implementation from the \texttt{cdt} (Causal Discovery Toolbox) library (MIT license).
}

\subsection{Hyperparameter Selection}
\label{sec:csl-hyperparams}

For all algorithms that require hyperparameter tuning, we perform the search on separate, held-out systems that follow the same configurations as the ones we present in our final experimental results.
We run the algorithms $20$ times per configuration and choose the median $\operatorname{F1}$ score as the criterion for selecting the best hyperparameters.

To run \notears, 
we need to specify the regularisation strength $\lambda$ and a weight threshold $\eta$ for thresholding the final weights for graph structure prediction. 
To select these hyperparameters, we run a parameter search with $\lambda \in \{0.0, 0.05, 0.1, 0.15, 0.2, 0.25, 0.3\}$ and three possible values of the weight threshold $\{0.1, 0.2, 0.3\}$. Table \ref{tab:notears-params-all} presents all final hyperparameter configurations for \notears.
For some hyperparameter configurations, \num{1} in \num{20} runs experienced numerical issues caused by the acyclicity constraint. 
However, this never occurs for the selected, optimal hyperparameters, neither when performing the hyperparameter search nor when running the reported experiments.

To run the \pc algorithm, 
one needs to choose a test significance level $\alpha$. During the hyperparamter search we consider $\alpha \in \{0.01, 0.001, 0.0001\}$. Table \ref{tab:pc-params-all} presents all final hyperparameter configurations for the \pc algorithm.

To run \golemev and \golemnv we need to tune sparsity penalty coefficient $\lambda_1$, acyclicity penalty coefficient $\lambda_2$ and the weight threshold $\eta$. We consider $\lambda_1 \in \{0.02, 0.002, 0.0002\}$, $\lambda_2 \in \{2.0, 5.0, 8.0\}$ and $\eta \in \{0.1, 0.2, 0.3\}$. \Cref{tab:golemev-params-all,tab:golemnv-params-all} present the best configurations.

\rebuttal{To run \cam we need to tune the cutoff value $\alpha \in \{0.05, 0.10, 0.15\}$ for
variable selection with hypothesis testing over regression coefficients and the number and order of splines
to use for the feature function for which we consider sets $\{5, 10\}$ and $\{2, 3\}$ respectively. \Cref{tab:cam-params-all} presents the best configurations.}

\begin{table}
    \caption{\textbf{\notears hyperparameters for all experiments.} Final settings for the regularization strength $\lambda$ and the weight threshold $\eta$ after hyperparameter tuning on the respective models and data-generating processes together with the F1 (median) validation scores achieved by \notears.}
    \vspace{10pt}
    \label{tab:notears-params-all}

\begin{subtable}[t]{\textwidth}
    \centering
    \caption{$\er{20, 2}$ DAGs, \linear mechanisms}
    \begin{tabular}{llcccc}
    \toprule
     Weight Distribution & Model &  $\lambda$ &  $\eta$ & F1 (median) \\
    \midrule
    $\unif_{\pm}{[0.3, 0.8]}$ &                    SCM &    0.05 &              0.20 &       0.97\\
    $\unif_{\pm}{[0.3, 0.8]}$ &       Standardized SCM &    0.15 &              0.10 &       0.59\\
    $\unif_{\pm}{[0.3, 0.8]}$ &                   iSCM &    0.15 &              0.10 &       0.57\\
    \midrule
    $\unif_{\pm}{[0.5, 2.0]}$ &                    SCM &    0.00 &              0.30 &       0.98\\
    $\unif_{\pm}{[0.5, 2.0]}$ &       Standardized SCM &    0.15 &              0.20 &       0.30 \\
    $\unif_{\pm}{[0.5, 2.0]}$ &                   iSCM &    0.15 &              0.10 &       0.50\\
    \midrule
    $\unif_{\pm}{[1.3, 3.0]}$ &                    SCM &    0.05 &              0.30 &       0.98\\
    $\unif_{\pm}{[1.3, 3.0]}$ &       Standardized SCM &    0.25 &              0.10 &       0.24\\
    $\unif_{\pm}{[1.3, 3.0]}$ &                   iSCM &    0.20 &              0.10 &       0.40\\
    \bottomrule
    \end{tabular}
    \vspace{10pt}
    \label{tab:notears-linear-small-graphs}
\end{subtable}

\begin{subtable}[t]{\textwidth}
    \centering
    \caption{$\er{100, 2}$ DAGs, \linear mechanisms}
    \begin{tabular}{llcccc}
    \toprule
     Weight Distribution & Model &  $\lambda$ &  $\eta$ & F1 (median) \\
    \midrule
    $\unif_{\pm}{[0.3, 0.8]}$ &                    SCM &    0.10 &              0.10 &       0.99\\
    $\unif_{\pm}{[0.3, 0.8]}$ &       Standardized SCM &    0.10 &              0.10 &       0.83\\
    $\unif_{\pm}{[0.3, 0.8]}$ &                   iSCM &    0.10 &              0.10 &       0.84\\
    \midrule
    $\unif_{\pm}{[0.5, 2.0]}$ &                    SCM &    0.05 &              0.30 &       0.94\\
    $\unif_{\pm}{[0.5, 2.0]}$ &       Standardized SCM &    0.15 &              0.10 &       0.47\\
    $\unif_{\pm}{[0.5, 2.0]}$ &                   iSCM &    0.15 &              0.10 &       0.76\\
    \midrule
    $\unif_{\pm}{[1.3, 3.0]}$ &                    SCM &    0.10 &              0.30 &       0.82\\
    $\unif_{\pm}{[1.3, 3.0]}$ &       Standardized SCM &    0.20 &              0.10 &       0.30\\
    $\unif_{\pm}{[1.3, 3.0]}$ &                   iSCM &    0.15 &              0.10 &       0.70\\
    \bottomrule
    \end{tabular}
    \vspace{10pt}
    \label{tab:notears-linear-big-graphs}
\end{subtable}

\begin{subtable}[t]{\textwidth}
    \centering
    \caption{$\er{20, 2}$ DAGs, \nonlinear mechanisms}
    \begin{tabular}{lccc}
    \toprule
     Model &  $\lambda$ &  $\eta$ & F1 (median) \\
    \midrule
                       SCM &    0.15 &              0.30 &       0.58 \\
          Standardized SCM &    0.15 &              0.10 &       0.33 \\
                      iSCM &    0.15 &              0.20 &       0.42 \\
    \bottomrule
    \end{tabular}
    \vspace{10pt}
    \label{tab:notears-nonlinear-small-graphs}
\end{subtable}

\begin{subtable}[t]{\textwidth}
    \centering
    \caption{$\er{ 100, 2}$ DAGs, \nonlinear mechanisms}
    \begin{tabular}{lccc}
    \toprule
     Model &  $\lambda$ &  $\eta$ & F1 (median) \\
    \midrule
                       SCM &    0.30 &              0.30 &       0.50 \\
         Standardized SCM &    0.15 &              0.10 &       0.43 \\
                      iSCM &    0.15 &              0.10 &       0.61 \\
    \bottomrule
    \end{tabular}
    \vspace{10pt}
    \label{tab:notears-nonlinear-big-graphs}
\end{subtable}

\begin{subtable}[t]{\textwidth}
    \centering
    \caption{Noise transfer experiment: $\er{100, 2}$ DAGs, \linear mechanisms $w_{ij} \sim \unif_{\pm}{[0.5, 2.0]}$ 
    }
    \begin{tabular}{llccc}
    \toprule
     Model &  $\lambda$ &  $\eta$ & F1 (median) \\
    \midrule
                        Original &    0.05 &              0.30 &       0.96 \\
           Noise from standardized SCM &    0.10 &              0.30 &       0.72\\
                       Noise from iSCM &    0.05 &              0.30 &       0.82\\
    
    \bottomrule
    \end{tabular}
    \label{tab:notears-linear-noise-transfer}
\end{subtable}
\end{table}

\begin{table}
    \caption{\textbf{\golemev hyperparameters for all experiments.} Final settings for the sparsity penalty coefficient $\lambda_1$, acyclicity penalty coefficient $\lambda_2$ and the weight threshold $\eta$ after hyperparameter tuning on the respective models and data-generating processes together with the F1 (median) validation scores achieved by \golemev.}
    \vspace{10pt}
    \label{tab:golemev-params-all}

\begin{subtable}[t]{\textwidth}
    \centering
    \caption{$\er{20, 2}$ DAGs, \linear mechanisms}
\begin{tabular}{llrrrrr}
\toprule
 Weight Distribution & Model &  $\lambda_1$ &  $\lambda_2$ &  $\eta$ &  F1 (median) \\
\midrule
$\unif_{\pm}{[0.5, 2.0]}$ &                    SCM &     0.002 &     5.00 &              0.30 &       1.00\\
$\unif_{\pm}{[0.5, 2.0]}$ &       Standardized SCM &     0.020 &     8.00 &              0.10 &       0.15\\
$\unif_{\pm}{[0.5, 2.0]}$ &                   iSCM &     0.020 &     2.00 &              0.10 &       0.36\\
\midrule
$\unif_{\pm}{[1.3, 3.0]}$ &                    SCM &     0.002 &     5.00 &              0.30 &       1.00\\
$\unif_{\pm}{[1.3, 3.0]}$ &       Standardized SCM &     0.020 &     8.00 &              0.10 &       0.12\\
$\unif_{\pm}{[1.3, 3.0]}$ &                   iSCM &     0.020 &     5.00 &              0.10 &       0.34\\
\midrule
$\unif_{\pm}{[0.3, 0.8]}$ &                    SCM &     0.020 &     2.00 &              0.10 &       1.00\\
$\unif_{\pm}{[0.3, 0.8]}$ &       Standardized SCM &     0.020 &     5.00 &              0.10 &       0.33\\
$\unif_{\pm}{[0.3, 0.8]}$ &                   iSCM &     0.020 &     5.00 &              0.10 &       0.36\\
\bottomrule
\end{tabular}

    \vspace{10pt}
    \label{tab:golemev-linear-small-graphs}
\end{subtable}

\begin{subtable}[t]{\textwidth}
    \centering
    \caption{$\er{100, 2}$ DAGs, \linear mechanisms}
    \begin{tabular}{llrrrrr}
\toprule
 Weight Distribution & Model &  $\lambda_1$ & $\lambda_2$ &  $\eta$ &  F1 (median)\\
\midrule
$\unif_{\pm}{[0.5, 2.0]}$ &                    SCM &    0.020 &    2.00 &             0.20 &      1.00\\
$\unif_{\pm}{[0.5, 2.0]}$ &       Standardized SCM &    0.020 &    8.00 &             0.10 &      0.13\\
$\unif_{\pm}{[0.5, 2.0]}$ &                   iSCM &    0.020 &    5.00 &             0.10 &      0.24\\
\midrule
$\unif_{\pm}{[1.3, 3.0]}$ &                    SCM &    0.020 &    8.00 &             0.30 &      0.90\\
$\unif_{\pm}{[1.3, 3.0]}$ &       Standardized SCM &    0.020 &    5.00 &             0.10 &      0.08\\
$\unif_{\pm}{[1.3, 3.0]}$ &                   iSCM &    0.020 &    5.00 &             0.10 &      0.19\\
\midrule
$\unif_{\pm}{[0.3, 0.8]}$ &                    SCM &    0.020 &    2.00 &             0.20 &      1.00\\
$\unif_{\pm}{[0.3, 0.8]}$ &       Standardized SCM &    0.020 &    5.00 &             0.10 &      0.30\\
$\unif_{\pm}{[0.3, 0.8]}$ &                   iSCM &    0.020 &    2.00 &             0.10 &      0.40\\
\bottomrule
\end{tabular}

    \vspace{10pt}
    \label{tab:golemev-linear-big-graphs}
\end{subtable}

\begin{subtable}[t]{\textwidth}
    \centering
    \caption{$\er{20, 2}$ DAGs, \nonlinear mechanisms}
\begin{tabular}{lrrrrr}
\toprule
Model &  $\lambda_1$ & $\lambda_2$ &  $\eta$ &  F1 (median)\\
\midrule
                   SCM &    0.020 &    8.00 &             0.30 &      0.39\\
      Standardized SCM &    0.002 &    8.00 &             0.10 &      0.20\\
                  iSCM &    0.020 &    2.00 &             0.10 &      0.25\\
\bottomrule
\end{tabular}

    \vspace{10pt}
    \label{tab:golemev-nonlinear-small-graphs}
\end{subtable}

\begin{subtable}[t]{\textwidth}
    \centering
    \caption{$\er{ 100, 2}$ DAGs, \nonlinear mechanisms}
\begin{tabular}{lrrrrr}
\toprule
Model &  $\lambda_1$ & $\lambda_2$ &  $\eta$ &  F1 (median)\\
\midrule
                   SCM &    0.020 &    8.00 &             0.10 &      0.27\\
      Standardized SCM &    0.020 &    8.00 &             0.10 &      0.14\\
                  iSCM &    0.020 &    5.00 &             0.10 &      0.14\\
\bottomrule
\end{tabular}

    \vspace{10pt}
    \label{tab:golemev-nonlinear-big-graphs}
\end{subtable}

\end{table}

\begin{table}
    \caption{\textbf{\golemnv hyperparameters for all experiments.} Final settings for the sparsity penalty coefficient $\lambda_1$, acyclicity penalty coefficient $\lambda_2$ and the weight threshold $\eta$ after hyperparameter tuning on the respective models and data-generating processes together with the F1 (median) validation scores achieved by \golemnv.}
    \vspace{10pt}
    \label{tab:golemnv-params-all}

\begin{subtable}[t]{\textwidth}
    \centering
    \caption{$\er{20, 2}$ DAGs, \linear mechanisms}
\begin{tabular}{llrrrrr}
\toprule
 Weight Distribution & Model &  $\lambda_1$ & $\lambda_2$ &  $\eta$ &  F1 (median)\\
\midrule
$\unif_{\pm}{[0.5, 2.0]}$ &                    SCM &     0.0002 &     2.00 &              0.20 &       0.16\\
$\unif_{\pm}{[0.5, 2.0]}$ &       Standardized SCM &     0.0200 &     2.00 &              0.10 &       0.38\\
$\unif_{\pm}{[0.5, 2.0]}$ &                   iSCM &     0.0200 &     2.00 &              0.20 &       0.45\\
\midrule

$\unif_{\pm}{[1.3, 3.0]}$ &                    SCM &     0.0002 &     2.00 &              0.10 &       0.20\\
$\unif_{\pm}{[1.3, 3.0]}$ &       Standardized SCM &     0.0002 &     5.00 &              0.10 &       0.37\\
$\unif_{\pm}{[1.3, 3.0]}$ &                   iSCM &     0.0200 &     2.00 &              0.20 &       0.37\\
\midrule

$\unif_{\pm}{[0.3, 0.8]}$ &                    SCM &     0.0020 &     5.00 &              0.20 &       0.13\\
$\unif_{\pm}{[0.3, 0.8]}$ &       Standardized SCM &     0.0200 &     2.00 &              0.20 &       0.55\\
$\unif_{\pm}{[0.3, 0.8]}$ &                   iSCM &     0.0200 &     2.00 &              0.10 &       0.58\\
\bottomrule
\end{tabular}

    \vspace{10pt}
    \label{tab:golemnv-linear-small-graphs}
\end{subtable}

\begin{subtable}[t]{\textwidth}
    \centering
    \caption{$\er{100, 2}$ DAGs, \linear mechanisms}
\begin{tabular}{llrrrrr}
\toprule
 Weight Distribution & Model &  $\lambda_1$ & $\lambda_2$ &  $\eta$ &  F1 (median)\\
\midrule
$\unif_{\pm}{[0.5, 2.0]}$ &                    SCM &    0.002 &    5.00 &             0.20 &      0.10\\
$\unif_{\pm}{[0.5, 2.0]}$ &       Standardized SCM &    0.020 &    2.00 &             0.10 &      0.32\\
$\unif_{\pm}{[0.5, 2.0]}$ &                   iSCM &    0.020 &    2.00 &             0.10 &      0.51\\
\midrule

$\unif_{\pm}{[1.3, 3.0]}$ &                    SCM &    0.002 &    2.00 &             0.10 &      0.21\\
$\unif_{\pm}{[1.3, 3.0]}$ &       Standardized SCM &    0.002 &    5.00 &             0.10 &      0.18\\
$\unif_{\pm}{[1.3, 3.0]}$ &                   iSCM &    0.020 &    2.00 &             0.10 &      0.46\\
\midrule
$\unif_{\pm}{[0.3, 0.8]}$ &                    SCM &    0.020 &    2.00 &             0.10 &      0.18\\
$\unif_{\pm}{[0.3, 0.8]}$ &       Standardized SCM &    0.020 &    2.00 &             0.20 &      0.65\\
$\unif_{\pm}{[0.3, 0.8]}$ &                   iSCM &    0.020 &    2.00 &             0.10 &      0.67\\
\bottomrule
\end{tabular}

    \vspace{10pt}
    \label{tab:golemnv-linear-big-graphs}
\end{subtable}

\begin{subtable}[t]{\textwidth}
    \centering
    \caption{$\er{20, 2}$ DAGs, \nonlinear mechanisms}
\begin{tabular}{lrrrrr}
\toprule
Model &  $\lambda_1$ & $\lambda_2$ &  $\eta$ &  F1 (median)\\
\midrule
                   SCM &    0.002 &    8.00 &             0.20 &      0.07\\
      Standardized SCM &    0.020 &    2.00 &             0.20 &      0.30\\
                  iSCM &    0.020 &    2.00 &             0.10 &      0.41\\
\bottomrule
\end{tabular}

    \vspace{10pt}
    \label{tab:golemnv-nonlinear-small-graphs}
\end{subtable}

\begin{subtable}[t]{\textwidth}
    \centering
    \caption{$\er{ 100, 2}$ DAGs, \nonlinear mechanisms}
    \begin{tabular}{lrrrrr}
\toprule
Model &  $\lambda_1$ & $\lambda_2$ &  $\eta$ &  F1 (median)\\
\midrule
                   SCM &    0.002 &    5.00 &             0.20 &      0.07\\
      Standardized SCM &    0.020 &    2.00 &             0.10 &      0.24\\
                  iSCM &    0.020 &    2.00 &             0.10 &      0.36\\
\bottomrule
\end{tabular}

    \vspace{10pt}
    \label{tab:golemnv-nonlinear-big-graphs}
\end{subtable}

\end{table}

\begin{table}
    \caption{\textbf{\pc hyperparameters for all experiments.} Final settings for the significance level $\alpha$ after hyperparameter tuning on the respective models and data-generating processes together with the F1 (median) validation scores achieved by \pc.}
    \vspace{10pt}
    \label{tab:pc-params-all}

\begin{subtable}[t]{\textwidth}
    \centering
    \caption{$\er{20, 2}$ DAGs, \linear mechanisms}
    \begin{tabular}{llcccc}
    \toprule
     Weight Distribution & Model &  $\alpha$ & F1 (median) \\
    \midrule
    $\unif_{\pm}{[0.3, 0.8]}$ &                    SCM &                 0.01 &       0.71\\
    $\unif_{\pm}{[0.3, 0.8]}$ &       Standardized SCM &                 0.01 &       0.70\\
    $\unif_{\pm}{[0.3, 0.8]}$ &                   iSCM &                 0.01 &       0.72\\
    \midrule
    $\unif_{\pm}{[0.5, 2.0]}$ &                    SCM &                 0.01 &       0.47\\
    $\unif_{\pm}{[0.5, 2.0]}$ &       Standardized SCM &                 0.01 &       0.46\\
    $\unif_{\pm}{[0.5, 2.0]}$ &                   iSCM &                 0.01 &       0.58\\
    \midrule
    $\unif_{\pm}{[1.3, 3.0]}$ &                    SCM &                 0.01 &       0.35\\
    $\unif_{\pm}{[1.3, 3.0]}$ &       Standardized SCM &                 0.01 &       0.38\\
    $\unif_{\pm}{[1.3, 3.0]}$ &                   iSCM &                 0.01 &       0.48\\
    \bottomrule
    \end{tabular}
    \vspace{10pt}
    \label{tab:pc-linear-small-graphs}
\end{subtable}

\begin{subtable}[t]{\textwidth}
    \centering
    \caption{$\er{100, 2}$ DAGs, \linear mechanisms}
    \begin{tabular}{llcccc}
    \toprule
     Weight Distribution & Model &   $\alpha$ & F1 (median) \\
    \midrule
    $\unif_{\pm}{[0.3, 0.8]}$ &                    SCM &                 0.01 &       0.82\\
    $\unif_{\pm}{[0.3, 0.8]}$ &       Standardized SCM &                 0.01 &       0.85\\
    $\unif_{\pm}{[0.3, 0.8]}$ &                   iSCM &                 0.01 &       0.86\\
    \midrule
    $\unif_{\pm}{[0.5, 2.0]}$ &                    SCM &                 0.01 &       0.62\\
    $\unif_{\pm}{[0.5, 2.0]}$ &       Standardized SCM &                 0.01 &       0.57\\
    $\unif_{\pm}{[0.5, 2.0]}$ &                   iSCM &                 0.01 &       0.79\\
    \midrule
    $\unif_{\pm}{[1.3, 3.0]}$ &                    SCM &                 0.01 &       0.42\\
    $\unif_{\pm}{[1.3, 3.0]}$ &       Standardized SCM &                 0.01 &       0.43\\
    $\unif_{\pm}{[1.3, 3.0]}$ &                   iSCM &                 0.01 &       0.71\\
    \bottomrule
    \end{tabular}
    \vspace{10pt}
    \label{tab:pc-linear-big-graphs}
\end{subtable}

\begin{subtable}[t]{\textwidth}
    \centering
    \caption{$\er{20, 2}$ DAGs, \nonlinear mechanisms}
    \begin{tabular}{lccc}
    \toprule
     Model &   $\alpha$ & F1 (median) \\
    \midrule
                       SCM &                0.01 &       0.53 \\
          Standardized SCM &                0.01 &       0.54 \\
                      iSCM &                0.01 &       0.65 \\
    \bottomrule
    \end{tabular}
    \vspace{10pt}
    \label{tab:pc-nonlinear-small-graphs}
\end{subtable}

\begin{subtable}[t]{\textwidth}
    \centering
    \caption{$\er{100, 2}$ DAGs, \nonlinear mechanisms}
    \begin{tabular}{lccc}
    \toprule
     Model &   $\alpha$ & F1 (median) \\
    \midrule
                       SCM &                 0.01 &       0.53 \\
          Standardized SCM &                 0.01 &       0.63 \\
                      iSCM &                 0.01 &       0.68 \\
    \bottomrule
    \end{tabular}
    \vspace{10pt}
    \label{tab:pc-nonlinear-big-graphs}
\end{subtable}
\end{table}

\begin{table}
    \caption{\textbf{\cam hyperparameters for all experiments.} Final settings for the cutoff value $\alpha$ for variable selection with hypothesis testing over regression coefficients, the number and order of splines to use for the feature function, together with the F1 (median) validation scores achieved by \cam.}
    \vspace{10pt}
    \label{tab:cam-params-all}

\begin{subtable}[t]{\textwidth}
    \centering
    \caption{$\er{20, 2}$ DAGs, \linear mechanisms}
    \begin{tabular}{llcccc}
    \toprule
Weight Distribution & Model & $\alpha$ & Number of Splines & Spline Order & F1 (median) \\
\midrule
$\unif_{\pm}{[0.3, 0.8]}$ & SCM & 0.05 & 5 & 3 & 0.49 \\
$\unif_{\pm}{[0.3, 0.8]}$ & Stand. SCM & 0.05 & 5 & 2 & 0.46 \\
$\unif_{\pm}{[0.3, 0.8]}$ & iSCM & 0.05 & 5 & 3 & 0.57 \\
\midrule
$\unif_{\pm}{[0.5, 2.0]}$ & SCM & 0.10 & 10 & 3 & 0.31 \\
$\unif_{\pm}{[0.5, 2.0]}$ & Stand. SCM & 0.10 & 10 & 2 & 0.23 \\
$\unif_{\pm}{[0.5, 2.0]}$ & iSCM & 0.10 & 5 & 2 & 0.53 \\
\midrule
$\unif_{\pm}{[1.3, 3.0]}$ & SCM & 0.05 & 10 & 2 & 0.24 \\
$\unif_{\pm}{[1.3, 3.0]}$ & Stand. SCM & 0.05 & 10 & 3 & 0.27 \\
$\unif_{\pm}{[1.3, 3.0]}$ & iSCM & 0.05 & 5 & 2 & 0.42 \\

\bottomrule

    \end{tabular}
    \vspace{10pt}
    \label{tab:cam-linear-small-graphs}
\end{subtable}

\begin{subtable}[t]{\textwidth}
    \centering
    \caption{$\er{100, 2}$ DAGs, \linear mechanisms}
    \begin{tabular}{llcccc}
    \toprule
Weight Distribution & Model & $\alpha$ & Number of Splines & Spline Order & F1 (median) \\
\midrule
$\unif_{\pm}{[0.3, 0.8]}$ & SCM & 0.05 & 10 & 3 & 0.54\\
$\unif_{\pm}{[0.3, 0.8]}$ & Stand. SCM & 0.05 & 10 & 2 & 0.57 \\
$\unif_{\pm}{[0.3, 0.8]}$ & iSCM & 0.05 & 5 & 3 & 0.61 \\
\midrule
$\unif_{\pm}{[0.5, 2.0]}$ & SCM & 0.05 & 10 & 2 & 0.39 \\
$\unif_{\pm}{[0.5, 2.0]}$ & Stand. SCM & 0.05 & 5 & 2 & 0.39 \\
$\unif_{\pm}{[0.5, 2.0]}$ & iSCM & 0.05 & 5 & 3 & 0.62 \\
\midrule
$\unif_{\pm}{[1.3, 3.0]}$ & SCM & 0.05 & 5 & 3 & 0.27 \\
$\unif_{\pm}{[1.3, 3.0]}$ & Stand. SCM & 0.05 & 10 & 3 & 0.26 \\
$\unif_{\pm}{[1.3, 3.0]}$ & iSCM & 0.05 & 5 & 3 & 0.61 \\
\bottomrule
    \end{tabular}
    \vspace{10pt}
    \label{tab:cam-linear-big-graphs}
\end{subtable}

\begin{subtable}[t]{\textwidth}
    \centering
    \caption{$\er{20, 2}$ DAGs, \nonlinear mechanisms}
    \begin{tabular}{lcccc}
\toprule
Model & $\alpha$ & Number of Splines & Spline Order & F1 (median) \\
\midrule
SCM & 0.05 & 10 & 2 & 0.50 \\
Standardized SCM & 0.05 & 5 & 2 & 0.52 \\
iSCM & 0.05 & 5 & 2 & 0.57 \\
\bottomrule
\end{tabular}

    \vspace{10pt}
    \label{tab:cam-nonlinear-small-graphs}
\end{subtable}

\begin{subtable}[t]{\textwidth}
    \centering
    \caption{$\er{100, 2}$ DAGs, \nonlinear mechanisms}
    \begin{tabular}{lcccc}
\toprule
Model & $\alpha$ & Number of Splines & Spline Order & F1 (median)  \\
\midrule
SCM & 0.05 & 10 & 3 & 0.50 \\
Standardized SCM & 0.05 & 10 & 2 & 0.51 \\
iSCM & 0.05 & 10 & 3 & 0.57 \\
\bottomrule
\end{tabular}

    \vspace{10pt}
    \label{tab:cam-nonlinear-big-graphs}
\end{subtable}
\end{table}

\subsection{Transferring Noise Variances While Keeping $\varop$-Sortability Unchanged} 
\label{ssec:adjusting-variances}

\cite{reisach2021beware} show that post-hoc standardization of SCM data strongly impairs the performance of \notears.
When comparing the performance of \notears between data sampled from \ourss and standardized SCMs, there are at least two factors that can affect the performance of \notears, low $\varop$-sortability and the violation of the equal noise variance assumption. 
Our experiments in Figure \ref{fig:results-induced-noise-main} of Section \ref{sec:experiments} aim at isolating the effect of the latter.
Specifically, we investigate whether \notears performs better on $\varop$-sortable datasets that have the noise scale patterns implied when assuming SCMs generated the data---when in fact the data was sampled from \ourss or standardized SCMs. 
To achieve this, we ensure that the $\varop$-sortability metrics of the data sampled from the models is the same, here close to \num{1}. 

Given two linear SCMs $S^a$ and $S^b$ with the same underlying graph $\mathcal{G}$, our goal is to construct a system $S^t$ with the same marginal variances as $S^a$ (condition 1) and the same noise variances as $S^b$ (condition 2). For this task to be well-defined, we assume that the noise variances of the root variables in $S^a$ and $S^b$ are the same.
The first step in constructing $S^t$ is to copy the noise variances from $S^b$, so that for every $i \in \numvarset$.
\begin{equation*}
    {\sigma_i^2}^{t} := {\sigma_i^2}^b \, .
\end{equation*}
This satisfies condition 2. Given this, we define $x_i^t$ as
\begin{equation*}
    x_i^t := \sqrt{\frac{\var{x_i^a} - {\sigma_i^2}^{b}}{\var{{\mathbf{w}^a_i}^T\pa{i}^t}}}{\mathbf{w}^a_i}^T\pa{i}^t + \noise{i}^t \, ,
\end{equation*}
where $\noise{i}^t$ has variance ${\sigma_i^2}^{t}$.
By construction, the condition of $S^t$ sharing the noise variances with $S^b$ and the marginal variances with $S^a$ is fulfilled for the root variables.
For all the remaining variables, it holds that
\begin{equation*}
\begin{aligned}
    \var{\tr{x_i}} &= \varop \left [ \sqrt{\frac{\var{x_i^a} - {\sigma_i^2}^{b}}{\var{{\mathbf{w}^a_i}^T\pa{i}^t}}}{\mathbf{w}^a_i}^T\pa{i}^t + \noise{i}^t \right ] \\
    &=\frac{\var{x_i^a} - {\sigma_i^2}^{b}}{\var{{\mathbf{w}^a_i}^T\pa{i}^t}}\var{{\mathbf{w}^a_i}^T\pa{i}^t} + {\sigma_i^2}^b \\
    &=\var{x_i^a} \, ,
\end{aligned}
\end{equation*}
which satisfies condition 1.
Since the systems $S^t$ and $S^a$ have the same marginal variances, they have the same $\varop$-sortability.
In the noise transfer experiment of Figure \ref{fig:results-induced-noise-main}, we transfer the noise variances from the implied models of \ourss and standardized SCMs. To obtain the noise variances in the implied models, we divide the original noise variances (equal to \num{1}) by the estimated marginal variances of the corresponding variable before standardization, which we estimate from $\numsamples=1000$ datapoints. For \ours, this corresponds to an empirical statistics of Equation \eqref{eq:implied_weights_noise_variances}.

\subsection{Compute Resources}\label{ssec:compute-resources}
Our experiments were run on an internal cluster. 
All experiments in this work were computed using CPUs with \num{3}GB of memory per CPU, with an exception of the \avici runs on graphs with \num{100} vertices, which used \num{12}GB per CPU. 
The data generation takes less than a few minutes on a single CPU, with the exception of the sortability results (Section \ref{ssec:experimentas-rtwo}). 
For the sortability results, it takes around \num{30} minutes to generate the datasets for a single graph specification across all weight supports and graph sizes.
This is due to a bigger number of configurations and repetitions than in the other experiments.
For a single graph specification and across all weight supports and graph sizes, it takes around \num{6} hours to compute the sortability statistics on a single CPU.
All benchmarked methods take no longer than a few minutes per small graph ($d=20$) and no longer than half an hour per big graph ($d=100$).
The \sortregress baselines run in less than \num{1}min per graph.

\section{Additional Experimental Results}
\label{sec:additional-results}

\subsection{Structure Learning}
\label{ssec:csl-performance}
Figure \ref{fig:results-benchmark-appendix1} summarizes the structural Hamming distance (SHD) between the predicted and true graphs for the same datasets and algorithms as in Figure \ref{fig:results-benchmark-main}.

In Figures \ref{fig:results-benchmark-appendix2} and \ref{fig:results-benchmark-appendix3}, we present the F1 scores and SHD attained by the structure learning algorithms on data of \linear \ourss, SCMs, and standardized SCMs, across different weight distribution supports and graph sizes.
We find that the difference in performance of \notears on data sampled from \ours and standardized SCMs is larger for larger weight magnitudes and for bigger graphs.
For smaller weights, the difference in the mean F1 score of \notears between the two standardization approaches is smaller, which is in line with our proposed explanation about the shifts of the implied noise variance distribution in Section \ref{ssec:experimentas-structure-learning}. 

In \cref{fig:results-benchmark-appendix2}, we also find that when weight magnitudes are below $1$, \rtwosortregress performs  similarly for both standardized SCMs and \ourss. We also observe this for \avici. 
Meanwhile, for larger weights with support extending above $1$, these algorithms achieve significantly higher F1 scores on standardized SCMs. 
This suggests that our condition of $\abs{w_{i,j}}>1$ for all edges $(v_i,v_j)$ in the statement of \cref{th:three_part_ident}, concerning the identifiability of  linear standardized SCMs, may have a more fundamental practical significance, rather than being merely an \artefact of the analysis. 

\rebuttal{
In Figure \ref{fig:results-benchmark-nongaussian}, we report results for when the additive noise in the ground-truth SCMs is non-Gaussian. 
In this setting, the causal graphs of SCMs are identifiable from observational data (see Section \ref{sec:background}).
Here, we also benchmark \lingam \citep{shimizu2006linear}, which is designed for this setting.
While \lingam performs very well as expected, it performs significantly worse on standardized SCMs, possibly because independent component analysis suffers in practice under the very low noise scales implied by post-hoc standardization.
This would be in line with our discussion in Section \ref{ssec:experimentas-structure-learning}.
}

\begin{figure}
    \vspacefiguretop
    \centering

  \hspace{20pt}\includegraphics[width=0.72\linewidth]{figures/legend/legend_all.pdf}
        \vspace{-5pt}
    
        \begin{tabular}{cc}
            \yaxislabel{0pt}{15pt}{$\er{20, 2}$}{-10pt} &
        \begin{subfigure}{0.94\linewidth}
        \includegraphics[width=0.5\linewidth, trim=0 6pt 7pt 0 , clip]{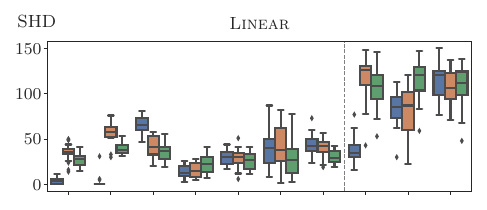}
            \hspace{5pt}\includegraphics[width=00.49\linewidth, trim=4pt 6pt 7pt 6pt , clip]{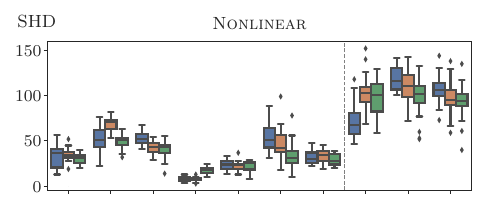}
       \end{subfigure}     
            \\
            \yaxislabel{0pt}{45pt}{$\er{100, 2}$}{-10pt} &
            \begin{subfigure}{0.94\linewidth}\includegraphics[width=0.5\linewidth, trim=4pt 0 7pt 5pt , clip]{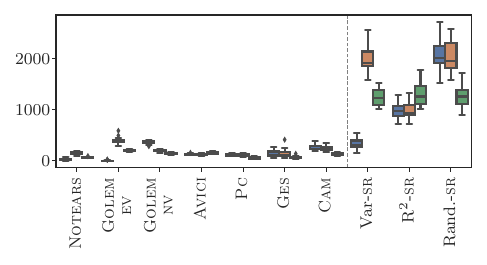}
            \hspace{5pt}\includegraphics[width=00.49\linewidth, trim=8pt 0 7pt 5pt , clip]{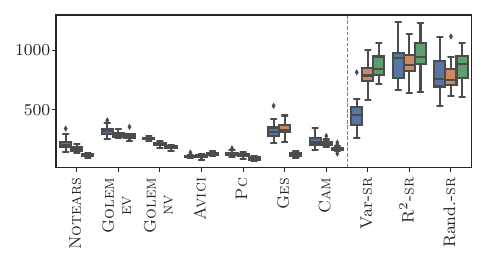}
            \end{subfigure}\\
        \end{tabular}
        \vspace{-6pt}
    \caption{\textbf{SHD to the true causal graph for  \linear and \nonlinear mechanisms.} 
    Box plots show median and interquartile range (IQR). Whiskers extend to the largest value inside \num{1.5}$\times$IQR from the boxes.
    Left (right) column shows results for linear (nonlinear) causal mechanisms with additive noise $\noise{i} \sim \mathcal{N}(0, 1)$. \linear mechanisms have weights $w_{i, j} \sim \unif_{\pm}{}[0.5, 2.0]$.}
    \label{fig:results-benchmark-appendix1}
\end{figure}

\begin{figure}
    \vspacefiguretop
    \centering

  \hspace{20pt}\includegraphics[width=0.72\linewidth]{figures/legend/legend_all.pdf}
        \vspace{-5pt}
    
        \begin{tabular}{cc}
            \yaxislabel{0pt}{15pt}{$\er{20, 2}$}{-10pt} &
        \begin{subfigure}{0.96\linewidth}
        \includegraphics[width=0.52\linewidth, trim=0 6pt 7pt 0 , clip]{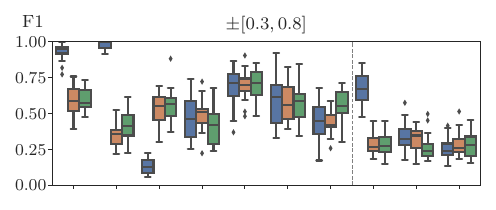}
            \hspace{5pt}\includegraphics[width=00.462\linewidth, trim=25pt 6pt 7pt 0 , clip]{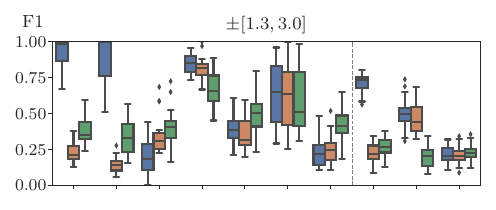}
       \end{subfigure}     
            \\
            \yaxislabel{0pt}{45pt}{$\er{100, 2}$}{-10pt} &
            \begin{subfigure}{0.96\linewidth}\includegraphics[width=0.52\linewidth, trim=0 0 7pt 6pt , clip]{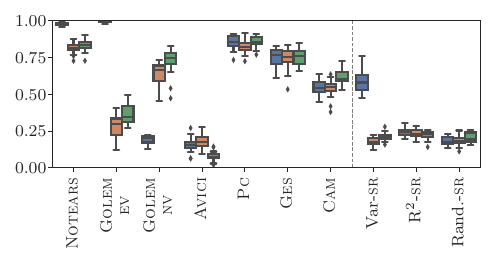}
            \hspace{5pt}\includegraphics[width=00.462\linewidth, trim=25pt 0 7pt 6pt , clip]{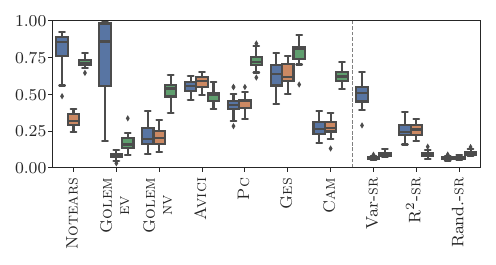}
                    \vspace{-12pt}
    \caption{\textbf{F1 scores}}
    \label{fig:results-benchmark-appendix2}
            \end{subfigure}\\
        \end{tabular}
    \vspace*{15pt}

    \begin{tabular}{cc}
            \yaxislabel{0pt}{15pt}{$\er{20, 2}$}{-10pt} &
        \begin{subfigure}{0.94\linewidth}
        \includegraphics[width=0.5\linewidth, trim=0 6pt 7pt 0 , clip]{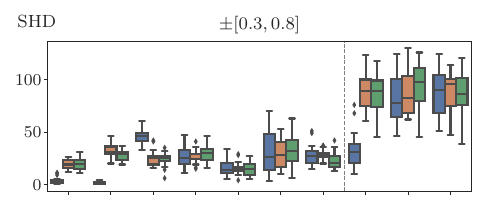}
            \hspace{5pt}\includegraphics[width=00.49\linewidth, trim=4pt 6pt 7pt 6pt , clip]{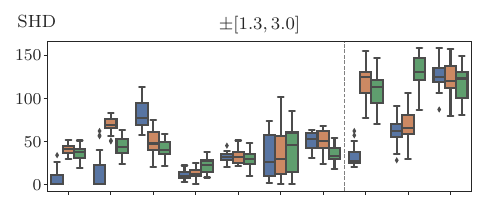}
       \end{subfigure}     
            \\
            \yaxislabel{0pt}{45pt}{$\er{100, 2}$}{-10pt} &
            \begin{subfigure}{0.94\linewidth}\includegraphics[width=0.5\linewidth, trim=4pt 0 7pt 5pt , clip]{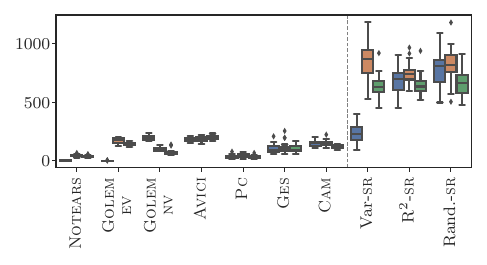}
            \hspace{5pt}\includegraphics[width=00.49\linewidth, trim=8pt 0 7pt 5pt , clip]{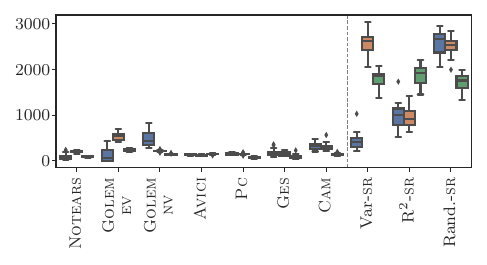}
            \caption{\textbf{SHD to the true causal graph}
    }\label{fig:results-benchmark-appendix3}
    \vspace{0pt}
            \end{subfigure}\\
        \end{tabular}

    \vspace*{10pt}

    \caption{\textbf{Structure learning results for different \linear weight ranges.} 
    Results for \linear causal mechanisms with additive noise $\noise{i} \sim \mathcal{N}(0, 1)$ and weights sampled uniformly from support indicated above each column. 
    Box plots show median and interquartile range (IQR). Whiskers extend to the largest value inside \num{1.5}$\times$IQR from the boxes.
    For every model, we sample \num{20} systems and $\numsamples =$\num{1000} data points each.
    }\label{fig:results-benchmark-appendix-full}
\end{figure}

\subsection{\rtwo-Sortability}
\label{ssec:app-r2-sortability}

Figure \ref{fig:r2-sortability-appendix} reports the \rtwo-sortability statistics across varying graph sizes and weight distributions but for the denser graphs ER($d$, $4$) and SF($d$, $4$).
We again observe \rtwo-sortability very close to \num{0.5} for datasets sampled from \ours and high degrees of \rtwo-sortability for data drawn from standardized SCMs. 
\rebuttal{Moreover, in Figure \ref{fig:r2-sortability-vs-degree}, we show the \rtwo-sortability for varying expected node degrees in the graph. Data sampled from \ourss remains close to not \rtwo-sortable for denser graphs drawn from the graph families considered here.}
We omit standard SCMs from the plots as the datasets of SCMs and their standardized versions have the same \rtwo-sortability, since the \rtwo coefficient is scale invariant.

\begin{figure}
    \centering
    \vspace{-10pt}
    \includegraphics[width=\linewidth]{figures/legend/legend_short.pdf}
    
    \begin{subfigure}{0.49\textwidth}
    \centering
    \includegraphics[width=\linewidth]{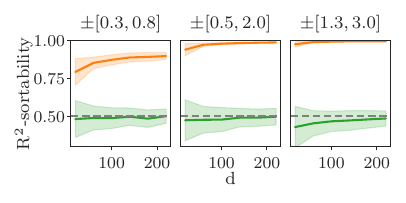}
    \caption{$\er{\numvars, 4}$}
    \label{fig:r2_sortability_er4}

    \end{subfigure}
    \begin{subfigure}{0.49\textwidth}
    \centering
    \includegraphics[width=\linewidth]{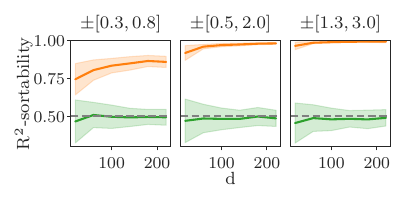}
    \caption{$\usf{\numvars, 4}$}
    \label{fig:r2_sortability_sf4}
    \end{subfigure}
    \caption{\textbf{\rtwo-sortability for different graph sizes.}
    Linear standardized SCMs and \ourss with $\noise{i} \sim \mathcal{N}(0, 1)$ and weights drawn from uniform distributions with supports given above each plot. 
    For every model, we sample \num{100} systems and $\numsamples =$\num{1000} data points each.
    Lines and shaded regions denote mean and standard deviation of \rtwo-sortability across runs.
    Datasets that satisfy \rtwo-sortability$\,=0.5$ (dashed) are not \rtwo-sortable. \looseness-1}
    \label{fig:r2-sortability-appendix}
\end{figure}

\begin{figure}
    \vspacefiguretop
    \centering
    \vspace{-5pt}
    \includegraphics[width=\linewidth]{figures/legend/legend_short.pdf}
    \vspace{-15pt}

    \begin{subfigure}{0.49\textwidth}
    \centering
    \hspace{-13pt}\includegraphics[width=\linewidth]{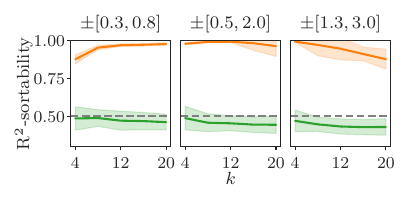}
    
    \vspace{-5pt}
    \caption{$\er{100, k}$}
    \label{fig:r2_sortability_vs_degree_er}

    \end{subfigure}
    \hfill
    \begin{subfigure}{0.49\textwidth}
    \centering
    \hspace{-13pt}\includegraphics[width=\linewidth]{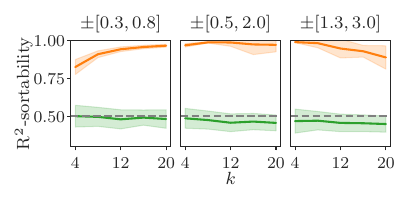}
    
    \vspace{-5pt}
    \caption{$\usf{100, k}$}
    \label{fig:r2_sortability_vs_degree_sf}
    \end{subfigure}
    \caption{\rebuttal{\textbf{\rtwo-sortability for different (expected) node degrees.}
    Linear standardized SCMs and \ourss with $\noise{i} \sim \mathcal{N}(0, 1)$ and weights drawn from uniform distributions with supports given above each plot. 
    For every model, we evaluate \num{100} systems and $\numsamples =$\num{1000} samples each.
    Lines and shaded regions denote mean and standard deviation.
    Datasets that satisfy \rtwo-sortability $=0.5$ (dashed) are not \rtwo-sortable.}
    \vspacefigurecaptionbottom
    }
    \label{fig:r2-sortability-vs-degree}
\end{figure}

\begin{figure}
    \vspacefiguretop
    \centering

  \hspace{20pt}\includegraphics[width=0.72\linewidth]{figures/legend/legend_all.pdf}
        \vspace{-5pt}
    \centering
        \begin{tabular}{c}
        \begin{subfigure}{0.94\linewidth}
        \centering
        \includegraphics[width=0.65\linewidth, trim=0 6pt 7pt 0 , clip]{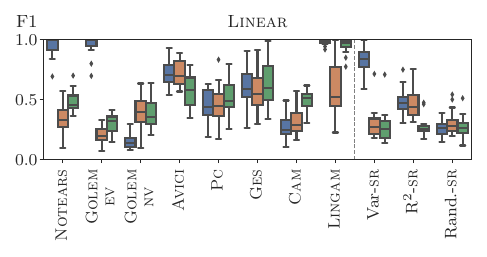}
       \end{subfigure}     
        \end{tabular}
        \vspace{-0pt}
    \caption{
    \rebuttal{
    \textbf{Structure learning results for {\em non-Gaussian} noise distributions.} 
    Causal mechanisms have additive noise $\noise{i} \sim \operatorname{Unif}\smash{[-\sqrt{3}, \sqrt{3}]}$, which induces $\var{\noise{i}} = 1$, and \linear mechanisms with weights $w_{i, j} \sim \unif_{\pm}{}[0.5, 2.0]$.
    Graphs are sampled from $\er{20, 2}$. To obtain the results, we use the same hyperparameters as the ones we used to obtain the top-left panel of \cref{fig:results-benchmark-main}.
    Box plots show median and interquartile range (IQR). Whiskers extend to the largest value inside \num{1.5}$\times$IQR from the boxes.
    }}
    \label{fig:results-benchmark-nongaussian}
\end{figure}

\subsection{Implied Noise Scales}\label{sec:results-similar-implied-noise-scales}

Figure \ref{fig:inverse-noise-scales-dist} shows the inverse implied noise scales of standardized SCMs and \ourss for linear models with smaller weights magnitudes than in Figure \ref{fig:results-induced-noise-main} of the main text.
In this setting with smaller weights, the distributions of the implied noise scales of standardized SCMs and \ourss show significantly greater overlap than in Figure \ref{fig:results-induced-noise-main}.
Since the weights are smaller, the effect of the exploding marginal variances and thus collapsing implied noise scales is weaker in the SCMs.

In Figure \ref{fig:results-benchmark-appendix-full} (left), we evaluate the algorithms considered in Section \ref{sec:experiments} on these systems with smaller weights.
We see that, in this setting, \notears performs very similarly on standardized SCMs and \ourss, with \notears slightly outperforming on \ourss for bigger graphs, since we do not remove the growing variance problem completely even for weights of small magnitude. 
This is inline with our reasoning in Section \ref{ssec:experimentas-structure-learning}.

\begin{figure}
    \centering
    \includegraphics[width=\linewidth]{figures/legend/legend_short.pdf}    
    \includegraphics[width=0.3\linewidth,trim=0pt 0 0 0]{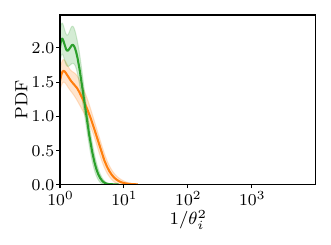}

    \caption{\textbf{Distribution over inverse implied noise scales in the implied SCMs} for $\er{100, 2}$ graphs with smaller weights $w_{i,j} \sim \pm \unif_{[0.3, 0.8]}$, estimated with kernel density estimation. Lines and shading denote mean and standard deviation respectively.}
    \label{fig:inverse-noise-scales-dist}
\end{figure}

\subsection{Covariance Matrices for Figure \ref{fig:figure1}}
Figure \ref{fig:full-covariance-chain} visualizes the full mean absolute covariance (correlation) matrices of the systems presented in Figure \ref{fig:figure1}.
The matrix shows that the pattern of increasing mean absolute covariance in standardized SCMs is not only a feature of neighboring nodes, but it also occurs for vertex pairs further apart, though less strongly. 
This is not the case for \ourss, where any two pairs of equally spaced vertices have equal covariances in expectation over the weight sampling distribution.

\begin{figure}
    \centering
    \includegraphics[width=\linewidth]{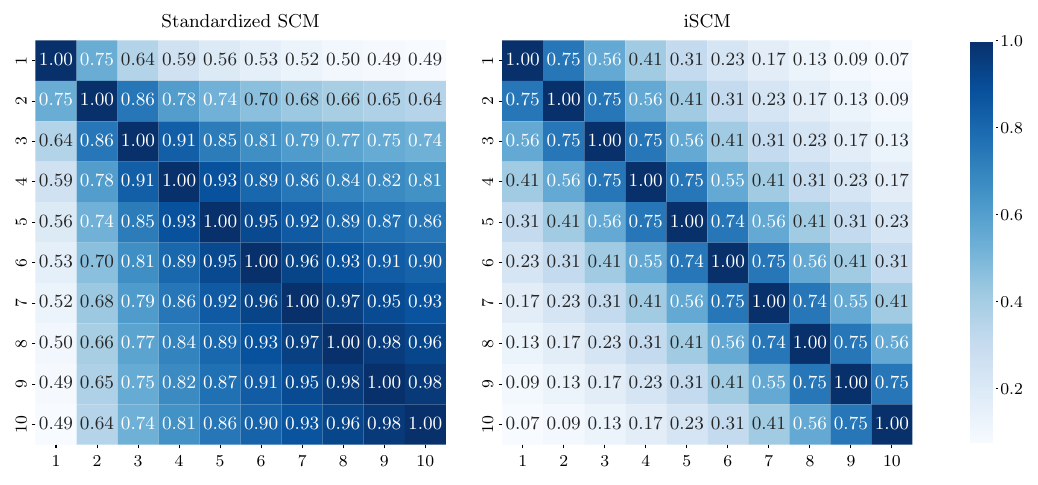}
    \caption{\textbf{Mean absolute covariance matrices for models in Figure \ref{fig:figure1}.} Linear standardized SCMs (left) and \ourss (right) with \num{10}-variable chain DAGs from $x_1$ to $x_{10}$ and weights $w_{i, j} \sim \unif_{\pm}{}[0.5, 2.0]$ and additive noise from $\mathcal{N}(0,1)$. Mean covariances are estimated from $\numsamples =$ \num{100000} datapoints and averaged over \num{100000} models. Since both models have unit marginal variances, covariance equals correlation.}
    \label{fig:full-covariance-chain}
\end{figure}


\end{document}